\documentclass[english]{article}

\usepackage{geometry}
\usepackage[T1]{fontenc}
\usepackage[latin9]{inputenc}
\usepackage{bm}
\usepackage{amsmath}
\usepackage{amssymb} 
\usepackage[unicode=true,
 bookmarks=false,
 breaklinks=false,pdfborder={0 0 1},colorlinks=false]
 {hyperref}
\hypersetup{
 colorlinks,citecolor=blue,filecolor=blue,linkcolor=blue,urlcolor=blue}

\usepackage{xcolor,colortbl}
\definecolor{Gray}{gray}{0.85}

\usepackage{enumitem}

\makeatletter
\usepackage{amsthm}
\usepackage{cite}  
\usepackage{comment}
\usepackage{natbib}
\usepackage{booktabs,mathtools}

\usepackage{graphicx}
\usepackage[linesnumbered,ruled,vlined]{algorithm2e}

\usepackage{algorithmic}

\usepackage{float}
\usepackage{multirow}
 
\usepackage{dsfont}
\usepackage{color}
\definecolor{yxc}{RGB}{255,0,0}
\definecolor{yjc}{RGB}{125,0,0}
\definecolor{ytw}{RGB}{255,69,0}
\definecolor{gen}{RGB}{0,0,200}
\definecolor{lxs}{RGB}{138,43,226}

\newcommand{\red}[1]{\textcolor{red}{#1}}

\allowdisplaybreaks

\DeclareMathOperator{\ind}{\mathds{1}}  

\newcommand{\defn}{\coloneqq}

\newcommand{\Var}{\mathsf{Var}}

\newcommand{\cb}{c_{\mathrm{b}}}

\newcommand{\cA}{\mathcal{A}}

\newcommand{\cS}{{\mathcal{S}}}

\newcommand{\mymid}{\,|\,} 

\newcommand{\DADV}{{\sf Q-EarlySettled-Advantage}\xspace}

\newcommand{\UCBQ}{{\sf UCB-Q}\xspace}
\newcommand{\UCB}{{\sf UCB}\xspace}
\newcommand{\UCBVI}{{\sf UCB-VI}\xspace}
\newcommand{\UCBQA}{{\sf UCB-Q-Advantage}\xspace}
\newcommand{\UCBM}{{\sf UCB-M-Q}\xspace}
\newcommand{\LCBQ}{{\sf LCB-Q}\xspace}
\newcommand{\LCB}{{\sf LCB}\xspace}

\newcommand{\rref}{\mathsf{R}}
\newcommand{\refmur}{\mu^{\mathsf{ref}, \mathsf{R}}}

\newcommand{\sumb}{B}
\newcommand{\re}{\mathsf{ref}}
\newcommand{\adv}{\mathsf{adv}}
\newcommand{\refmu}{\mu^{\mathsf{ref}}} 
\newcommand{\refsg}{\sigma^{\mathsf{ref}}}
\newcommand{\advmu}{\mu^{\mathsf{adv}}}
\newcommand{\advsg}{\sigma^{\mathsf{adv}}}
\newcommand{\nextb}{B^{\mathsf{next}}}

\usepackage{scalerel,stackengine}
\stackMath
\newcommand\reallywidehat[1]{%
\savestack{\tmpbox}{\stretchto{%
  \scaleto{%
    \scalerel*[\widthof{\ensuremath{#1}}]{\kern-.6pt\bigwedge\kern-.6pt}%
    {\rule[-\textheight/2]{1ex}{\textheight}}
  }{\textheight}%
}{0.5ex}}%
\stackon[1pt]{#1}{\tmpbox}%
}
\newcommand\reallywidecheck[1]{%
\savestack{\tmpbox}{\stretchto{%
  \scaleto{
    \scalerel*[\widthof{\ensuremath{#1}}]{\kern-.6pt\bigwedge\kern-.6pt}%
    {\rule[-\textheight/2]{1ex}{\textheight}}
  }{\textheight}%
}{0.5ex}}%
\stackon[1pt]{#1}{\scalebox{-1}{\tmpbox}}%
}

\title{Breaking the Sample Complexity Barrier to \\ Regret-Optimal Model-Free Reinforcement Learning}

\author{Gen Li\thanks{Department of Statistics and Data Science, Wharton School, University of Pennsylvania, Philadelphia, PA 19104, USA.}  \\
UPenn    \\
\and
	Laixi Shi\thanks{Department of Electrical and Computer Engineering, Carnegie Mellon University, Pittsburgh, PA 15213, USA.}\\
	CMU\\
	\and
	Yuxin Chen\footnotemark[1] \\
	UPenn \\
\and
	Yuejie Chi\footnotemark[2] \\
 	CMU  
	}
\date{October 2021; Revised: August 2022}

\begin{document}
\theoremstyle{plain} \newtheorem{lemma}{\textbf{Lemma}}\newtheorem{proposition}{\textbf{Proposition}}\newtheorem{theorem}{\textbf{Theorem}}

\theoremstyle{remark}\newtheorem{remark}{\textbf{Remark}}

\maketitle

\begin{abstract}

Achieving sample efficiency in online episodic reinforcement learning (RL) requires optimally balancing  exploration and exploitation. When it comes to a finite-horizon episodic Markov decision process with $S$ states, $A$ actions and horizon length $H$, substantial progress has been achieved towards characterizing the minimax-optimal regret, which scales on the order of $\sqrt{H^2SAT}$ (modulo log factors) with $T$ the total number of samples. While several competing solution paradigms have been proposed to minimize regret, they are either memory-inefficient, or fall short of optimality unless the sample size exceeds an enormous threshold (e.g., $S^6A^4 \,\mathrm{poly}(H)$ for existing model-free methods). 

To overcome such a large sample size barrier to efficient RL, we design a novel model-free algorithm, with space complexity $O(SAH)$, that achieves near-optimal regret  as soon as the sample size exceeds the order of $SA\,\mathrm{poly}(H)$. 
In terms of this sample size requirement (also referred to the initial burn-in cost), 
	our method improves --- by at least a factor of $S^5A^3$ --- upon any prior memory-efficient algorithm that is asymptotically regret-optimal.   Leveraging the recently introduced variance reduction strategy (also called {\em reference-advantage decomposition}),  
the proposed algorithm employs an {\em early-settled}  reference update rule,
with the aid of two Q-learning sequences with upper and lower confidence bounds.  
The design principle of our early-settled variance reduction method might be of independent interest to other RL settings that involve intricate exploration-exploitation trade-offs. 

\end{abstract}


\allowdisplaybreaks

\setcounter{tocdepth}{2}
\tableofcontents

\section{Introduction}

Contemporary reinforcement learning (RL) has to deal with unknown environments with unprecedentedly large dimensionality. 
How to make the best use of samples in the face of high-dimensional state/action space lies at the core of modern RL practice.  
An ideal RL algorithm would learn to act favorably even when the number of available data samples scales sub-linearly in the ambient dimension of the model, i.e., the number of parameters needed to describe the transition dynamics of the environment. 
The challenge is further compounded when this task needs to be accomplished with limited memory.

Simultaneously achieving the desired sample and memory efficiency is particularly challenging when it comes to online episodic RL scenarios. In contrast to the simulator setting that permits sampling of any state-action pair,  an agent in online episodic RL is only allowed to draw sample trajectories by executing a policy in the unknown Markov decision process (MDP), 
where the initial states are pre-assigned and might even be chosen by an adversary. 
Careful deliberation needs to be undertaken when deciding what policies to use to allow for effective interaction with the unknown environment, how to optimally balance exploitation and exploration, and how to process and store the collected information intelligently without causing redundancy.

\subsection{Regret-optimal model-free RL? A sample size barrier}

In order to evaluate and compare the effectiveness of RL algorithms in high dimension, 
a recent body of works sought to develop a finite-sample theoretical framework to analyze the algorithms of interest, 
with the aim of delineating the dependency of algorithm performance on all salient problem parameters in a non-asymptotic fashion 
\citep{dann2017unifying,kakade2003sample}.  Such finite-sample guarantees are brought to bear towards understanding and tackling the challenges in the sample-starved regime commonly encountered in practice. 
To facilitate discussion, let us take a moment to summarize the state-of-the-art theory for episodic finite-horizon MDPs with non-stationary transition kernels, focusing on minimizing cumulative regret  --- a metric that quantifies the performance difference between the learned policy and the true optimal policy --- with the fewest number of samples.
Here and throughout, we denote by $S$, $A$, and $H$ the size of the state space, the size of the action space, and the horizon length of the MDP, respectively, and let $T$ represent the sample size. In addition, the immediate reward gained at each time step is assumed to lie between $0$ and $1$.

\paragraph{Fundamental regret lower bound.} Following the arguments in \citet{jaksch2010near,auer2002nonstochastic}, the recent works \cite{jin2018q,domingues2021episodic} developed a fundamental lower bound\footnote{It is worth emphasizing that \citep{domingues2021episodic}  adopts the notation $T$ to represent the number of trajectories (with each trajectory containing $H$ samples), while the present paper employs $K$ to denote the number of sample trajectories and $T = KH$ the total number of samples.
Consequently, the lower bound developed in \citep{domingues2021episodic}  for non-stationary finite-horizon MDPs reads $\Omega(\sqrt{H^3SAK})$, or equivalently, $\Omega(\sqrt{H^2SAT})$ using the notation adopted herein.} on the expected total regret for this setting. Specifically, this lower bound claims that: no matter what algorithm to use, one can find an MDP such that the accumulated regret incurred by the algorithm necessarily exceeds the order of
		\begin{align}
			\text{(lower bound)} \qquad \sqrt{H^2SAT} ,
			\label{eq:lower-bound-regret}
		\end{align} 
		as long as $T\geq H^2SA$.\footnote{Given that a trivial upper bound on the regret is $T$, one needs to impose a lower bound $T\geq H^2SA$ in order for \eqref{eq:lower-bound-regret} to be meaningful.} This sublinear regret lower bound in turn imposes a sampling limit if one wants to achieve $\varepsilon$ average regret.

\paragraph{Model-based RL.} 
		
Moving beyond the lower bound, let us examine the effectiveness of model-based RL --- an approach that can be decoupled into a model estimation stage (i.e., estimating the transition kernel using available data) and a subsequent stage of planning using the learned model \citep{jaksch2010near,azar2017minimax,efroni2019tight,agrawal2017posterior,pacchiano2020optimism}. 
		In order to ensure a sufficient degree of exploration, 
\citet{azar2017minimax} came up with an algorithm called \UCBVI that blends model-based learning and the optimism principle,
which achieves a regret bound\footnote{Here and throughout, we use the standard notation $f(n)=O(g(n))$ to indicate that $f(n)/g(n)$ is bounded above by a constant as $n$ grows. The notation $\widetilde{O}(\cdot)$ resembles  $O(\cdot)$  except that it hides any logarithmic scaling. The notation $f(n)=o(g(n))$ means that $\lim_{n\rightarrow\infty} {f(n)}/{g(n)}=0$. } $\widetilde{O}\big(\sqrt{H^2SAT} \big)$ that nearly attains the lower bound \eqref{eq:lower-bound-regret} as $T$ tends to infinity. 
Caution needs to be exercised, however, that existing theory does not guarantee 
the near optimality of this algorithm unless the sample size $T$ surpasses 
		$$T\geq S^3AH^6,$$ 
		a threshold that is significantly larger than the dimension of the underlying model.
This threshold can also be understood as the initial {\em burn-in cost} of the algorithm, namely, a sampling burden needed for the algorithm to exhibit the desired performance. 
		In addition, model-based algorithms typically require storing the estimated probability transition kernel, resulting in a space complexity that could be as high as ${O}(S^2AH)$ \citep{azar2017minimax}.

\paragraph{Model-free RL.}

Another competing solution paradigm is the model-free approach, which circumvents the model estimation stage and attempts to learn the optimal values directly \citep{strehl2006pac,jin2018q,bai2019provably,yang2021q}. In comparison to the model-based counterpart, the model-free approach holds the promise of low space complexity, as it eliminates the need of storing a full description of the model. In fact, in a number of previous works (e.g., \citet{strehl2006pac,jin2018q}), an algorithm is declared to be model-free only if its space complexity is $o(S^2AH)$ regardless of the sample size $T$. 
\begin{itemize}
	\item {\em Memory-efficient model-free methods.} 
		\citet{jin2018q} proposed the first memory-efficient model-free algorithm --- which is an optimistic variant of classical Q-learning --- that achieves a regret bound proportional to $\sqrt{T}$ with a space complexity $O(SAH)$. Compared to the lower bound \eqref{eq:lower-bound-regret}, however, the regret bound in \citet{jin2018q} is off by a factor of $\sqrt{H}$ and hence suboptimal for problems with long horizon. This drawback has recently been overcome in \citet{zhang2020almost} by leveraging the idea of variance reduction (or the so-called ``reference-advantage decomposition'') for large enough $T$. While the resulting regret matches the information-theoretic limit asymptotically, its optimality in the non-asymptotic regime is not guaranteed unless the sample size $T$ exceeds (see \citet[Lemma~7]{zhang2020almost})
		$$T\geq S^6A^4H^{28},$$ 
		a requirement that is even far more stringent than the burn-in cost imposed by  \citet{azar2017minimax}.

	\item {\em A memory-inefficient ``model-free'' variant.} 
 	The recent work \citet{menard2021ucb} put forward a novel sample-efficient variant  of Q-learning called \UCBM,  
	which relies on a carefully chosen momentum term for bias reduction.   	
		This algorithm is guaranteed to yield near-optimal regret $\widetilde{O}\big(\sqrt{H^2SAT} \big)$ as soon as the sample size exceeds $T\geq SA \mathrm{poly}(H)$,
	which is a remarkable improvement vis-\`a-vis previous regret-optimal methods \citep{azar2017minimax,zhang2020almost}.
		Nevertheless, akin to the model-based approach,  it comes at a price in terms of the space and computation complexities, as the space required to store all bias-value function is $O(S^2AH)$ and the computation required is $O(ST)$, both of which are larger by a factor of $S$ than other model-free algorithms like \citet{zhang2020almost}. In view of this memory inefficiency, \UCBM falls short of fulfilling the definition of model-free algorithms in \citet{strehl2006pac,jin2018q}. See  \citet[Section 3.3]{menard2021ucb} for more detailed discussions.

\end{itemize}

\noindent 
A more complete summary of prior results can be found in Table~\ref{tab:prior-work}.

\newcommand{\topsepremove}{\aboverulesep = 0mm \belowrulesep = 0mm} \topsepremove

\begin{table}[t]
	\begin{center}
\begin{tabular}{c|c|c|c}
\toprule

\multirow{2}{*}{Algorithm} & \multirow{2}{*}{Regret} & Range of sample sizes $T$ \vphantom{$\frac{1^{7}}{1^{7^{7}}}$} & Space \tabularnewline
	&  & that attain optimal regret  \vphantom{$\frac{1}{1^{7^{7}}}$} & complexity \tabularnewline
\toprule
 
\UCBVI \vphantom{$\frac{1^{7}}{1^{7^{7}}}$} & \multirow{2}{*}{$\begin{array}{cc}
\sqrt{H^2SAT}+ H^{4}S^{2}A
\end{array}$ } & \multirow{2}{*}{$\red{[S^3AH^6,\infty)}$} & \multirow{2}{*}{\red{$S^2AH$}}\tabularnewline
	\citep{azar2017minimax}  &  &  & \tabularnewline
\hline 
	{\sf UCB-Q-Hoeffding} \vphantom{$\frac{1^{7}}{1^{7^{7}}}$}  & \multirow{2}{*}{$\sqrt{H^4SAT}$  } & \multirow{2}{*}{\red{\text{never}}} & \multirow{2}{*}{$SAH$}\tabularnewline
\citep{jin2018q} &  &  & \tabularnewline
\hline 
	{\sf UCB-Q-Bernstein} \vphantom{$\frac{1^{7}}{1^{7^{7}}}$} & \multirow{2}{*}{$\sqrt{H^3SAT} + \sqrt{H^9S^3A^3}$} & \multirow{2}{*}{\red{never}} & \multirow{2}{*}{$SAH$}\tabularnewline
\citep{jin2018q} &  &  & \tabularnewline
\hline 
	{\sf UCB2-Q-Bernstein} \vphantom{$\frac{1^{7}}{1^{7^{7}}}$} & \multirow{2}{*}{$\sqrt{H^3SAT} + \sqrt{H^9S^3A^3}$} & \multirow{2}{*}{\red{never}} & \multirow{2}{*}{$SAH$}\tabularnewline
	\citep{bai2019provably} &  &  & \tabularnewline

\hline 
	\UCBQA \vphantom{$\frac{1^{7}}{1^{7^{7}}}$} & \multirow{2}{*}{$\sqrt{H^2SAT} + H^8S^2A^{\frac{3}{2}}T^{\frac{1}{4}}$} & \multirow{2}{*}{$\red{[S^6A^4H^{28},\infty)}$} & \multirow{2}{*}{$SAH$}\tabularnewline
\citep{zhang2020almost} &  &  & \tabularnewline
\hline
	\UCBM \vphantom{$\frac{1^{7}}{1^{7^{7}}}$} & \multirow{2}{*}{$\sqrt{H^2SAT} + H^4SA$} & \multirow{2}{*}{$[SAH^{6},\infty)$} & \multirow{2}{*}{\red{$S^2AH$}}\tabularnewline
\citep{menard2021ucb} &  &  & \tabularnewline \hline
\rowcolor{Gray} 
	\DADV \vphantom{$\frac{1^{7}}{1^{7^{7}}}$} &  &  & \tabularnewline
\rowcolor{Gray}
	{\bf (this work)} & \multirow{-2}{*}{\cellcolor{Gray}$\sqrt{H^2SAT}+H^6SA$} & \multirow{-2}{*}{\cellcolor{Gray}$[SAH^{10},\infty)$} &  \multirow{-2}{*}{\cellcolor{Gray}$SAH$}\tabularnewline
\toprule
Lower bound \vphantom{$\frac{1^{7}}{1^{7^{7}}}$} & \multirow{2}{*}{$ \sqrt{H^2SAT} $} & \multirow{2}{*}{n/a} & \multirow{2}{*}{n/a}\tabularnewline
\citep{domingues2021episodic} &  &  & \tabularnewline
\toprule
 
\end{tabular}

	\end{center}

	\caption{Comparisons between prior art and our results for non-stationary episodic MDPs when $T\geq H^2SA$. The table includes the order of the regret bound, the range of sample sizes that achieve the optimal regret $\widetilde{O}(\sqrt{H^2SAT})$, and the memory complexity, with all logarithmic factors omitted for simplicity of presentation. 
	The red text highlights the suboptimal part of the respective algorithms. 
	\label{tab:prior-work}}  
	
\end{table}

\subsection{A glimpse of our contributions}

In brief, while it is encouraging to see that both model-based and model-free approaches allow for near-minimal regret as $T$ tends to infinity,  they are either memory-inefficient, or require the sample size to exceed a threshold substantially larger than the model dimensionality. In fact, no prior algorithms have been shown to be {\em simultaneously regret-optimal and memory-efficient} unless
$$T\geq S^6A^4  \,\mathrm{poly}(H),$$
which constitutes a stringent sample size barrier constraining their utility in the sample-starved and memory-limited regime. 
The presence of this sample complexity barrier motivates one to pose a natural question: 
{
\setlist{rightmargin=\leftmargin}
\begin{itemize}
	\item[] {\em Is it possible to design an algorithm that accommodates a significantly broader sample size range without compromising regret optimality and memory efficiency?}
\end{itemize}
}
\noindent 
In this paper, we answer this question affirmatively, by designing a new model-free algorithm, dubbed as \DADV, that enjoys the following performance guarantee.
\begin{theorem}[informal]
	The proposed \DADV algorithm, which has a space complexity $O(SAH)$, achieves near-optimal regret $\widetilde{O}\big( \sqrt{H^2SAT} \big)$ as soon as the sample size exceeds $T\geq  SA \,\mathrm{poly}(H)$.
\end{theorem}
As can be seen from Table~\ref{tab:prior-work}, the space complexity of the proposed algorithm is $O(SAH)$, 
which is far more memory-efficient than both the model-based approach in \citet{azar2017minimax} and the \UCBM algorithm in \citet{menard2021ucb} (both of these prior algorithms require $S^2AH$ units of space). In addition, the sample size requirement $T\geq SA \, \mathrm{poly}(H)$ of our algorithm improves --- by a factor of at least $S^5A^3$ --- upon that of any prior algorithm that is simultaneously regret-optimal and memory-efficient. In fact, this requirement is nearly sharp in terms of the dependency on both $S$ and $A$, and was previously achieved only by the \UCBM algorithm at a price of a much higher storage burden.

Let us also briefly highlight the key ideas of our algorithm. As an optimistic variant of variance-reduced Q-learning, \DADV leverages the recently-introduced reference-advantage decompositions for variance reduction \citep{zhang2020almost}.
As a distinguishing feature from prior algorithms, we employ an {\em early-stopped} reference update rule, 
with the assistance of two Q-learning sequences that incorporate upper and lower confidence bounds, respectively.  The design of our early-stopped variance reduction scheme, as well as its analysis framework, might be of independent interest to other settings that involve managing intricate exploration-exploitation trade-offs.

\subsection{Related works}

We now take a moment to discuss a small sample of other related works. 
We limit our discussions primarily to RL algorithms in the tabular setting with finite state and action spaces, 
which are the closest to our work. The readers interested in those model-free variants with function approximation are referred to \citet{du2019provably,fan2019theoretical,murphy2005generalization} and the references therein.

\paragraph{PAC bounds for synchronous and asynchronous Q-learning.} 
Q-learning is arguably among the most famous model-free algorithms developed in the RL literature \citep{watkins1992q,tsitsiklis1994asynchronous,jaakkola1994convergence,szepesvari1997asymptotic}, which enjoys a space complexity $O(SAH)$. 
Non-asymptotic sample analysis and probably approximately correct (PAC) bounds have seen extensive developments in the last several years, 
including but not limited to the works of
\citet{wainwright2019stochastic,even2003learning,beck2012error,chen2020finite,li2021tightening} for the synchronous setting (the case with access to a generative model or a simulator), 
and the works of \citet{even2003learning,beck2012error,qu2020finite,li2020sample,chen2021lyapunov} for the asynchronous setting (where one observes a single Markovian trajectory induced by a behavior policy). Finite-time guarantees of other variants of Q-learning have also been developed; partial examples include speedy Q-learning \citep{ghavamzadeh2011speedy}, double Q-learning \citep{xiong2020finite},  variance-reduced Q-learning \citep{wainwright2019variance,li2020sample}, momentum Q-learning \citep{weng2020momentum}, 
and Q-learning for linearly parameterized MDPs \citep{wang2021sample}. 
This line of works did not account for exploration, and hence the success of Q-learning in these settings heavily relies on the access to a simulator or a behavior policy with sufficient coverage over the state-action space.

\paragraph{Regret analysis for model-free RL with exploration.} 

When it comes to online episodic RL (so that a simulator is unavailable), 
regret analysis is the prevailing analysis paradigm employed to capture the trade-off between exploration and exploitation. 
A common theme is to augment the original model-free update rule (e.g., the Q-learning update rule) by an exploration bonus, 
which typically takes the form of, say, certain upper confidence bounds (UCBs) motivated by the bandit literature \citep{lai1985asymptotically,auer2010ucb}.  
In addition to the ones in Table~\ref{tab:prior-work} for episodic finite-horizon settings, 
sample-efficient model-free algorithms have been investigated for infinite-horizon MDPs as well \citep{dong2019q,zhang2020reinforcement,zhang2020model,jafarnia2020model,liu2020gamma,yang2021q}.

\paragraph{Variance reduction in RL.} The seminal idea of variance reduction was originally proposed to accelerate finite-sum stochastic optimization, e.g., \cite{johnson2013accelerating,gower2020variance,nguyen2017sarah}. Thereafter, the  variance reduction strategy has been imported to RL, which assists in improving the sample efficiency of RL algorithms in multiple contexts, including but not limited to 
policy evaluation \citep{du2017stochastic,wai2019variance,xu2019reanalysis,khamaru2020temporal},
RL with a generative model \citep{sidford2018near,sidford2018variance,wainwright2019variance}, 
asynchronous Q-learning \citep{li2020sample}, 
and offline RL \citep{yin2021near,shi2022pessimistic}.

\paragraph{Model-based approach.} 

Model-based RL is known to be minimax-optimal in the presence of a simulator \citep{azar2013minimax,agarwal2019optimality,li2020breaking}, beating the state-of-the-art model-free algorithms by achieving optimality for the entire sample size range \citep{li2020breaking}. 
When it comes to online episodic RL, \citet{azar2017minimax} was the first work that managed to achieve near-optimal regret (at least for large $T$);  
in fact, this was also the first result (for any algorithm) matching existing lower bounds for large $T$. 
The sample efficiency of the model-based approach 
has subsequently been established for other settings, including but not limited  to discounted infinite-horizon MDPs  \citep{he2020nearly}, MDPs with bounded total reward \citep{zanette2019tighter,zhang2020reinforcement},  and Markov games \citep{zhang2020model_game}.

\paragraph{Regret lower bound.} 
Inspired by the classical lower bound argument developed for multi-armed bandits \citep{auer2002nonstochastic}, 
the work \citet{jaksch2010near} established a regret lower bound for MDPs with finite diameters (so that for an arbitrary pair of states, the expected time to transition between them is assumed to be finite as long as a suitable policy is used), which has been reproduced in the note \citet{osband2016lower} with the purpose of facilitating comparison with \citet{bartlett2009regal}. The way to construct hard MDPs in \citet{jaksch2010near} has since been adapted by \citet{jin2018q} to exhibit a lower bound on episodic MDPs (with a sketched proof provided therein). It was recently revisited in \citet{domingues2021episodic}, which presented a detailed and rigorous proof argument with a different construction.

\section{Problem formulation}
\label{sec:problem-formulation}

In this section, we formally describe the problem setting. Here and throughout, we denote by $\Delta(\cS)$ the probability simplex over a set $\cS$, and introduce the notation $[M]\coloneqq \{1,\cdots,M\}$ for any integer $M>0$.

\paragraph{Basics of finite-horizon MDPs.}

Let $\mathcal{M}=(\mathcal{S},\mathcal{A},H, \{P_h\}_{h=1}^H, \{r_h\}_{h=1}^H)$ represent a finite-horizon MDP, where $\mathcal{S}\coloneqq \{1,\cdots S\}$ is the state space of size $S$, 
$\mathcal{A}\coloneqq \{1,\cdots, A\}$ is the action space of size $A$, 
$H$ denotes the horizon length, and $P_h : \cS \times \cA \rightarrow \Delta (\cS) $ (resp.~$r_h: \cS \times \cA \rightarrow [0,1]$)
represents the probability transition kernel (resp.~reward function) at the $h$-th time step, $1\leq h\leq H$, respectively. 
More specifically, $P_h(\cdot\mymid s,a)\in \Delta(\cS)$ stands for the transition probability vector from state $s$ at time step $h$ when action $a$ is taken, 
while $r_h(s,a)$ indicates the immediate reward received at time step $h$ for a state-action pair $(s,a)$ (which is assumed to be deterministic and fall within the range $[0,1]$). 
The MDP is said to be non-stationary when the $P_h$'s are not identical across $1\leq h\leq H$. 
A policy of an agent is represented by $\pi =\{\pi_h\}_{h=1}^H$ with $\pi_h: \mathcal{S} \rightarrow \mathcal{A}$ the action selection rule at time step $h$, so that 
$\pi_h(s)$ specifies which action to execute in state $s$ at time step $h$. Throughout this paper, we concentrate on deterministic policies.

\paragraph{Value functions, Q-functions, and Bellman equations.}

The value function $V^{\pi}_{h}(s)$ of a (deterministic) policy $\pi$ at step $h$ is defined as the expected cumulative rewards received between time steps $h$ and $H$ when executing this policy from an initial state $s$ at time step $h$, namely, 
\begin{align}
	\label{eq:def_Vh}
	V^{\pi}_{h}(s) & \coloneqq  \mathop{\mathbb{E}}\limits_{s_{t+1}\sim P_{t}(\cdot|s_{t},\pi_{t}(s_{t})), \, t\geq h} 
	\left[  \sum_{t=h}^{H} r_{t}\big(s_{t},\pi_{t}(s_{t}) \big) \,\Big|\, s_{h}=s \right], 
\end{align}
where the expectation is taken over the randomness of the MDP trajectory $\{s_t \mid h\leq t\leq H\}$. 
The action-value function (a.k.a.~the Q-function) $Q^{\pi}_h(s,a)$ of a policy $\pi$ at step $h$ can be defined analogously 
except that the action at step $h$ is fixed to be $a$, that is, 
\begin{align} 
	\label{eq:def_Qh}
	Q^{\pi}_{h}(s,a) & \coloneqq r_{h}(s,a)+ \mathop{\mathbb{E}}\limits_{\substack{s_{h+1}\sim P_h(\cdot | s, a), \\ s_{t+1}\sim P_{t}(\cdot|s_{t},\pi_{t}(s_{t})), \, t> h}} \left[  \sum_{t=h +1}^{H} r_t \big(s_t,\pi_t(s_t) \big) \,\Big|\, s_{h}=s, a_h  = a\right].
\end{align}
In addition, we define $V^{\pi}_{H+1}(s)= Q^{\pi}_{H+1}(s,a)=0$ for any policy $\pi$ and any state-action pair $(s,a)\in \cS \times \cA$. 
By virtue of basic properties in dynamic programming \citep{bertsekas2017dynamic}, 
the value function and the Q-function satisfy the following Bellman equation:
\begin{align}
	\label{eq:bellman}
	Q^{\pi}_{h}(s,a)=r_{h}(s,a)+ \mathop{\mathbb{E}}\limits_{s'\sim P_h(\cdot\mid s,a)} \big[V^{\pi}_{h+1}(s') \big]. 
\end{align}

A policy $\pi^{\star} =\{\pi_h^{\star}\}_{h=1}^H$ is said to be an optimal policy if it maximizes the value function simultaneously for all states among all policies. 
The resulting optimal value function $V^{\star} =\{ V_h^{\star} \}_{h=1}^H $ and optimal Q-functions $Q^{\star} =\{ Q_h^{\star} \}_{h=1}^H $ satisfy  
\begin{align}
	V_h^{\star}(s) =V_h^{\pi^{\star}}(s) = \max_{\pi} V_h^{\pi}(s) \qquad \mbox{and}\qquad   Q_h^{\star}(s,a)  =Q_h^{\pi^{\star}}(s,a)  = \max_{\pi} Q_h^{\pi}(s,a) 
\end{align}
for any $(s,a,h)\in \cS\times \cA \times [H]$. 
It is well known that the optimal policy always exists \citep{puterman2014markov}, and satisfies the Bellman optimality equation:
\begin{align}
	\label{eq:bellman_optimality}
	\forall (s,a,h)\in \cS\times \cA\times [H]: \qquad
	Q^{\star}_{h}(s,a)=r_{h}(s,a)+ \mathop{\mathbb{E}}\limits_{s'\sim P_h(\cdot\mid s,a)} \big[V^{\star}_{h+1}(s') \big]. 
\end{align}

\paragraph{Online episodic RL.} 

This paper investigates the online episodic RL setting, where the agent is allowed to execute the MDP sequentially in a total number of $K$ episodes each of length $H$.
This amounts to collecting  $$T=KH\text{ samples}$$ in total. 
More specifically, in each episode $k=1,\ldots, K$, the agent is assigned an arbitrary initial state $s_1^k$ (possibly by an adversary), 
and selects a policy $\pi^k =\{\pi_h^k \}_{h=1}^H$  
learned based on the information collected up to the $(k-1)$-th episode. 
The $k$-th episode is then executed following the policy $\pi^k$ and the dynamic of the MDP $\mathcal{M}$, 
leading to a length-$H$ sample trajectory.

\paragraph{Goal: regret minimization.}

In order to evaluate the quality of the learned policies $\{\pi^k\}_{1\leq k\leq K}$, 
a frequently used performance metric is the cumulative regret defined as follows:
\begin{equation} 
	\label{eq:def_regret}
	\mathsf{Regret}(T) \coloneqq \sum_{k=1}^{K}\big(V^{\star}_{1}(s_1^k)-V^{\pi^{k}}_{1}(s_1^k)\big).
\end{equation}
 In words, the regret reflects the sub-optimality gaps between the values of the optimal policy and those of the learned policies aggregated over $K$ episodes. 
A natural objective is thus to design a sample-optimal algorithm, namely, an algorithm whose resulting regret scales optimally in the sample size $T$. 
Accomplishing this goal requires carefully managing the trade-off between exploration and exploitation, which is particularly challenging in the sample-limited regime.

\paragraph{Notation.} 

Before presenting our main results, 
we take a moment to introduce some convenient notation to be used throughout the remainder of this paper. 
For any vector $x \in \mathbb{R}^{SA}$ that constitutes certain quantities for all state-action pairs,
we shall often use $x(s,a)$ to denote the entry associated with the state-action pair $(s,a)$, as long as it is clear from the context. 
We shall also let
\begin{align}
	P_{h,s,a} = P_h(\cdot \mymid s,a) \in \mathbb{R}^{1\times S}
	\label{eq:defn-P-hsa}
\end{align}
abbreviate the transition probability vector given the $(s,a)$ pair at time step $h$. 
Additionally, we denote by $e_i$ the $i$-th standard basis vector, with the only non-zero element being in the $i$-th entry and equal to 1.

\section{Algorithm and theoretical guarantees}

In this section, we present the proposed algorithm called \DADV, as well as the accompanying theory confirming its sample and memory efficiency.

\subsection{Review: Q-learning with UCB exploration and reference advantage}
\label{sec:standard-ucb-Q}

This subsection briefly reviews the Q-learning algorithm with UCB exploration proposed in \citet{jin2018q}, 
as well as a variant that further exploits the idea of variance reduction \citep{zhang2020almost}. These two model-free algorithms inspire the algorithm design in the current paper.

\paragraph{Q-learning with UCB exploration (\UCBQ or {\sf UCB-Q-Hoeffding}).}

Recall that the classical $Q$-learning algorithm has been proposed as a stochastic approximation scheme \citep{robbins1951stochastic} to solve the Bellman optimality equation \eqref{eq:bellman_optimality}, 
which consists of the following update rule \citep{watkins1989learning,watkins1992q}:

\begin{equation}
	Q_h(s,a) \leftarrow (1-\eta)Q_h(s,a) + \eta\Big\{ r_h(s,a) + 
	\hspace{-1em} \underset{\text{stochastic estimate of }{P}_{h,s,a}V_{h+1}}{\underbrace{ \widehat{P}_{h,s,a} V_{h+1} }}  \hspace{-1em} \Big\} .
	\label{eq:classical-Q-update}
\end{equation}

Here, $Q_h$ (resp.~$V_h$) indicates the running estimate of $Q_h^{\star}$ (resp.~$V_h^{\star}$), 
$\eta$ is the (possibly iteration-varying) learning rate or stepsize, 
and $\widehat{P}_{h,s,a} V_{h+1}$ is a stochastic estimate of $P_{h,s,a} V_{h+1}$ (cf.~\eqref{eq:defn-P-hsa}). 
For instance, if one has available a sample $(s,a,s')$ transitioning from state $s$ at step $h$ to $s'$ at step $h+1$ under action $a$, 
then a stochastic estimate of ${P}_{h,s,a}V_{h+1}$ can be taken as $V_{h+1}(s')$, which is unbiased in the sense that
\[
	\mathbb{E}\big[ V_{h+1}(s') \big] = {P}_{h,s,a}V_{h+1}. 
\]
To further encourage exploration, the algorithm proposed in \cite{jin2018q} --- which shall be abbreviated as \UCBQ or {\sf UCB-Q-Hoeffding} hereafter --- augments the Q-learning update rule \eqref{eq:classical-Q-update} in each episode via an additional exploration bonus:
\begin{equation}
	\label{equ:standard-ucb-update}
	Q_h^{\UCB}(s,a) \leftarrow (1-\eta)Q_h^{\UCB}(s,a) + \eta \big\{ r_h(s,a) + \widehat{P}_{h,s,a} V_{h+1} + b_h \big\}.
\end{equation}
The bonus term $b_h \geq 0$ is chosen to be a certain upper confidence bound for $(\widehat{P}_{h,s,a} -{P}_{h,s,a})V_{h+1}$, an exploration-efficient scheme that originated from the bandit literature \citep{lai1985asymptotically,lattimore2020bandit}. The algorithm then proceeds to the next episode by executing/sampling the MDP using a greedy policy w.r.t.~the updated Q-estimate. These steps are repeated until the algorithm is terminated.

\paragraph{Q-learning with UCB exploration and reference advantage (\UCBQA).} 

The regret bounds derived for \UCBQ \citep{jin2018q}, however, fall short of being optimal, 
as they are at least a factor of $\sqrt{H}$ away from the fundamental lower bound. 
In order to further shave this  $\sqrt{H}$ factor, 
one strategy is to leverage the idea of variance reduction to accelerate convergence \citep{johnson2013accelerating,sidford2018variance,wainwright2019variance,li2020sample}. 
An instantiation of this idea for the regret setting 
is a variant of \UCBQ based on reference-advantage decomposition, which was put forward in \cite{zhang2020almost} and shall be abbreviated as \UCBQA throughout this paper. 

To describe the key ideas of \UCBQA, imagine that we are able to maintain a collection of reference values $V^{\rref}= \{V_h^{\rref} \}_{h=1}^H$, which form reasonable estimates of $V^{\star} = \{V_h^{\star}\}_{h=1}^H$ and become increasingly more accurate as the algorithm progresses. 

At each time step $h$, the algorithm adopts the following update rule
\begin{equation}
	\label{equ:principle_of_qref}
	Q_h^{\rref}(s,a) \leftarrow (1-\eta)Q_h^{\rref}(s,a) 
	+ \eta \Big\{ r_h(s,a) + 
	\hspace{-2.5em} \underset{\text{stochastic estimate of }{P}_{h,s,a} \big(V_{h+1}- V^{\rref}_{h+1}\big)}{\underbrace{ \widehat{P}_{h,s,a} \big(V_{h+1}- V^{\rref}_{h+1}\big) }} \hspace{-2.5em}
	+\big[ \widehat{ P_{h}V^{\rref}_{h+1} } \big](s,a) + b_h^{\rref} \Big\} .
\end{equation}

Two ingredients of this update rule are worth noting.
\begin{itemize}
\item
Akin to the \UCBQ case, 
we can take $\widehat{P}_{h,s,a} \big(V_{h+1}- V^{\rref}_{h+1}\big)$ to be the stochastic estimate $V_{h+1}(s')- V^{\rref}_{h+1}(s')$ if we observe a sample transition $(s,a,s')$ at time step $h$. If $V_{h+1}$ is fairly close to the reference $V_{h+1}^{\rref}$, then this stochastic term can be less volatile than the stochastic term $\widehat{P}_{h,s,a} V_{h+1}$ in \eqref{equ:standard-ucb-update}. 

\item
Additionally, the term $\widehat{P_{h}V^{\rref}_{h+1} }$ indicates an estimate of the one-step look-ahead value $P_{h}V^{\rref}_{h+1}$, which shall be computed using a batch of samples. 

The variability of $\widehat{P_{h}V^{\rref}_{h+1} }$ can be well-controlled through the use of batch data, 
at the price of an increased sample size.  
\end{itemize}
Accordingly, the exploration bonus term $b_h^{\rref}$ is taken to be an upper confidence bound for the above-mentioned two terms combined. 
Given that the uncertainty of \eqref{equ:principle_of_qref} largely stems from these two terms (which can both be much smaller than the variability in \eqref{equ:standard-ucb-update}), 
the incorporation of the reference term helps accelerate convergence.

\subsection{The proposed algorithm: \DADV}
\label{sec:dadv}

As alluded to previously, however,  
the sample size required for \UCBQA to achieve optimal regret needs to exceed a large polynomial $S^6A^4$ in the size of the state/action space. 
To overcome this sample complexity barrier, we come up with an improved variant called \DADV.

\paragraph{Motivation: early settlement of a reference value.} 

An important insight obtained from previous algorithm designs is that: in order to achieve low regret, it is desirable to maintain an estimate of $Q$-function that (i) provides an optimistic view (namely, an over-estimate) of the truth $Q^{\star}$, and (ii) mitigates the bias $Q-Q^{\star}$ as much as possible. With two additional optimistic Q-estimates in hand --- $Q_h^{\UCB}$ based on \UCBQ, and $Q_h^{\rref}$ based on the reference-advantage decomposition --- it is natural to combine them as follows to further reduce the bias without violating the optimism principle: 
\begin{equation}
	\label{equ:ref-update-vr-real}
	Q_h(s_h, a_h) \leftarrow \min \Big\{Q_h^{\rref}(s_h, a_h),\, Q_h^{\UCB}(s_h, a_h),\, Q_h(s_h, a_h) \Big\}. 
\end{equation}
This is precisely what is conducted in \UCBQA. 
However, while the estimate $Q_{h}^{\rref}$ obtained with the aid of reference-advantage decomposition provides great promise, fully realizing its potential in the sample-limited regime relies on the ability to quickly {\em settle on} a desirable ``reference'' during the initial stage of the algorithm. 
This leads us to a dilemma that requires careful thinking. On the one hand, the reference value $V^{\rref}$ needs to be updated in a timely manner in order to better control the stochastic estimate of ${P}_{h,s,a} \big(V_{h+1}- V^{\rref}_{h+1}\big)$.  
On the other hand, updating $V^{\rref}$ too frequently incurs an overly large sample size burden,
as new samples need to be accumulated whenever $V^{\rref}$ is updated. 

Built upon the above insights, it is advisable to prevent frequent updating of the reference value $V^{\rref}$. In fact, it would be desirable to stop updating the reference value once a point of sufficient quality --- denoted by $V^{\rref, \mathsf{final}}$ --- has been obtained.  Locking on a reasonable reference value early on means that (a) fewer samples will be wasted on estimating a drifting target $P_{h}V^{\rref}_{h+1} $, and (b) all ensuing samples can then be dedicated to estimating the key quantity of interest $ P_{h}V^{\rref, \mathsf{final}}_{h+1} $. 

\begin{remark}
	In \citet{zhang2020almost}, the algorithm \UCBQA requires collecting $\widetilde{O}\big( SAH^6 \big) $ samples {\em for each state} before settling on the reference value, which inevitably contributes to the large burn-in cost. 
\end{remark}

\paragraph{The proposed \DADV algorithm.}

We now propose a new model-free algorithm that allows for early settlement of the reference value. 
A few key ingredients are as follows.
\begin{itemize}

	\item {\em An auxiliary sequence based on LCB.}  
In addition to the two optimistic Q-estimates $Q_h^{\rref}$ and $Q_h^{\UCB}$ described previously, 
we intend to maintain another {\em pessimistic} estimate $Q_h^{\LCB}\leq Q_h^{\star}$ using the subroutine \texttt{update-lcb-q}, 
based on lower confidence bounds (LCBs). 
We will also maintain the corresponding value function $V_h^{\LCB}$, which lower bounds $V_h^{\star}$.

\item {\em Termination rules for reference updates.} 
	With $V_h^{\LCB}\leq V^{\star}_h$ in place, the updates of the references (lines \ref{eq:line-number-15}-\ref{eq:line-number-18} of Algorithm~\ref{algo:vr}) are designed to terminate when
\begin{equation}\label{equ:stop-condition}
V_h(s_h) \leq V_h^{\LCB}(s_h) + 1  \leq V_h^\star(s_h) + 1.
\end{equation}

Note that $V_h^{\rref}$ keeps tracking the value of $V_h$ before it stops being updated. In effect, when the additional condition in lines \ref{eq:line-number-15} is violated and thus \eqref{equ:stop-condition} is satisfied,  we claim that it is unnecessary to update the reference $V_h^{\rref}$ afterwards, since it is of sufficient quality (being close enough to the optimal value $V_h^\star$) and further drifting the reference does not appear beneficial. As we will make it rigorous shortly, this reference update rule is sufficient to ensure that $|V_{h}-V_h^{\rref}|\leq 2$ throughout the execution of the algorithm,  which in turn suggests that the standard deviation of $\widehat{P}_{h,s,a}(V_{h+1}-V_{h+1}^{\rref})$ might be $O(H)$ times smaller than that of $\widehat{P}_{h,s,a}V_{h+1}$ (i.e., the stochastic term used in \eqref{eq:classical-Q-update} of \UCBQ). This is a key observation that helps shave the addition factor $H$ in the regret bound of \UCBQ.

\item {\em Update rules for $Q_h^{\UCB}$ and $Q_h^{\rref}$.} 
The two optimistic Q-estimates $Q_h^{\UCB}$ and $Q_h^{\rref}$ are updated using the subroutine \texttt{update-ucb-q} (following the standard Q-learning with Hoeffding bonus \citep{jin2018q}) and \texttt{update-ucb-q-advantage}, respectively. 
Note that $Q_h^{\rref}$ continues to be updated even after $V_h^{\rref}$ is no longer updated. 

\end{itemize}

\begin{algorithm}[t]
 \textbf{Parameters:} some universal constant $\cb>0$ and probability of failure $\delta \in(0,1)$; \\
\textbf{Initialize}  $Q_h(s, a), Q^{\UCB}_h(s, a), Q^{\rref}_h(s,a) \leftarrow H$; $V_h(s), V^{\rref}_h(s) \leftarrow H$; $Q^{\LCB}_h(s, a) \leftarrow 0$; $V^{\LCB}_h(s) \leftarrow 0$; $N_h(s, a)\leftarrow 0$; $\refmu_h(s, a), \refsg_h(s, a), \advmu_h(s, a), \advsg_h(s, a), \delta^{\rref}_h(s, a), B^{\rref}_h(s, a) \leftarrow 0$; and $u_\mathrm{ref}(s) = \mathsf{True}$ for all $(s, a, h) \in \cS\times \cA \times [H]$. \\
\For{Episode $ k = 1$ \KwTo $K$}{
    Set initial state $s_1 \leftarrow s_1^k$. \\
    \For{Step $ h = 1$ \KwTo $H$}{
        Take action $a_h=\pi_h^k(s_h) = \arg\max_a Q_h(s_h, a)$, 
	and draw $ s_{h+1} \sim P_h(\cdot \mymid s_h, a_h)$. \label{line:dadv1} {\small\color{blue}\tcp{sampling}}   
        	$N_h(s_h, a_h) \leftarrow N_h(s_h, a_h) + 1$; $n \leftarrow N_h(s_h,a_h)$. {\small\color{blue}\tcp{update the counter} }   
		$\eta_n \leftarrow \frac{H + 1}{H + n}$.  
		{\small\color{blue}\tcp{update the learning rate}} 		
  
        $Q_h^{\UCB}(s_h, a_h)\leftarrow \texttt{update-ucb-q()}$.  {\small\color{blue}\tcp{run \UCBQ; see Algorithm~\ref{algo:subroutine}}} \normalsize
        $Q_h^{\LCB}(s_h, a_h)\leftarrow \texttt{update-lcb-q()}$.  {\small\color{blue}\tcp{run \LCBQ; see Algorithm~\ref{algo:subroutine}}} \normalsize

         $Q_h^{\rref} (s_h, a_h)\leftarrow\texttt{update-ucb-q-advantage()}.  $    \small {\color{blue}\tcp{estimate $Q_h^{\rref}$; see Algorithm~\ref{algo:subroutine}}} \normalsize

     {\small\color{blue}\tcc{update Q-estimates using all estimates in hand, and update value estimates}} \normalsize
    $Q_h(s_h, a_h) \leftarrow \min\big\{Q^{\rref}_h(s_h, a_h),\, Q_h^{\UCB}(s_h, a_h),\, Q_h(s_h, a_h)\big\}$. \label{eq:line-number-11}\\
    $V_{h}(s_h) \leftarrow \max_a Q_h(s_h, a)$. \label{eq:line-number-12}\\
    $V_{h}^{\LCB}(s_{h}) \leftarrow \max\left\{\max_a Q_{h}^{\LCB}(s_{h}, a), V_{h}^{\LCB}(s_{h})\right\}$. \\

     {\small\color{blue}\tcc{update reference values}} \normalsize 
    \If{$V_{h}(s_h) - V_{h}^{\LCB}(s_{h}) > 1$ \label{eq:line-number-15}}
	 {$V^{\rref}_h(s_h) \leftarrow V_{h}(s_h)$.  \label{eq:line-number-16} \\}
    \ElseIf{$u_{\mathrm{ref}}(s_h) = \mathsf{True}$ \label{eq:line-number-17}}
    {$V^{\rref}_h(s_h) \leftarrow V_{h}(s_h)$, \qquad $u_{\mathrm{ref}}(s_h) = \mathsf{False}$. \label{eq:line-number-18} \\}
    }
 }
\caption{\DADV}
\label{algo:vr}
\end{algorithm}

\paragraph{Q-learning with reference-advantage decomposition.} 

The rest of this subsection is devoted to explaining the subroutine \texttt{update-ucb-q-advantage}, 
which produces a Q-estimate $Q^{\rref}$ based on the reference-advantage decomposition similar to \citet{zhang2020almost}. 
To facilitate the implementation, let us introduce the parameters associated with a reference value $V^{\rref}$, which include six different components, i.e.,
\begin{equation}
\big[  \refmu_h(s, a), \refsg_h(s, a), \advmu_h(s, a), \advsg_h(s, a), \delta^{\rref}_h(s, a), B^{\rref}_h(s, a) \big] , 
	\label{eq:defn-Wh-R}
\end{equation}
for all $(s,a,h)\in \cS\times \cA \times [H]$. 
Here 
 $\refmu_h(s, a)$ and $\refsg_h(s, a)$ estimate the running mean and 2nd moment of the reference $\big[P_{h}V^{\rref}_{h+1}\big](s, a)$;
   $\advmu_h(s, a)$ and  $\advsg_h(s, a)$ estimate the running (weighted) mean and 2nd moment of the advantage $\big[P_{h}(V_{h+1} - V_{h+1}^{\rref})\big](s, a)$;
   $B^{\rref}_h(s, a)$ aggregates the empirical standard deviations of the reference and the advantage combined; and
 last but not least,
  $\delta_h^{\rref}(s, a)$ is the temporal difference between $B^{\rref}_h(s, a)$ and its previous value.

 As alluded to previously, the Q-function estimation follows the strategy \eqref{equ:principle_of_qref} at a high level.
 Upon observing a sample transition $(s_h,a_h,s_{h+1})$, we compute the following estimates to update $Q^{\rref}(s_h, a_h)$. 
 \begin{itemize} 

 \item The term $\widehat{P}_{h,s,a} \big(V_{h+1}- V^{\rref}_{h+1}\big)$ is set to be $V_{h+1}(s_{h+1})- V^{\rref}_{h+1}(s_{h+1})$, which is an unbiased stochastic estimate of $P_{h,s,a} \big(V_{h+1}- V^{\rref}_{h+1}\big)$.

 \item The term $\big[P_{h}V^{\rref}_{h+1} \big](s,a)$  is estimated via
 $\refmur_h$ (cf.~line \ref{line:refmu_h}). 
Given that this is estimated using all previous samples, 
we expect the variability of this term to be well-controlled as the sample size increases (especially after $V^{\rref}$ is locked). 

 \item The exploration bonus $b_h^{\rref}(s, a)$ is updated using $B^{\rref}_h(s_h, a_h)$ and $\delta_h^{\rref}(s_h, a_h)$ (cf.~lines~\ref{line:bonus_1}-\ref{line:bonus_2} of Algorithm~\ref{algo:subroutine}), which is a confidence bound accounting for both the reference and the advantage. Let us also explain line~\ref{line:bonus_2} of Algorithm~\ref{algo:subroutine} a bit. If we augment the notation by letting $b_h^{\rref,n+1}(s, a)$ and $B_h^{\rref,n+1}(s, a)$ denote respectively $b_h^{\rref}(s, a)$ and $B_h^{\rref}(s, a)$ after $(s,a)$ is visited for the $n$-th time, then this line is designed to ensure that
\[
	\eta_n  b_h^{\rref,n+1}(s, a) + (1-\eta_n) B_h^{\rref,n}(s, a) \approx B_h^{\rref,n+1}(s, a). 
\]
\end{itemize}
With the above updates implemented properly, 
$Q_h^{\rref}$ provides the advantage-based update of the Q-function at time step $h$, according to the update rule \eqref{equ:principle_of_qref}.

\begin{algorithm}[t]

        \SetKwFunction{FMain}{update-ucb-q}
  \SetKwProg{Fn}{Function}{:}{}
  \Fn{\FMain{}}{

        $Q_h^{\UCB}(s_h, a_h) \leftarrow (1 - \eta_n)Q_h^{\UCB}(s_h, a_h) + \eta_n \Big(r_h(s_h, a_h) + V_{h+1}(s_{h+1}) + \cb \sqrt{\frac{H^3\log\frac{SAT}{\delta}}{n}} \Big)$. \label{line:dadv3} \\
        
        }
        
        \SetKwFunction{FMain}{update-lcb-q}
  \SetKwProg{Fn}{Function}{:}{}
  \Fn{\FMain{}}{

        $Q_h^{\LCB}(s_h, a_h) \leftarrow (1 - \eta_n)Q_h^{\LCB}(s_h, a_h) + \eta_n \Big(r_h(s_h, a_h) + V_{h+1}^{\LCB}(s_{h+1}) - \cb \sqrt{\frac{H^3\log\frac{SAT}{\delta}}{n}} \Big)$. \label{line:dadv4} \\
        
        }    

        \SetKwFunction{FMain}{update-ucb-q-advantage}
  \SetKwProg{Fn}{Function}{:}{}
  \Fn{\FMain{}}{
                         
                         \small {\color{blue}\tcc{update the moment statistics of $V^{\rref}_h$}} \normalsize
                $[\refmu_h, \refsg_h, \advmu_h, \advsg_h](s_h,a_h) \leftarrow \texttt{update-moments()}$;  \\
                
                \small {\color{blue}\tcc{update the accumulative bonus and bonus difference} }\normalsize
                 $[\delta_h^{\rref} , B_h^{\rref}](s_h, a_h) \leftarrow \texttt{update-bonus()}$; \label{line:bonus_1}

		  $b_h^{\rref} \leftarrow B_h^{\rref}(s_h, a_h) + (1-\eta_n) \frac{\delta_h^{\rref}(s_h,a_h)}{\eta_n} + \cb \frac{H^{2}\log\frac{SAT}{\delta}}{n^{3/4}} $; \label{line:bonus_2}\\

               \small {\color{blue}\tcc{update the Q-estimate based on reference-advantage decomposition} }\normalsize

            $Q^{\rref}_h(s_h, a_h) \leftarrow (1 - \eta_n)Q^{\rref}_h(s_h, a_h) + \eta_n\big(r_h(s_h, a_h) + V_{h+1}(s_{h+1}) - V^{\rref}_{h+1}(s_{h+1}) + \refmu_h(s_h, a_h) + b^{\rref}_h \big);$  \label{line:ref-q-update} \\

        }

    \SetKwFunction{FMain}{update-moments}
  \SetKwProg{Fn}{Function}{:}{}
  \Fn{\FMain{}}{

        ${\refmu_h(s_h, a_h) \leftarrow  (1-\tfrac{1}{n}) \refmu_h(s_h, a_h) + \tfrac{1}{n}V^{\rref}_{h+1}(s_{h+1})}$; \label{eq:line-mu-mean}
         {\small\color{blue}\tcp{mean of the reference} }\normalsize   \label{line:refmu_h}

        ${\refsg_h(s_h, a_h) \leftarrow  (1-\tfrac{1}{n}) \refsg_h(s_h, a_h) + \tfrac{1}{n}\big(V^{\rref}_{h+1}(s_{h+1}) \big)^2}$;
        \small {\color{blue}\tcp{$2^{\text{nd}}$ moment of the reference}} \normalsize  \label{line:refsigma_h}
                        
        $\advmu_h(s_h, a_h) \leftarrow (1-\eta_n)\advmu_h(s_h, a_h) + \eta_n \big( V_{h+1}(s_{h+1}) - V^{\rref}_{h+1}(s_{h+1}) \big)$;
        \small {\color{blue}\tcp{weighted average of the advantage}} \normalsize \label{line:advmu_h}
        
        $\advsg_h(s_h, a_h) \leftarrow  (1-\eta_n)\advsg_h (s_h, a_h) + \eta_n \big( V_{h+1}(s_{h+1}) - V^{\rref}_{h+1}(s_{h+1}) \big)^2$. 
        \small {\color{blue}\tcp{weighted $2^{\text{nd}}$ moment of the advantage} }\normalsize \label{line:advsigma_h}

    }

     \SetKwFunction{FMain}{update-bonus}
  \SetKwProg{Fn}{Function}{:}{}
  \Fn{\FMain{}}{
        
	  {$\nextb_h(s_h, a_h) \leftarrow \cb\sqrt{\frac{\log\frac{SAT}{\delta}}{n}}\Big(\sqrt{ \refsg_h(s_h, a_h) - \big( \refmu_h(s_h, a_h) \big)^2} + \sqrt{H}\sqrt{\advsg_h(s_h, a_h) - \big( \advmu_h(s_h, a_h) \big)^2} \,\Big)$;} \label{eq:line-number-19} \\
        
        $\delta_h^{\rref}(s_h,a_h) \leftarrow \nextb_h(s_h, a_h) - B_h^{\rref}(s_h, a_h) ;$ \label{eq:line-delta} \\ $B_h^{\rref}(s_h, a_h)  \leftarrow \nextb_h(s_h, a_h) .$

        }

\caption{Auxiliary functions}
\label{algo:subroutine}
\end{algorithm}

\subsection{Main results}
\label{sec:main-results}

Encouragingly, the proposed \DADV algorithm manages to achieve near-optimal regret even in the sample-limited and memory-limited regime, as formalized by the following theorem. 
\begin{theorem}
	\label{thm:finite}
	Consider any $\delta \in(0,1)$, and suppose that $\cb>0$ is chosen to be a sufficiently large universal constant. 
	Then there exists some absolute constant $C_0>0$ such that Algorithm~\ref{algo:vr} achieves
\begin{equation}
	\label{equ:regret bound}
	\mathsf{Regret}(T) \leq C_0 \left( \sqrt{H^2SAT\log^4\frac{SAT}{\delta} }+ H^6SA\log^3\frac{SAT}{\delta} \right) 
\end{equation}
with probability at least $1-\delta$. 
\end{theorem}

Theorem~\ref{thm:finite} delivers a non-asymptotic characterization of the performance of our algorithm \DADV. 
Several appealing features of the algorithm are noteworthy.

\begin{itemize}
	\item {\em Regret optimality. }
		Our regret bound \eqref{equ:regret bound} simplifies to 
\begin{equation}
	\mathsf{Regret}(T) \leq \widetilde{O} \big(\sqrt{H^2SAT}\big) 
	\label{eq:regret-optimal-main-results}
\end{equation}
as long as the sample size $T$ exceeds
\begin{equation}
	\label{eq:sample-size-LB-main-results}
	T\geq SA \, \mathrm{poly}(H). 
\end{equation}
This sublinear regret bound \eqref{eq:regret-optimal-main-results} is essentially optimal, as it coincides with the existing lower bound \eqref{eq:lower-bound-regret} modulo some logarithmic factor.

	\item  {\em Sample complexity and substantially reduced burn-in cost.} 
		As an interpretation of our theory \eqref{eq:regret-optimal-main-results}, our algorithm attains $\varepsilon$ average regret (i.e., $\frac{1}{K}\mathsf{Regret}(T) \leq \varepsilon$) with a sample complexity 
		$$ \widetilde{O}\Big( \frac{SAH^4}{\varepsilon^2} \Big) .$$  
	 Crucially, the burn-in cost \eqref{eq:sample-size-LB-main-results} 
is significantly lower than that of the state-of-the-art memory-efficient model-free algorithm \citep{zhang2020almost} (whose optimality is guaranteed only in the range $T\geq S^6A^4\,\mathrm{poly}(H)$).

	\item {\em Memory efficiency.} Our algorithm, which is model-free in nature, achieves a low space complexity $O(SAH)$.
		This is basically un-improvable for the tabular case, since even storing the optimal Q-values alone takes $O(SAH)$ units of space. 
		In comparison,  while \citet{menard2021ucb} also accommodates the sample size range \eqref{eq:sample-size-LB-main-results}, 
		the algorithm proposed therein incurs a space complexity of  $O(S^2AH)$  that is $S$ times higher than ours.

	\item {\em Computational complexity.}  An additional intriguing feature of our algorithm is its low computational complexity. The runtime of \DADV is no larger than $O(T)$, which is proportional to the time taken to read the samples.  This matches the computational cost of the model-free algorithm \UCBQ proposed in \citet{jin2018q}, 
		and is considerably lower than that of the \UCBM algorithm in \citet{menard2021ucb} (which has a computational cost of at least $O(ST)$).  

\end{itemize}

\section{Analysis}
\label{sec:analysis}

\begin{algorithm}[t]
 \textbf{Parameters:} some universal constant $\cb>0$ and probability of failure $\delta \in(0,1)$; \\
\textbf{Initialize}  $Q_h^1(s, a), Q^{\UCB, 1}_h(s, a), Q^{\rref, 1}_h(s,a) \leftarrow H$; $Q^{\LCB, 1}_h(s, a) \leftarrow 0$;
	$N_h^0(s, a)\leftarrow 0$;  
	$V_h^1(s), V^{\rref, 1}_h(s) \leftarrow H$; $ \refmu_h(s, a), \refsg_h(s, a), \advmu_h(s, a), \advsg_h(s, a), \delta^{\rref}_h(s, a), B^{\rref}_h(s, a) \leftarrow 0$; and $u_\mathrm{ref}^1(s) = \mathsf{True}$, for all $(s, a, h) \in \cS\times \cA \times [H]$. \\
\For{Episode $ k = 1$ \KwTo $K$}{
        Set initial state $s_1 \leftarrow s_1^k$. \\
    \For{Step $ h = 1$ \KwTo $H$}{
        Take action $a_h^k=\pi_h^k(s_h) = \arg\max_a Q_h^k(s_h^k, a)$, 
	and draw $ s_{h+1}^k \sim P_h(\cdot \mymid s_h^k, a_h^k)$. \label{line:dadv1-k} {\small\color{blue}\tcp{sampling}}   
        	$N_h^{k}(s_h^k, a_h^k) \leftarrow N_h^{k-1}(s_h^k, a_h^k) + 1$; $n \leftarrow N_h^{k}(s_h^k, a_h^k)$. {\small\color{blue}\tcp{update the counter} }   
		$\eta_n \leftarrow \frac{H + 1}{H + n}$.  
		{\small\color{blue}\tcp{update the learning rate}} 

        $Q_h^{\UCB, k+1}(s_h^k, a_h^k)\leftarrow \texttt{update-ucb-q()} $.  {\small\color{blue}\tcp{run \UCBQ; see Algorithm~\ref{algo:subroutine}}} \normalsize
        $Q_h^{\LCB, k+1}(s_h^k, a_h^k)\leftarrow \texttt{update-lcb-q()} $.  {\small\color{blue}\tcp{run \LCBQ; see Algorithm~\ref{algo:subroutine}}} \normalsize

                $Q_h^{\rref, k+1}(s_h^k, a_h^k)\leftarrow\texttt{update-ucb-q-advantage()}.$  \small {\color{blue}\tcp{estimate $Q_h^{\rref}$; see Algorithm~\ref{algo:subroutine}}} \normalsize\label{line-Q-ref-W-update}

     {\small\color{blue}\tcc{update Q-estimates using all estimates in hand, and update value estimates}} \normalsize
    $Q_h^{k+1}(s_h^k, a_h^k) \leftarrow \min\big\{Q^{\rref, k+1}_h(s_h^k, a_h^k),\, Q_h^{\UCB, k+1}(s_h^k, a_h^k),\, Q_h^k(s_h^k, a_h^k)\big\};$\label{eq:line-number-11-k} \\
    $V_{h}^{k+1}(s_h^k) \leftarrow \max_a Q_h^{k+1}(s_h^k, a)$. \label{eq:line-number-12-k}\\

    $V_{h}^{\LCB, k+1}(s_{h}^k) \leftarrow \max\left\{\max_a Q_{h}^{\LCB,k+1}(s_{h}^k, a), \, V_{h}^{\LCB,k}(s_{h}^k)\right\}$. \label{eq:line-number-13-k}\\

     {\small\color{blue}\tcc{update reference values}} \normalsize 
    \If{$V_{h}^{k+1}(s_h^k) - V_{h}^{\LCB, k+1}(s_{h}^k) > 1$ \label{eq:line-number-15-k}}{
	$V^{\rref, k+1}_h(s_h^k) \leftarrow V_{h}^{k+1}(s_h^k)$, \qquad $u_{\mathrm{ref}}^{k+1}(s_h^k) = \mathsf{True}$;  \label{eq:line-number-16-k} \\
    }
    \ElseIf{$u_{\mathrm{ref}}^k(s_h^k) = \mathsf{True}$ \label{eq:line-number-17-k}}
    {$V^{\rref, k+1}_h(s_h) \leftarrow V^{k+1}_{h}(s_h)$, \qquad $u_{\mathrm{ref}}^{k+1}(s_h^k) = \mathsf{False}$. \label{eq:line-number-18-k} \\}
    }
 }	
    \caption{\DADV (a rewrite of Algorithm~\ref{algo:vr} that specifies dependency on $k$)}
\label{algo:vr-k}
\end{algorithm}

In this section, we outline the main steps needed to prove our main result in Theorem~\ref{thm:finite}.

\subsection{Preliminaries: basic properties about learning rates}

Before continuing, let us first state some basic facts regarding the learning rates. 
Akin to \citet{jin2018q},  the proposed algorithm adopts the linearly rescaled learning rate  
\begin{equation}
	\eta_n = \frac{H+1}{H+n}
	\label{eq:eta-n-definition}
\end{equation}
for the $n$-th visit of a state-action pair at any time step $h$. 
For notation convenience, we further introduce two sequences of related quantities defined for any integer $N\geq 0$ and $n \geq 1$:  
\begin{equation}
	\label{equ:learning rate notation}
	\eta_n^N \defn \begin{cases} \eta_n \prod_{i = n+1}^N(1-\eta_i), & \text{if }N>n, \\ \eta_n,  & \text{if }N=n, \\
	0, &\text{if } N < n \end{cases}
		\qquad \text{and} \qquad    
	\eta_0^N \defn \begin{cases} \prod_{i=1}^N(1-\eta_i) =0 , & \text{if }N>0, \\ 1, & \text{if }N=0. \end{cases}	
\end{equation}
As can be easily verified, we have
\begin{align}
	\sum_{n=1}^N \eta_n^N =  \begin{cases} 1,\qquad  & \text{if }N > 0, \\ 0, & \text{if }N=0.  \end{cases}
	\label{eq:sum-eta-n-N}
\end{align}
The following properties play an important role in the analysis.
\begin{lemma}
	\label{lemma:property of learning rate} 
	For any integer $N>0$, the following properties hold:
\begin{subequations}
\label{eq:properties-learning-rates}
\begin{align}
	& \frac{1}{N^a}  \le\sum_{n=1}^{N}\frac{\eta_{n}^{N}}{n^a}\le\frac{2}{N^a}, \qquad \mbox{for all}\quad  \frac{1}{2} \leq a \leq 1,
	\label{eq:properties-learning-rates-12}\\
\max_{1\le n\le N} & \eta_{n}^{N}\le\frac{2H}{N},\qquad\sum_{n=1}^{N}(\eta_{n}^{N})^{2}\le\frac{2H}{N}, \qquad
	 \sum_{N=n}^{\infty}\eta_{n}^{N}\le1+\frac{1}{H}.
	 \label{eq:properties-learning-rates-345}
\end{align}

\end{subequations}

\end{lemma}

\begin{proof} See Appendix~\ref{sec:proof-properties-learning-rates}. \end{proof}

\subsection{Additional notation used in the proof}
\label{sec:additional-notation}

In order to enable a more concise description of the algorithm,  
we have suppressed the dependency of many quantities on the episode number $k$ in Algorithms~\ref{algo:vr} and \ref{algo:subroutine}. 
This, however, becomes notationally inconvenient when presenting the proof. As a consequence, 
we shall adopt, throughout the analysis, a more complete set of notation, detailed below.  
\begin{itemize}
	\item $(s_h^k, a_h^k)$: the state-action pair encountered and chosen at time step $h$ in the $k$-th episode.

	\item $k_h^n(s, a)$: the index of the episode in which $(s, a)$ is visited for the $n$-th time at time step $h$; for the sake of conciseness, we shall sometimes use the shorthand $k^n =  k_h^n(s, a)$ whenever it is clear from the context.

	\item $k_h^n(s)$: the index of the episode in which state $s$ is visited for the $n$-th time at time step $h$;
		we might sometimes abuse the notation by abbreviating $k^n =  k_h^n(s)$.

	\item $P_h^k \in \{0,1\}^{1\times |\cS|}$: the empirical transition at time step $h$ in the $k$-th episode, namely, 
		\begin{equation}
			P_h^k (s) = \ind\big( s = s_{h+1}^k \big).
			\label{eq:P-hk-defn-s}
		\end{equation}
\end{itemize}
In addition, for several parameters of interest in Algorithm~\ref{algo:vr}, we introduce the following set of augmented notation. 
\begin{itemize}
	\item $N_h^k(s, a)$ denotes $N_h(s, a)$ {\em by the end} of the $k$-th episode; for the sake of conciseness, we shall often abbreviate $N^k = N_h^k(s, a)$ or $N^k = N_h^k(s_h^k,a_h^k)$ (depending on which result we are proving). 

	\item $ Q_h^k(s,a)$, $V_h^k(s)$, and $Q_h^{\UCB, k}(s, a)$ denote respectively  $Q_h(s,a)$, $V_h(s)$ and $Q_h^{\UCB}(s, a)$  {\em at the beginning} of the $k$-th episode.

	\item $Q_h^{\LCB, k}(s, a)$ and $V_h^{\LCB, k}(s)$ denote respectively $Q_h^{\LCB}(s, a)$ and $V_h^{\LCB}(s)$ {\em at the beginning} of the $k$-th episode. 
	
	\item $Q_h^{\rref, k}(s, a)$, $V_h^{\rref, k}(s)$ and $u_{\mathrm{ref}}^k(s)$ denote respectively $Q_h^{\rref}(s, a)$, $V_h^{\rref}(s)$ and $u_{\mathrm{ref}}(s)$ {\em at the beginning} of the $k$-th episode.

	\item $\big[\mu^{\re, k}_h, \sigma^{\re, k}_h, \mu^{\adv, k}_h, \sigma^{\adv, k}_h, \delta^{\rref, k}_h, B^{\rref, k}_h \big]$ denotes  $\big[\mu^{\re}_h, \sigma^{\re}_h, \mu^{\adv}_h, \sigma^{\adv}_h, \delta^{\rref}_h, B^{\rref}_h \big]$ {\em at the beginning} of the $k$-th episode.

	\end{itemize}

Further, for any matrix $P=[P_{i,j}]_{1\leq i\leq m, 1\leq j\leq n}$, 
we define $\|P\|_1\coloneqq \max_{1\leq i\leq m} \sum_{j=1}^n |P_{i,j}|$. 
For any vector $V= [V_i]_{1\leq i\leq n}$, we define its $\ell_{\infty}$ norm as 
$\|V\|_{\infty} \coloneqq \max_{1\leq i\leq n} |V_i|$.  
We often overload scalar functions and expressions to take vector-valued arguments, with the understanding that they are applied in an entrywise manner. For example, for a vector $x=[x_i]_{1\leq i\leq n}$, we denote $x^2 =[x_i^2]_{1\leq i\leq n}$. For any two vectors $x=[x_i]_{1\leq i\leq n}$ and $y=[y_i]_{1\leq i\leq n}$, the notation $ {x}\leq {y}$ (resp.~$x\geq {y}$) means
$x_{i}\leq y_{i}$ (resp.~$x_{i}\geq y_{i}$) for all $1\leq i\leq n$. 
For any given vector $V \in \mathbb{R}^{S}$, we 
 define the variance parameter w.r.t.~$P_{h,s,a}$ (cf.~\eqref{eq:defn-P-hsa}) as follows
\begin{equation} \label{lemma1:equ2}
	\Var_{h, s, a}(V) \defn \mathop{\mathbb{E}}\limits_{s' \sim P_{h,s,a}} \Big [\big(V(s') -  P_{h,s, a}V \big)^2\Big] 
	= P_{h,s,a} \big(V^{ 2}\big) - \big(P_{h,s,a}V \big)^2.
\end{equation}
Finally, let $\mathcal{X}\defn \big( S, A,  H , T, \frac{1}{\delta}\big)$. 
The notation $f(\mathcal{X})\lesssim g(\mathcal{X})$ (resp.~$f(\mathcal{X})\gtrsim g(\mathcal{X})$) means that there exists a universal constant $C_{0}>0$ such that $ f(\mathcal{X}) \leq C_{0} g(\mathcal{X}) $ (resp.~$ f(\mathcal{X}) \geq C_{0} g(\mathcal{X})$);
the notation $f(\mathcal{X})\asymp g(\mathcal{X})$ means that $f(\mathcal{X})\lesssim g(\mathcal{X})$ and $f(\mathcal{X})\gtrsim g(\mathcal{X})$ hold simultaneously.

\subsection{Key properties of Q-estimates and auxiliary sequences}
\label{subsec:key_properties}

In this subsection, we introduce several key properties of our Q-estimates and value estimates, which play a crucial role in the proof of Theorem~\ref{thm:finite}. The proofs for this subsection are deferred to Appendix~\ref{sec:key_properties_Q}.

\paragraph{Properties of the Q-estimate $Q_h^{k}$: monotonicity and optimism.} We first make an important observation regarding the monotonicity of the value estimates $Q_h^{k}$ and $V_h^k$. To begin with, it is straightforward to see that the update rule in Algorithm~\ref{algo:vr-k}  (cf.~line \ref{eq:line-number-11-k}) ensures the following monotonicity property: 
\begin{subequations}
\label{eq:Qh-Vh-k-monotone-123}
\begin{align}
	Q_h^{k +1}(s, a) \le Q_h^{k}(s, a) \qquad\text{for all }( s, a,k,h)\in \cS\times\cA \times [K] \times [H] ,
	\label{eq:Qh-k-monotone-123}
\end{align}
which combined with line \ref{eq:line-number-12-k} of Algorithm~\ref{algo:vr-k} leads to monotonicity of $V_h(s)$ as follows:  
\begin{equation}
	\label{eq:monotonicity_Vh}
	V_h^{k +1}(s) = Q_h^{k +1}\big(s, \pi_h^{k+1}(s) \big) \le Q_h^{k}\big(s, \pi_h^{k+1}(s) \big) \le V_h^{k}(s).
\end{equation}
\end{subequations}
Moreover, by virtue of the update rule in line~\ref{eq:line-number-11-k} of Algorithm~\ref{algo:vr-k}, 
we can immediately obtain (via induction) the following useful property
\begin{align}\label{equ:property-Q-ref}
	Q_h^{\rref, k}(s,a) \geq Q_h^{k}(s,a) \qquad\text{for all } (k,h,s,a) \in [K]  \times [H] \times \cS \times \cA.
\end{align}

In addition, $Q_h^{k}$ and $V^{k}_h$ form an ``optimistic view'' of $Q_h^{\star}$ and $V^{\star}_h$, respectively, as asserted by the following lemma. 
\begin{lemma} \label{lem:Qk-lower}
	Consider any $\delta \in (0, 1)$. Suppose that $\cb >0$ is some sufficiently large constant.
	Then with probability at least $1-\delta$, 
\begin{equation} 
	\label{eq:Qk-lower}
	Q^{k}_h(s, a) \ge Q_h^{\star}(s, a)
	\qquad \text{and} \qquad
	V^{k}_h(s) \geq V^{\star}_h(s)
\end{equation}
hold simultaneously for all $(s,a, k, h) \in \mathcal{S} \times \mathcal{A} \times [K] \times [H]$.
\end{lemma}
Lemma~\ref{lem:Qk-lower} implies that $Q_h^{k}$ (resp.~$V_h^k$) is a pointwise upper bound on $Q_h^{\star}$ (resp.~$V_h^\star$). 
Taking this result together with the non-increasing property \eqref{eq:Qh-Vh-k-monotone-123}, we see that $Q_h^{k}$ (resp.~$V_h^k$) becomes an increasingly tighter estimate of $Q_h^{\star}$ (resp.~$V_h^\star$) as the number of episodes $k$ increases. This important fact forms the basis of the subsequent proof, allowing us to replace $V_h^\star$ with $V_h^k$ when upper bounding the regret.   
Combining Lemma~\ref{lem:Qk-lower} with \eqref{equ:property-Q-ref},  we can straightforwardly see that with probability at least $1-\delta$:
\begin{align}\label{equ:optimism-ref-Q}
	Q_h^{\rref, k}(s,a) \geq Q_h^{\star}(s,a) \qquad\text{for all } (k,h,s,a) \in [K]  \times [H] \times \cS \times \cA.
\end{align}

\paragraph{Properties of the Q-estimate $Q^{\LCB, k}_h$: pessimism and proximity.} 
In parallel, we formalize the fact that $Q^{\LCB, k}_h$ and $V^{\LCB,k}_h$ provide a ``pessimistic view'' of $Q^\star_h$ and $V^\star_h$, respectively. 
Furthermore, it becomes increasingly more likely for $Q^{\LCB, k}_h$ and $Q^{k}_h$ to stay close to each other as $k$ increases, 
which indicates that the confidence interval that contains 
the optimal value $Q_h^{\star}$ becomes shorter and shorter. 
These properties are summarized in the following lemma. 
\begin{lemma} \label{lem:Qk-lcb}
	Consider any $\delta \in (0, 1)$, and suppose that $\cb >0$ is some sufficiently large constant. 
	Then with probability at least $1-\delta$, 
\begin{equation} 
	\label{eq:Qk-lcb-upper}
	Q^{\LCB, k}_h(s, a) \le Q_h^{\star}(s, a)
	\qquad \text{and} \qquad
	V^{\LCB, k}_h(s) \le V^{\star}_h(s)
\end{equation}
hold for all $(s,a, k, h) \in \mathcal{S} \times \mathcal{A}\times [K] \times [H]$, and 
\begin{align}
	\label{eq:main-lemma}
	\sum_{h=1}^H \sum_{k=1}^K \mathds{1}\left(Q^{k}_h(s_h^k, a_h^k) - Q^{\LCB, k}_h(s_h^k, a_h^k) > \varepsilon \right) 
	\lesssim \frac{H^6SA\log\frac{SAT}{\delta}}{\varepsilon^2}
\end{align}
holds for all $\varepsilon \in (0,H]$.
\end{lemma}
Interestingly,  the upper bound \eqref{eq:main-lemma} only scales logarithmically in the number $K$ of episodes,
thus implying the closeness of $Q^{\LCB, k}_h$ and $Q^{k}_h$ for a large fraction of episodes.
Note that it is straightforward to ensure the monotonicity property of $V^{\LCB,k}_h$ from the update rule in Algorithm~\ref{algo:vr-k}  (cf.~line \ref{eq:line-number-13-k}):
\begin{align} 
V^{\LCB,k+1}_h(s) \geq   V^{\LCB,k}_h(s)  \qquad\text{for all }( s, k,h)\in \cS \times [K] \times [H] ,
	\label{eq:V-LCB-h-k-monotone}
\end{align}
which in conjunction with \eqref{eq:Qk-lcb-upper}, implies that $V^{\LCB,k}_h(s)$ gets closer to $V^{\star}_h(s)$ as the number of episodes $k$ increases.
Together with the monotonicity of $V_h^k$ (cf.~\eqref{eq:monotonicity_Vh}), an important consequence is that the reference value $V_h^{\rref}$ will stop being updated shortly after the following condition is met for the first time (according to lines~\ref{eq:line-number-15}-\ref{eq:line-number-18} of Algorithm~\ref{algo:vr})
\begin{equation}
  	V_h^{k}(s) \leq V^{\LCB,k}_h(s) + 1 \leq V^{\star}_h(s) + 1
	\qquad \mbox{for all } s\in\cS. 
	\label{eq:V-star-violation-Vh-123}
\end{equation}

\paragraph{Properties of the reference $V_h^{\rref,k}$.} 

The above fact ensures that $V_h^{\rref,k}$ will not be updated too many times. 
In fact, its value stays reasonably close to $V_h^{k}$ even after being locked to a fixed value, which ensures its fidelity as a reference signal. Moreover, the aggregate difference between $V^{\rref, k}_{h}$ and the final reference $V^{\rref, K}_{h}$ 
over the entire trajectory can be bounded in a reasonably tight fashion (owing to \eqref{eq:main-lemma}), as formalized in the next lemma. These properties play a key role in reducing the burn-in cost of the proposed algorithm.

 \begin{lemma}\label{lem:Vr_properties}
Consider any $\delta \in (0, 1)$. Suppose that $\cb >0$ is some sufficiently large constant.
	 Then with probability exceeding $1-\delta$, one has 
 \begin{equation}
	 \label{eq:close-ref-v}
	\big| V^{k}_{h}(s) - V^{\rref, k}_{h}(s) \big|  \le 2
\end{equation}
	 for all $(k, h, s) \in [K] \times [H] \times \cS$, and
\begin{align}  \label{eq:Vr_lazy}
&\sum_{h=1}^H\sum_{k=1}^K \left(V^{\rref, k}_{h}(s^{k}_{h}) - V^{\rref, K}_{h}(s^{k}_{h})\right)  \nonumber \\
&\quad \leq H^2S + \sum_{h=1}^H\sum_{k=1}^K \Big(Q^{k}_{h}(s_{h}^k, a_{h}^k) - Q^{\LCB, k}_{h}(s_{h}^k, a_{h}^k)  \Big)
	\ind\Big(Q^{k}_{h}(s_{h}^k, a_{h}^k) - Q^{\LCB, k}_{h}(s_{h}^k, a_{h}^k)  > 1\Big) \\
	& \quad \lesssim H^6SA\log\frac{SAT}{\delta}.
\end{align}
\end{lemma}
In words, Lemma~\ref{lem:Vr_properties} guarantees that (i) our value function estimate and the reference value are always sufficiently close (cf.~\eqref{eq:close-ref-v}), and (ii) the aggregate difference between $V_h^{\rref, k}$ and the final reference value $V^{\rref, K}_{h}$  is nearly independent of the sample size $T$ (except for some logarithmic scaling). 

\subsection{Main steps of the proof}\label{sec:main-pipeline}

We are now ready to embark on the regret analysis for \DADV, which consists of multiple steps as follows.

\paragraph{Step 1: regret decomposition.}
Lemma~\ref{lem:Qk-lower} allows one to upper bound the regret as follows 
\begin{equation} 	
	\label{eq:regret_optimism}
	\mathsf{Regret}(T): = \sum_{k=1}^{K}\big(V^{\star}_{1}(s_1^k)-V^{\pi^{k}}_{1}(s_1^k)\big) \leq \sum_{k=1}^{K}\big( V_1^k(s^k_1) - V_1^{\pi^k}(s^k_1) \big).
\end{equation}
To continue, it boils down to controlling $V_1^k(s^k_1) - V_1^{\pi^k}(s^k_1) $. 
Towards this end, we intend to examine $V_h^k(s^k_h) - V_h^{\pi^k}(s^k_h) $ across all time steps $1\leq h\leq H$, 
which admits the following decomposition: 
\begin{align}
V^{k}_h(s^k_h) & - V_h^{\pi^k}(s^k_h) 
 \overset{\mathrm{(i)}}{=}Q^{k}_h(s^k_h, a^k_h) - Q^{\pi^k}_h(s^k_h, a^k_h) \nonumber\\
& =Q^{k}_h(s^k_h, a^k_h) - Q_h^{\star}(s^k_h, a^k_h) + Q_h^{\star}(s^k_h, a^k_h) - Q^{\pi^k}_h(s^k_h, a^k_h)  \nonumber\\
&  \overset{\mathrm{(ii)}}{=} Q^{k}_h(s^k_h, a^k_h) - Q_h^{\star}(s^k_h, a^k_h) + P_{h, s^k_h, a^k_h}\big(V^{\star}_{h+1} - V^{\pi^k}_{h+1}\big) \nonumber \\
&  \overset{\mathrm{(iii)}}{=} Q^{k}_h(s^k_h, a^k_h) - Q_h^{\star}(s^k_h, a^k_h) + \big(P_{h, s^k_h, a^k_h} - P^k_h\big)\big(V^{\star}_{h+1} - V^{\pi^k}_{h+1}\big)  + V^{\star}_{h+1}(s^k_{h+1}) - V^{\pi^k}_{h+1}(s^k_{h+1})\nonumber \\
& \leq  Q^{\rref, k}_h(s^k_h, a^k_h) - Q_h^{\star}(s^k_h, a^k_h) + \big(P_{h, s^k_h, a^k_h} - P^k_h\big)\big(V^{\star}_{h+1} - V^{\pi^k}_{h+1}\big)  + V^{\star}_{h+1}(s^k_{h+1}) - V^{\pi^k}_{h+1}(s^k_{h+1}) . 
\label{equ:decompose}
\end{align} 
Here, (i) holds since $\pi_h^k$ is a greedy policy w.r.t.~$Q_h^k$ and $\pi_h^k(s^k_h) = a^k_h$,  (ii) comes from the Bellman equations 
$$Q^{\pi^k}_h(s, a) -Q^{\star}_h(s, a)  = \big( r_h(s, a) + P_{h, s, a} V^{\pi^k}_{h+1} \big) - \big( r_h(s, a) + P_{h, s, a} V^{\star}_{h+1} \big) =P_{h, s, a} \big( V^{\pi^k}_{h+1}-V^{\star}_{h+1} \big)  ,$$ 
(iii) follows from $P^k_h ( V_{h+1}^{\star} - V_{h+1}^{\pi^k})  = V_{h+1}^{\star}(s^k_{h+1}) - V_{h+1}^{\pi^k}(s^k_{h+1})$ (see the notation \eqref{eq:P-hk-defn-s}), 
whereas the last inequality comes from \eqref{equ:property-Q-ref}.
Summing \eqref{equ:decompose} over $1\leq k\leq K$ and making use of Lemma~\ref{lem:Qk-lower}, we reach at
\begin{align} 
	\sum_{k=1}^{K}\big(V^{\star}_{h}(s_h^k)-V^{\pi^{k}}_{h}(s_h^k)\big)
	&\leq \sum_{k = 1}^K \big(V^k_h(s^k_h) - V_h^{\pi^k}(s^k_h)\big) \nonumber\\
	& \leq \sum_{k = 1}^K \big(Q^{\rref, k}_h(s^k_h, a^k_h) - Q_h^{\star}(s^k_h, a^k_h)\big) + \sum_{k = 1}^K \big(P_{h, s^k_h, a^k_h} - P^k_h \big) \big(V^{\star}_{h+1} - V^{\pi^k}_{h+1} \big)  \nonumber \\
& \qquad + \sum_{k = 1}^K \big(V^{\star}_{h+1}(s^k_{h+1}) - V^{\pi^k}_{h+1}(s^k_{h+1})\big). 	\label{eq:Vk-Qk}
\end{align}
This allows us to establish a connection between $\sum_{k}\big(V^{\star}_{h}(s_1^k)-V^{\pi^{k}}_{h}(s_h^k)\big)$ for step $h$ and 
$\sum_{k } \big(V^{\star}_{h+1}(s^k_{h+1}) - V^{\pi^k}_{h+1}(s^k_{h+1})\big)$ for step $h+1$.

\paragraph{Step 2: managing regret by recursion.}

 The regret can be further manipulated by leveraging the update rule of $Q^{\rref, k}_h$ as well as recursing over the time steps $h=1,2,\cdots, H$ with the terminal condition $V^{k}_{H+1} = V^{\pi^k}_{H+1} = 0$. This leads to a key decomposition as summarized in the lemma below, whose proof is provided in Appendix~\ref{proof:lem:Qk-upper}.

\begin{lemma} \label{lem:Qk-upper}
Fix $\delta \in (0,1)$. Suppose that $\cb >0$ is some sufficiently large constant. Then with probability at least $1-\delta$,  one has
\begin{align}
 \sum_{k = 1}^K \big(V^k_1(s^k_1) - V_1^{\pi^k}(s^k_1)\big) \leq \mathcal{R}_1 + \mathcal{R}_2 +\mathcal{R}_3 , 
	\label{eq:lemma-4-three-terms}
\end{align}
where
\begin{subequations} 
\label{eq:summary_of_terms}
\begin{align}
\mathcal{R}_1 & :=   \sum_{h = 1}^H \left(1+\frac{1}{H}\right)^{h-1} \bigg( HSA + 8\cb H^2 (SA)^{3/4} K^{1/4} \log\frac{SAT}{\delta} + \sum_{k = 1}^K \big(P_{h, s^k_h, a^k_h} - P^k_h \big) \big(V^{\star}_{h+1} - V^{\pi^k}_{h+1} \big) \bigg)  , \label{eq:expression_R1} \\
\mathcal{R}_2 &:=    \sum_{h = 1}^H \left(1+\frac{1}{H}\right)^{h-1}\sum_{k = 1}^{K}B^{\rref, k}_h(s^k_h, a^k_h)   , \label{eq:expression_R2} \\
	\mathcal{R}_3 &:=  \sum_{h = 1}^H \sum_{k = 1}^K\lambda^{k}_h\bigg( (P^k_{h} -  P_{h, s^k_h, a^k_h})(V^{\star}_{h+1} - V^{\rref, k}_{h+1}) + \frac{\sum_{i =1}^{N^{k}_h(s^k_h, a^k_h)} \big(V^{\rref, k^i_h(s^k_h, a^k_h)}_{h+1}(s^{k^i_h(s^k_h, a^k_h)}_{h+1}) - P_{h, s^k_h, a^k_h}V^{\rref, k}_{h+1}\big)}{N^{k}_h(s^k_h, a^k_h)}\bigg) , 
\label{eq:expression_R3}
\end{align}
\end{subequations}
with $$\lambda^{k}_h \defn \left(1+\frac{1}{H}\right)^{h-1}\sum_{n = N^{k}_h(s^k_h, a^k_h)}^{N^{K-1}_h(s^k_h, a^k_h)} \eta^{n}_{N^{k}_h(s^k_h, a^k_h)}.$$
\end{lemma}

This lemma attempts to upper bound the target quantity $\sum_{k = 1}^K \big(V^k_1(s^k_1) - V_1^{\pi^k}(s^k_1)\big)$ via three terms (see \eqref{eq:lemma-4-three-terms}). 
Informally, these terms reflect (i) the influence of the initialization as well as the finite-sample uncertainty of $P_h^k(V_{h+1}^{\star}-V_{h+1}^{\pi^k})$, 
(ii) the influence of the size of the bonus terms, and (iii) the discrepancy term when the running value iterates are replaced by the reference values. 
As we shall see in the analysis, 
the key in obtaining these terms lies in properly expanding the component $\sum_{k = 1}^K \big(Q^{\rref, k}_h(s^k_h, a^k_h) - Q_h^{\star}(s^k_h, a^k_h)\big)$  in \eqref{eq:Vk-Qk}, 
as well as applying induction across all $h=1,\cdots,H$.


\paragraph{Step 3: controlling the terms in \eqref{eq:summary_of_terms} separately.} 

As it turns out, each of the terms in \eqref{eq:summary_of_terms} can be well controlled. We provide the bounds for these terms in the following lemma.  

\begin{lemma}\label{lemma:bound_of_everything}
Consider any $\delta \in (0,1)$. With probability at least $1-\delta $, we have the following upper bounds: 
\begin{align*}
	\mathcal{R}_1 & \leq  C_{\mathrm{r}} \bigg\{ \sqrt{H^2SAT \log\frac{SAT}{\delta}} + H^{4.5}SA\log^2\frac{SAT}{\delta}  \bigg\} , \\
\mathcal{R}_2 & \leq  C_{\mathrm{r}} \bigg\{\sqrt{H^2SAT\log\frac{SAT}{\delta}} + H^4SA\log^2\frac{SAT}{\delta} \bigg\}, \\
\mathcal{R}_3 & \leq  C_{\mathrm{r}} \bigg\{  \sqrt{H^2SA T\log^4\frac{SAT}{\delta} }+ H^6SA\log^3\frac{SAT}{\delta} \bigg\} 
\end{align*}
for some universal constant $C_{\mathrm{r}}>0$. 
\end{lemma}

In order to derive the above bounds, the main strategy is to apply the Bernstein-type concentration inequalities carefully, 
and to upper bound the sum of variance in a careful manner.  
The proofs are deferred to Appendix~\ref{sec-proof-lemma:everything}.

\paragraph{Step 4: putting all this together.} 

We now have everything in place to establish our main result. 
Taking the preceding bounds in Lemma~\ref{lemma:bound_of_everything} together with \eqref{eq:summary_of_terms}, we see that with probability exceeding $1-\delta$, one has 
\begin{align*}
\mathsf{Regret}(T) \leq \mathcal{R}_1+ \mathcal{R}_2 + \mathcal{R}_3 \lesssim  \sqrt{H^2SAT\log^4\frac{SAT}{\delta} }+ H^6 SA\log^3\frac{SAT}{\delta} 
\end{align*}
as claimed.

\section{Discussion}

In this paper, we have proposed a novel model-free RL algorithm --- tailored to online episodic settings --- that attains near-optimal regret  $\widetilde{O}(\sqrt{H^2SAT})$ and near-minimal memory complexity $O(SAH)$ at once. Remarkably, the near-optimality of the algorithm comes into effect as soon as the sample size rises above $O( SA\,\mathrm{poly}(H))$, 
which has significantly improved upon the sample size requirements (or burn-in cost) for any prior regret-optimal model-free algorithm (based on the definition of the model-free algorithm in \citep{jin2018q}). 
Given that online data collection could be expensive, time-consuming, or high-stakes in a variety of contemporary applications (e.g., clinical trials, autonomous driving, online advertisement), 
reducing  burn-in sample sizes compromising sample optimality is crucial in enabling sample-efficient solutions in these sample-constrained applications.

The results in the current paper naturally suggest a number of possible extensions and directions for future investigation. We close the paper by listing a few of them. 
\begin{itemize}
\item While the proposed algorithm provably enables minimal burn-in cost in terms of the dependency on $S$ and $A$, 
	our current theory falls short of delivering optimal horizon dependency of the burn-in cost. 
	More specifically, even though our burn-in cost improves upon the state-of-the-art theory for sample-optimal model-free algorithms by a factor of at least $S^5A^3H^{18}$ (see \citep{zhang2020almost}), 
	the way we cope with the dependency on $H$ remains inadequate. 
	This calls for more refined analysis tools to optimize the horizon dependency.

\item The present paper focuses primarily on MDPs with non-stationary probability transition kernels. 
	Another important scenario is concerned with MDPs with stationary transition kernels (i.e., the case where $P_h$ is identical across different $h$). 
	It is worth noting that the algorithm developed herein is incapable of attaining optimal regret for the stationary case (i.e., the resulting regret might be off by a factor of $\sqrt{H}$). 
	While our analysis already contains multiple key ingredients that are useful for analyzing the stationary case, 
	how to complete the picture is non-trivial, which we leave for future work.

\item Admittedly, even though we are now able to settle the sample size dependency on the state-action space, 
the size of $SA$ might remain prohibitively large in many modern RL applications. 
As a result, parsimonious function representation/approximation of the underlying MDP 
is needed in order to further reduce the sample complexity.  
Prominent examples of this kind include linearly parameterized or realizable MDPs \citep{jin2020provably,du2019good,li2021sample}. 
We hope that the method and analysis framework developed herein might inspire further development of sample-efficient algorithms 
that can effectively accommodate low-dimensional function approximation.

\end{itemize}


\section*{Acknowledgment}
L.~Shi and Y.~Chi are supported in part
by the grants ONR N00014-19-1-2404, NSF CCF-2106778, CCF-2007911 and DMS-2134080.  
L.~Shi is also gratefully supported by the Leo Finzi Memorial Fellowship, Wei Shen and Xuehong Zhang Presidential Fellowship, and
Liang Ji-Dian Graduate Fellowship at Carnegie Mellon University. 
Y.~Chen is supported in part by the the Alfred P.~Sloan Research Fellowship, 
the AFOSR grant FA9550-22-1-0198, 
the ONR grant N00014-22-1-2354,  and the NSF grants CCF-2221009, CCF-1907661, 
IIS-2218713 and IIS-2218773, and the Google Research Scholar Award. 
Part of this work was done while Y.~Chen and G.~Li were visiting the Simons Institute for the Theory of Computing.



\appendix

\section{Freedman's inequality}

\subsection{A user-friendly version of Freedman's inequality}

Due to the Markovian structure of the problem, our analysis relies heavily on the celebrated Freedman's inequality
\citep{freedman1975tail,tropp2011freedman}, which extends the Bernstein's inequality to accommodate martingales.
For ease of reference, we state below a user-friendly version of Freedman's
inequality as provided in \citet[Section C]{li2021tightening}.
\begin{theorem}[Freedman's inequality]\label{thm:Freedman}
Consider a filtration $\mathcal{F}_0\subset \mathcal{F}_1 \subset \mathcal{F}_2 \subset \cdots$,
and let $\mathbb{E}_{k}$ stand for the expectation conditioned
on $\mathcal{F}_k$. 
Suppose that $Y_{n}=\sum_{k=1}^{n}X_{k}\in\mathbb{R}$,
where $\{X_{k}\}$ is a real-valued scalar sequence obeying
\[
	\left|X_{k}\right|\leq R\qquad\text{and}\qquad\mathbb{E}_{k-1} \big[X_{k}\big]=0\quad\quad\quad\text{for all }k\geq1
\]
for some quantity $R<\infty$. 
We also define
\[
W_{n}\coloneqq\sum_{k=1}^{n}\mathbb{E}_{k-1}\left[X_{k}^{2}\right].
\]
In addition, suppose that $W_{n}\leq\sigma^{2}$ holds deterministically for some given quantity $\sigma^2<\infty$.
Then for any positive integer $m \geq1$, with probability at least $1-\delta$
one has
\begin{equation}
\left|Y_{n}\right|\leq\sqrt{8\max\Big\{ W_{n},\frac{\sigma^{2}}{2^{m}}\Big\}\log\frac{2m}{\delta}}+\frac{4}{3}R\log\frac{2m}{\delta}.\label{eq:Freedman-random}
\end{equation}
\end{theorem}

\subsection{Application of Freedman's inequality}

We now develop several immediate consequences of Freedman's inequality, which lend themseleves well to our context. 
Before proceeding, we recall that  $N_h^i(s,a)$ denotes the number of times that the state-action pair $(s,a)$ has been visited at step $h$ by the end of the $i$-th episode, and $k_h^n(s,a)$ stands for the episode index when $(s,a)$ is visited at step $h$ for the $n$-th time (see Section~\ref{sec:additional-notation}).

Our first result is concerned with a martingale concentration bound as follows. 

\begin{lemma}
	\label{lemma:martingale-union-all}
	Let $\big\{ W_{h}^{i} \in \mathbb{R}^S \mid 1\leq i\leq K, 1\leq h \leq H+1 \big\}$ and $\big\{u_h^i(s,a,N)\in \mathbb{R} \mid 1\leq i\leq K, 1\leq h \leq H+1 \big\}$  be a collections of vectors and scalars, respectively, and suppose that they obey the following properties:
	\begin{itemize}
		\item $W_{h}^{i}$ is fully determined by the samples collected up to the end of the $(h-1)$-th step of the $i$-th episode;  
		\item $\|W_h^i\|_{\infty}\leq C_\mathrm{w}$; 
		\item $u_h^i(s,a, N)$  is fully determined by the samples collected up to the end of the $(h-1)$-th step of the $i$-th episode, and a given positive integer $N\in[K]$;
		\item $0\leq u_h^i(s,a, N) \leq C_{\mathrm{u}}$;
		\item $0\leq \sum_{n=1}^{N_h^k(s,a)} u_h^{k_h^n(s,a)}(s,a, N) \leq 2$.  
		
	\end{itemize}
	In addition, consider the following sequence  
	\begin{align}
		X_i (s,a,h,N) &\defn u_h^i(s,a, N) \big(P_h^{i} - P_{h,s,a}\big) W_{h+1}^{i} 
		\ind\big\{ (s_h^i, a_h^i) = (s,a)\big\},
		\qquad 1\leq i\leq K,  
	\end{align}
	with $P_h^{i}$ defined in \eqref{eq:P-hk-defn-s}. 
	Consider any $\delta \in (0,1)$. Then with probability at least $1-\delta$, 
	\begin{align}
		& \left|\sum_{i=1}^k X_i(s,a,h,N) \right| \notag\\
		& \quad \lesssim \sqrt{C_{\mathrm{u}} \log^2\frac{SAT}{\delta}}\sqrt{\sum_{n = 1}^{N_h^k(s,a)} u_h^{k_h^n(s,a)}(s,a,N) \Var_{h, s,a} \big(W_{h+1}^{k_h^n(s,a)}  \big)} + \left(C_{\mathrm{u}} C_{\mathrm{w}} + \sqrt{\frac{C_{\mathrm{u}}}{N}} C_{\mathrm{w}}\right) \log^2\frac{SAT}{\delta}
	\end{align}
	holds simultaneously for all $(k, h, s, a, N) \in [K] \times [H] \times \cS \times \cA \times [K]$.  
\end{lemma}
\begin{proof}
For the sake of notational convenience, we shall abbreviate $X_i (s,a,h,N)$ as $X_i$ throughout the proof of this lemma, as long as it is clear from the context.  
The plan is to apply Freedman's inequality (cf.~Theorem~\ref{thm:Freedman}) to control the term $\sum_{i=1}^k X_i$ of interest. 

Consider any given $(k, h, s, a, N)\in [K]\times [H]\times \cS\times \cA\times [K]$. It can be easily verified that 
$$ \mathbb{E}_{i-1}\left[ X_i \right] = 0,$$ 
where $\mathbb{E}_{i-1}$ denotes the expectation conditioned on 
	everything happening up to the end of the $(h-1)$-th step of the $i$-th episode. 
Additionally, we make note of the following crude bound: 
\begin{align}
	\big| X_i  \big| &\leq u_h^i(s,a,N) \Big|\big( P^{i}_{h} - P_{h, s,a} \big) W^{i}_{h+1}   \Big| \notag\\
	& \leq u_h^i(s,a,N) \Big( \big\| P^{i}_{h}  \big\|_1 + \big\| P_{h, s,a} \big\|_1 \Big) \big\| W^{i}_{h+1}  \big\|_{\infty}  
	\leq 2C_\mathrm{w} C_\mathrm{u}, 
	\label{lemma7-all:equ1}
\end{align}
which results from the assumptions $\|W^i_{h+1}\|_{\infty} \leq C_{\mathrm{w}}$, $0 \leq u_h^i(s,a,N) \leq C_{\mathrm{u}}$ as well as the basic facts  $\big\|P^{i}_{h} \big\|_1 =\big\| P_{h, s,a}\big\|_1=1$. 
To continue, recalling the definition of the variance parameter in \eqref{lemma1:equ2}, we obtain 
\begin{align}
\sum_{i=1}^{k}\mathbb{E}_{i-1}\left[\big|X_{i}\big|^{2}\right] & =\sum_{i=1}^{k}\big(u_{h}^{i}(s,a, N)\big)^{2} \ind\big\{(s_{h}^{i},a_{h}^{i})=(s,a)\big\}
	\mathbb{E}_{i-1}\left[\big|(P_{h}^{i}-P_{h,s,a})W_{h+1}^{i}\big|^{2}\right] \notag\\
 & =\sum_{n=1}^{N_{h}^{k}(s,a)}\big(u_{h}^{k_h^{n}(s,a)}(s,a,N)\big)^{2}\Var_{h,s,a}\big(W_{h+1}^{k_h^{n}(s,a)}\big) \notag\\
 & \leq C_{\mathrm{u}}\Bigg(\sum_{n=1}^{N_{h}^{k}(s,a)}u_{h}^{k_h^{n}(s,a)}(s,a,N)\Bigg)\big\| W_{h+1}^{k_h^{n}(s,a)}\big\|_{\infty}^{2}\notag\\
 & \leq 2C_{\mathrm{u}}C_{\mathrm{w}}^{2},\label{lemma7-all:equ3}
\end{align}
where the inequalities hold true due to the assumptions $\|W_{h}^i\|_{\infty} \le C_{\mathrm{w}}$, $0 \leq u_h^i(s,a,N) \leq C_{\mathrm{u}}$,
	and $0\leq \sum_{n=1}^{N_h^k(s,a)} u_h^{k_h^n(s,a)}(s,a,N) \leq 1$.

With \eqref{lemma7-all:equ1} and \eqref{lemma7-all:equ3} in place, 
we can invoke Theorem~\ref{thm:Freedman} (with $m=\lceil \log_2 N \rceil$) and take the union bound over all $(k, h, s, a, N) \in [K] \times [H] \times \cS \times \cA \times [K]$ to show that:  
with probability at least $1- \delta$, 
\begin{align*}
 & \bigg|\sum_{i=1}^{k}X_{i}\bigg|\lesssim\sqrt{\max\Bigg\{ C_{\mathrm{u}}\sum_{n=1}^{N_{h}^{k}(s,a)}u_{h}^{k_{h}^{n}(s,a)}(s,a,N)\Var_{h,s,a}\big(W_{h+1}^{k_{h}^{n}(s,a)}\big),\frac{C_{\mathrm{u}}C_{\mathrm{w}}^{2}}{N}\Bigg\}\log\frac{SAT^{2}\log N}{\delta}}\notag\\
 & \qquad+C_{\mathrm{u}}C_{\mathrm{w}}\log\frac{SAT^{2}\log N_{h}^{k}}{\delta}\nonumber\\
 & \lesssim\sqrt{C_{\mathrm{u}}\log^{2}\frac{SAT}{\delta}}\sqrt{\sum_{n=1}^{N_{h}^{k}(s,a)}u_{h}^{k_{h}^{n}(s,a)}(s,a,N)\Var_{h,s,a}\big(W_{h+1}^{k_{h}^{n}(s,a)}\big)}+\left(C_{\mathrm{u}}C_{\mathrm{w}}+\sqrt{\frac{C_{\mathrm{u}}}{N}}C_{\mathrm{w}}\right)\log^{2}\frac{SAT}{\delta}\notag
\end{align*}
holds simultaneously for all $(k, h, s, a, N) \in [K] \times [H] \times \cS \times \cA \times [K]$. 
\end{proof}

The next result is concerned with martingale concentration bounds for another type of sequences of interest.

\begin{lemma}\label{lemma:martingale-union-all2}
	Let $\big\{ N(s,a,h) \in [K] \mid (s,a,h)\in \cS\times \cA\times [H]\big\}$ be a collection of positive integers,
	and let $\{c_h: 0\leq c_h\leq e, h\in[H]\}$ be a collection of fixed and bounded universal constants.
	Moreover, let $\big\{ W_{h}^{i} \in \mathbb{R}^S \mid 1\leq i\leq K, 1\leq h \leq H+1 \big\}$ and $\big\{u_h^i(s_h^i,a_h^i)\in \mathbb{R} \mid 1\leq i\leq K, 1\leq h \leq H+1\big\}$ represent respectively a collection of random vectors and scalars, which obey the following properties. 
	\begin{itemize}
		\item $W_{h}^{i}$ is fully determined by the samples collected up to the end of the $(h-1)$-th step of the $i$-th episode;  
		\item $\|W_h^i\|_{\infty}\leq C_\mathrm{w}$ and $W_h^i \geq 0$; 
		\item $u_h^i(s_h^i, a_h^i)$ is fully determined by the integer $N(s_h^i, a_h^i,h)$ and all samples collected up to the end of the $(h-1)$-th step of the $i$-th episode; 
		\item $0\leq u_h^i(s_h^i, a_h^i) \leq C_{\mathrm{u}}$.
	\end{itemize}
	Consider any $\delta \in (0,1)$, and introduce the following sequences  
	\begin{align}
		X_{i,h} &\defn u_h^i(s_h^i,a_h^i)\big(P_h^{i} - P_{h,s_h^i,a_h^i}\big) W_{h+1}^{i},
		\qquad &1\leq i\leq K, 1\leq h \leq H+1,\\
		Y_{i,h} &\coloneqq c_h \big(P_h^{i} - P_{h,s_h^i,a_h^i}\big) W_{h+1}^{i},\qquad &1\leq i\leq K, 1\leq h \leq H+1.
	\end{align}
	Then with probability at least $1-\delta$, 
	\begin{align}
\left|\sum_{h=1}^{H}\sum_{i=1}^{K}X_{i,h}\right| & \lesssim\sqrt{C_{\mathrm{u}}^{2}\sum_{h=1}^{H}\sum_{i=1}^{K}\mathbb{E}_{i,h-1}\left[\big|(P_{h}^{i}-P_{h,s_{h}^{i},a_{h}^{i}})W_{h+1}^{i}\big|^{2}\right]\log\frac{T^{HSA}}{\delta}}+C_{\mathrm{u}}C_{\mathrm{w}}\log\frac{T^{HSA}}{\delta} \notag\\
 & \lesssim\sqrt{C_{\mathrm{u}}^{2}C_{\mathrm{w}}\sum_{h=1}^{H}\sum_{i=1}^{K}\mathbb{E}_{i,h-1}\left[P_{h}^{i}W_{h+1}^{i}\right]\log\frac{T^{HSA}}{\delta}}+C_{\mathrm{u}}C_{\mathrm{w}}\log\frac{T^{HSA}}{\delta} \notag\\
\left|\sum_{h=1}^{H}\sum_{i=1}^{K}Y_{i,h}\right| & \lesssim\sqrt{TC_{\mathrm{w}}^{2}\log\frac{1}{\delta}}+C_{\mathrm{w}}\log\frac{1}{\delta}\notag
	\end{align}
	holds simultaneously for all possible collections $\{N(s,a,h) \in [K] \mid (s,a,h) \in \cS\times \cA\times [H]\}$. 
\end{lemma}

\begin{proof}
This lemma can be proved by Freedman's inequality (cf.~Theorem~\ref{thm:Freedman}). 

\begin{itemize}
	\item
We start by controlling the first term of interest $\sum_{h=1}^H \sum_{i=1}^K X_{i,h}$.
As can be easily seen, $a_h^i = \arg\max  Q_h^i(s_h^i,a)$ is fully determined by what happens before step $h$ of the $i$-th episode.
Consider any given $\big\{ N(s,a,h) \in [K] \mid (s,a,h)\in \cS\times \cA\times [H]\big\}$. It is readily seen that  
$$ \mathbb{E}_{i, h-1}\left[ X_i \right] =\mathbb{E}_{i, h-1}\left[ u_h^i(s_h^i,a_h^i)\big(P_h^{i} - P_{h,s_h^i,a_h^i}\big) W_{h+1}^{i} \right]= 0,$$ 
where $\mathbb{E}_{i, h-1}$ denotes the expectation conditioned on 
everything happening before step $h$ of the $i$-th episode. 
In addition, we make note of the following crude bound: 
\begin{align}
	\big| X_{i,h}  \big| &\leq u_h^i(s_h^i,a_h^i) \Big|\big(P_h^{i} - P_{h,s_h^i,a_h^i}\big) W_{h+1}^{i}  \Big| \notag\\
	& \leq u_h^i(s_h^i,a_h^i) \Big( \big\| P^{i}_{h}  \big\|_1 + \big\| P_{h,s_h^i,a_h^i} \big\|_1 \Big) \big\| W^{i}_{h+1}  \big\|_{\infty}  
	\leq 2C_\mathrm{w} C_\mathrm{u}, 
	\label{lemma8-all:equ1}
\end{align}
which arises from the assumptions $\|W^i_{h+1}\|_{\infty} \leq C_{\mathrm{w}}$, $0 \leq u_h^i(s,a,N) \leq C_{\mathrm{u}}$ together with the basic facts  $\big\|P^{i}_{h} \big\|_1 =\big\| P_{h, s_h^i,a_h^i}\big\|_1=1$. 
Additionally, we can calculate that 
\begin{align}
\sum_{h=1}^{H}\sum_{i=1}^{K}\mathbb{E}_{i,h-1}\left[\big|X_{i,h}\big|^{2}\right] & =\sum_{h=1}^{H}\sum_{i=1}^{K}\big(u_{h}^{i}(s_{h}^{i},a_{h}^{i})\big)^{2}\mathbb{E}_{i,h-1}\left[\big|(P_{h}^{i}-P_{h,s_{h}^{i},a_{h}^{i}})W_{h+1}^{i}\big|^{2}\right]\notag\\
& \overset{(\mathrm{i})}{\leq} C_{\mathrm{u}}^{2}\sum_{h=1}^{H}\sum_{i=1}^{K}\mathbb{E}_{i,h-1}\left[\big|(P_{h}^{i}-P_{h,s_{h}^{i},a_{h}^{i}})W_{h+1}^{i}\big|^{2}\right] \label{lemma8-all:interval1}\\
 & \leq C_{\mathrm{u}}^{2}\sum_{h=1}^{H}\sum_{i=1}^{K}\mathbb{E}_{i,h-1}\left[\big|P_{h}^{i}W_{h+1}^{i}\big|^{2}\right]\notag\\
 & \overset{(\mathrm{ii})}{=}C_{\mathrm{u}}^{2}\sum_{h=1}^{H}\sum_{i=1}^{K}\mathbb{E}_{i,h-1}\left[P_{h}^{i}\big(W_{h+1}^{i}\big)^{2}\right]\notag\\
 & \overset{(\mathrm{iii})}{\leq}C_{\mathrm{u}}^{2}\sum_{h=1}^{H}\sum_{i=1}^{K}\big\| W_{h+1}^{i}\big\|_{\infty}\mathbb{E}_{i,h-1}\left[P_{h}^{i}W_{h+1}^{i}\right] \notag\\
 & \overset{(\mathrm{iv})}{\leq}C_{\mathrm{u}}^{2}C_{\mathrm{w}}\sum_{h=1}^{H}\sum_{i=1}^{K}\mathbb{E}_{i,h-1}\left[P_{h}^{i}W_{h+1}^{i}\right] \label{lemma8-all:equ3-middle}\\
 & \leq C_{\mathrm{u}}^{2}C_{\mathrm{w}}\sum_{h=1}^{H}\sum_{i=1}^{K}\big\| W_{h+1}^{i}\big\|_{\infty} 
 \overset{\mathrm{(v)}}{\leq} HKC_{\mathrm{u}}^{2}C_{\mathrm{w}}^2 = TC_{\mathrm{u}}^{2}C_{\mathrm{w}}^2. \label{lemma8-all:equ3}
\end{align}
Here, (i) holds true due to the assumption $0 \leq u_h^i(s_h^i,a_h^i) \leq C_{\mathrm{u}}$, 
(ii) is valid since $P_h^i$ only has one non-zero entry (cf.~\eqref{eq:P-hk-defn-s}), 
	(iii) relies on the assumptions that $W_h^i$ is non-negative,
whereas (iv) and (v) follow since $\|W_{h}^i\|_{\infty} \le C_{\mathrm{w}}$,

With \eqref{lemma8-all:equ1}, \eqref{lemma8-all:equ3-middle} and \eqref{lemma8-all:equ3} in mind, 
we can invoke Theorem~\ref{thm:Freedman} (with $m=\lceil \log_2 T \rceil$) and take the union bound over all possible collections $\big\{ N(s,a,h) \in [K] \mid (s,a,h)\in \cS\times \cA\times [H]\big\}$ --- which has at most $K^{HSA}$ possibilities --- to show that:  
with probability at least $1- \delta$, 
\begin{align*}
\bigg|\sum_{h=1}^{H}\sum_{i=1}^{k}X_{i,h}\bigg| & \lesssim\sqrt{\max\left\{ C_{\mathrm{u}}^{2}\sum_{h=1}^{H}\sum_{i=1}^{K}\mathbb{E}_{i,h-1}\left[\big|(P_{h}^{i}-P_{h,s_{h}^{i},a_{h}^{i}})W_{h+1}^{i}\big|^{2}\right],\frac{TC_{\mathrm{u}}^{2}C_{\mathrm{w}}^{2}}{2^{m}}\right\} \log\frac{K^{HSA}\log T}{\delta}}\\
 & \qquad+C_{\mathrm{u}}C_{\mathrm{w}}\log\frac{K^{HSA}\log T}{\delta}\\
 & \lesssim\sqrt{C_{\mathrm{u}}^{2}\sum_{h=1}^{H}\sum_{i=1}^{K}\mathbb{E}_{i,h-1}\left[\big|(P_{h}^{i}-P_{h,s_{h}^{i},a_{h}^{i}})W_{h+1}^{i}\big|^{2}\right]\log\frac{T^{HSA}}{\delta}}+C_{\mathrm{u}}C_{\mathrm{w}}\log\frac{T^{HSA}}{\delta}\\
 & \lesssim\sqrt{C_{\mathrm{u}}^{2}C_{\mathrm{w}}\sum_{h=1}^{H}\sum_{i=1}^{K}\mathbb{E}_{i,h-1}\left[P_{h}^{i}W_{h+1}^{i}\right]\log\frac{T^{HSA}}{\delta}}+C_{\mathrm{u}}C_{\mathrm{w}}\log\frac{T^{HSA}}{\delta}
\end{align*}
holds simultaneously for all $\big\{ N(s,a,h) \in [K] \mid (s,a,h)\in \cS\times \cA\times [H]\big\}$.
\item Then we turn to control the second term $\left|\sum_{h=1}^H \sum_{i=1}^K Y_{i,h}\right|$ of interest. Similar to $\left|\sum_{h=1}^H \sum_{i=1}^K X_{i,h}\right|$, we have
\begin{align*}
	|Y_{i,h}| &\leq 2e C_{\mathrm{w}}, \\
	\sum_{h=1}^{H}\sum_{i=1}^{K}\mathbb{E}_{i,h-1}\left[\big|Y_{i,h}\big|^{2}\right] & \leq e^2T C_{\mathrm{w}}^2.
\end{align*}
Invoke Theorem~\ref{thm:Freedman} (with $m=1$) to arrive at
\begin{align}
	\bigg| \sum_{h=1}^H \sum_{i=1}^{K} Y_{i,h} \bigg| \lesssim \sqrt{T C_{\mathrm{w}}^{2} \log\frac{1}{\delta}}  + C_{\mathrm{w}} \log\frac{1}{\delta}
\end{align}
with probability at least $1-\delta$.
\end{itemize}
\end{proof}

\section{Proof of Lemma~\ref{lemma:property of learning rate}}
\label{sec:proof-properties-learning-rates}

First of all, the properties in \eqref{eq:properties-learning-rates-345} follow directly from \citet[Lemma 4.1]{jin2018q}. 
Therefore, it suffices to establish the property in \eqref{eq:properties-learning-rates-12}, which forms the remainder of this subsection.

When $N =1$, the statement holds trivially since 
\begin{equation*}
	\sum_{n=1}^{N}\frac{\eta_n^{N}}{n^a} = \eta_1^1 = 1 \in [1,2]. 
\end{equation*}
Now suppose that $N\geq 2$. Making use of the basic relation $\eta_n^{N} = (1-\eta_N) \eta_n^{N-1}$ for all $n =1 , \cdots, N-1$, 
we observe the following identity: 
\begin{equation}
  \sum_{n=1}^{N}\frac{\eta_n^{N}}{n^a} = \frac{\eta_N}{N^a} +  (1-\eta_N) \sum_{n=1}^{N-1} \frac{\eta_n^{N-1}}{n^a}.
	\label{eq:basic-identity-sum-eta}
\end{equation}
We now prove the property in \eqref{eq:properties-learning-rates-12} by induction. Suppose for the moment that the property holds for $N-1$, namely, 
\begin{equation}
	\frac{1}{(N-1)^a} \le \sum_{n=1}^{N-1}\frac{\eta_n^{N-1}}{n^a} \le \frac{2}{(N-1)^a} .
	\label{eq:induction-hypothesis-N-1}
\end{equation}
Then it is readily seen from \eqref{eq:basic-identity-sum-eta} that
\begin{equation}
  	\sum_{n=1}^{N}\frac{\eta_n^{N}}{n^a} = \frac{\eta_N}{N^a} +  (1-\eta_N) \sum_{n=1}^{N-1} \frac{\eta_n^{N-1}}{n^a} 
	\geq \frac{\eta_N}{N^a} + \frac{1-\eta_N}{(N -1)^a} \geq \frac{\eta_N}{N^a} + \frac{1-\eta_N}{N^a} =  \frac{1}{N^a}, 
\end{equation}
where the first inequality comes from \eqref{eq:induction-hypothesis-N-1}. 
Similarly, one can upper bound
\begin{align*}
	\sum_{n=1}^{N}\frac{\eta_{n}^{N}}{n^{a}}&=\frac{\eta_{N}}{N^{a}}+(1-\eta_{N})\sum_{n=1}^{N-1}\frac{\eta_{n}^{N-1}}{n^{a}}\overset{(\mathrm{i})}{\leq}\frac{\eta_{N}}{N^{a}}+\frac{2(1-\eta_{N})}{\left(N-1\right)^{a}}\overset{(\mathrm{ii})}{=}\frac{H+1}{N^{a}(H+N)}+\frac{2(N-1)^{1-a}}{H+N}\\&\overset{(\mathrm{iii})}{\leq}\frac{H+1}{N^{a}(H+N)}+\frac{2N^{1-a}}{H+N}=\frac{1}{N^{a}}\left(\frac{H+1}{H+N}+\frac{2N}{H+N}\right) \overset{\mathrm{(iv)}}{\leq}\frac{2}{N^{a}},
\end{align*}

where (i) arises from \eqref{eq:induction-hypothesis-N-1}, (ii) follows from the choice $\eta_N = \frac{H +1}{H + N}$, 
(iii) holds since $a\leq 1$, 
and (iv) follows since $H\geq 1$. 
Consequently, we can immediately establish the advertised property \eqref{eq:properties-learning-rates-12} by induction.

\section{Proof of key lemmas in Section~\ref{subsec:key_properties}}
\label{sec:key_properties_Q}

\subsection{Proof of Lemma~\ref{lem:Qk-lower}}
\label{proof:lem:Qk-lower}

To begin with, suppose that we can prove 
\begin{align}
	Q^{k}_h(s, a) \ge Q_h^{\star}(s, a)
	\qquad \text{for all }(k,h,s,a) \in [K] \times [H]\times \mathcal{S} \times \cA.
	\label{eq:Q-kh-upper-bound-Qh-star}
\end{align}
Then this property would immediately lead to the claim w.r.t.~$V^{k}_h$, namely, 
\begin{equation}
	\label{equ:v-upper-vstar}
	V^{k}_h(s) \ge Q^{k}_h\big(s, \pi_h^{\star}(s)\big) \ge Q^{\star}_h\big(s, \pi_h^{\star}(s)\big) = V_h^{\star}(s)
	\qquad \text{for all }(k,h,s) \in [K] \times [H]\times \mathcal{S}.
\end{equation} 
As a result, it suffices to focus on justifying the claim \eqref{eq:Q-kh-upper-bound-Qh-star}, which we shall accomplish by induction.
\begin{itemize}
	\item {\em Base case.} 
		Given that the initialization obeys $Q^{1}_h(s, a) = H \ge Q_h^{\star}(s, a)$ 
		for all $(h,s,a)\in [H] \times \cS \times \cA$, 
		the claim \eqref{eq:Q-kh-upper-bound-Qh-star} holds trivially when $k=1$.

	\item {\em Induction.} Suppose that the claim \eqref{eq:Q-kh-upper-bound-Qh-star} holds all the way up to the $k$-th episode,
		and we wish to establish it for the $(k+1)$-th episode as well.  
		To complete the induction argument, it suffices to justify 
\begin{equation*}
	\min \Big\{Q^{\UCB, k+1}_h(s, a), Q^{{\rref}, k+1}_h(s, a)  \Big\} \ge Q_h^{\star}(s, a)
\end{equation*}
		according to line \ref{eq:line-number-11-k} of Algorithm~\ref{algo:vr-k}. 
		Recognizing that $Q^{\UCB, k+1}_h$ is computed via the standard \UCBQ update rule (see line \ref{line:dadv3} of Algorithm~\ref{algo:subroutine}), we can readily invoke the argument in \citet[Lemma 4.3]{jin2018q} to show that with probability at least $1-\delta$,
		$$  Q^{\UCB, k+1}_h(s, a) \ge Q_h^{\star}(s, a) $$
		holds simultaneously for all $(k, h, s, a) \in [K] \times [H] \times \mathcal{S} \times \mathcal{A}$.
		Therefore, it is sufficient to prove that 
		\begin{align}
			Q^{{\rref}, k+1}_h(s, a) \ge Q_h^{\star}(s, a). 
			\label{eq:upper-bound-reference-values}
		\end{align}
\end{itemize}
The remainder of the proof is thus devoted to justifying \eqref{eq:upper-bound-reference-values}, assuming that  the claim \eqref{eq:Q-kh-upper-bound-Qh-star} holds all the way up to $k$.

Since $Q_{h}^{\rref ,k}(s_h^k,a_h^k)$ is updated in the $k$-th episode while other entries of $Q_{h}^{\rref ,k}$ remain fixed, it suffices to verify 
$$Q^{\rref , k+1}_h(s_h^k,a_h^k) \ge Q_h^{\star}(s_h^k,a_h^k).$$ 
We remind the readers of two important short-hand notation that shall be used when it is clear from the context: 
\begin{itemize}
	\item $N_h^k = N_h^k(s_h^k,a_h^k)$ denotes the number of times that the state-action pair $(s_h^k,a_h^k)$ has been visited at step $h$ by the end of the $k$-th episode;
	\item $k^n=k_h^n(s_h^k,a_h^k)$ denotes the index of the episode in which the state-action pair $(s_h^k,a_h^k)$ is visited for the $n$-th time at step $h$. 
\end{itemize}

\paragraph{Step 1: decomposing $Q^{\rref, k+1}_h(s_h^k,a_h^k) - Q_h^{\star}(s_h^k,a_h^k)$.}

To begin with, the above definition of $N_h^k$ and $k^n$ allows us to write
\begin{equation}
	Q_{h}^{\rref ,k + 1}(s_h^k,a_h^k) = Q_{h}^{\rref ,k^{N_h^k}+1}(s_h^k,a_h^k),
	\label{eq:Qh-Qhref-equal-357}
\end{equation}
since $k^{N_h^k} = k^{N_h^k(s_h^k,a_h^k)} =k$.
According to the update rule (i.e., line~\ref{line-Q-ref-W-update} in Algorithm~\ref{algo:vr-k} and line~\ref{line:ref-q-update} in Algorithm~\ref{algo:subroutine}), we obtain 
\begin{align*}
 Q_{h}^{\rref ,k+1}(s_h^k,a_h^k) & =Q_{h}^{\rref ,k^{N_h^k}+1}(s_h^k,a_h^k)=(1-\eta_{N_h^k})Q_{h}^{\rref ,k^{N_h^k}}(s_h^k,a_h^k)\\
 & \qquad+\eta_{N_h^k}\bigg\{ r_{h}(s_h^k,a_h^k)+V_{h+1}^{k^{N_h^k}}(s_{h+1}^{k^{N_h^k}})-V_{h+1}^{\rref ,k^{N_h^k}}(s_{h+1}^{k^{N_h^k}})+\mu_{h}^{\re ,k^{N_h^k}+1}(s_h^k,a_h^k)+b_{h}^{\rref ,k^{N_h^k}+1}\bigg\} \\
 & = (1-\eta_{N_h^k})Q_{h}^{\rref ,k^{N_h^k-1}+1}(s_h^k,a_h^k) \\
& \qquad	+\eta_{N_h^k}\bigg\{ r_{h}(s_h^k,a_h^k)+V_{h+1}^{k^{N_h^k}}(s_{h+1}^{k^{N_h^k}})-V_{h+1}^{\rref ,k^{N_h^k}}(s_{h+1}^{k^{N_h^k}})+\mu_{h}^{\re ,k^{N_h^k}+1}(s_h^k,a_h^k)+b_{h}^{\rref ,k^{N_h^k}+1}\bigg\},	
\end{align*}
where the last identity again follows from our argument for justifying \eqref{eq:Qh-Qhref-equal-357}. 
Applying this relation recursively and invoking the definitions of $\eta_0^{N}$ and $\eta_n^{N}$ in \eqref{equ:learning rate notation}, 
we are left with
\begin{align}
	Q^{\rref , k+1}_h(s_h^k,a_h^k)
	& = \eta_0^{N_h^k} Q^{\rref , 1}_h(s_h^k,a_h^k)   \notag \\
	& \qquad + \sum_{n = 1}^{N_h^k} \eta_n^{N_h^k} \bigg\{ r_h(s_h^k,a_h^k) + V^{k^n}_{h+1}(s^{k^n}_{h+1}) - V^{\rref , k^n}_{h+1}(s^{k^n}_{h+1}) + \mu^{\re , k^n+1}_h(s_h^k,a_h^k) + b^{\rref , k^n+1}_h \bigg\}.
	\label{eq:Q-ref-decompose-12345}
\end{align}
Additionally, the basic relation $\eta_0^{N_h^k} + \sum_{n = 1}^{N_h^k} \eta_n^{N_h^k} = 1$ (see \eqref{equ:learning rate notation} and \eqref{eq:sum-eta-n-N}) tells us that
\begin{equation}\label{eq:decomp_Qhstar_eta}
	Q_{h}^{\star} (s_h^k,a_h^k) = \eta^{N_h^k}_0 Q_{h}^{\star} (s_h^k,a_h^k) + \sum_{n = 1}^{N_h^k} \eta^{N_h^k}_n Q_{h}^{\star} (s_h^k,a_h^k),
\end{equation}
which combined with \eqref{eq:Q-ref-decompose-12345} leads to
\begin{align}
 & Q_{h}^{\rref ,k+1}(s_h^k,a_h^k)-Q_{h}^{\star}(s_h^k,a_h^k)=\eta_{0}^{N_h^k}\big(Q_{h}^{\rref ,1}(s_h^k,a_h^k)-Q_{h}^{\star}(s_h^k,a_h^k)\big)\nonumber\\
 & \qquad+\sum_{n=1}^{N_h^k}\eta_{n}^{N_h^k}\bigg\{ r_{h}(s_h^k,a_h^k)+V_{h+1}^{k^{n}}(s_{h+1}^{k^{n}})-V_{h+1}^{\rref ,k^{n}}(s_{h+1}^{k^{n}})+\mu_{h}^{\re ,k^{n}+1}(s_h^k,a_h^k)+b_{h}^{\rref ,k^{n}+1}-Q_{h}^{\star}(s_h^k,a_h^k)\bigg\}.
	\label{equ:lemma4 vr 1}
\end{align}

To continue, invoking the Bellman optimality equation 
\begin{equation}\label{eq:bellman_opt_hk}
Q^{\star}_h(s_h^k,a_h^k) = r_h(s_h^k,a_h^k) + P_{h, s_h^k,a_h^k} V^\star_{h+1}
\end{equation}
and using the construction of $\mu^{\re }_h$ in line~\ref{eq:line-mu-mean} of Algorithm~\ref{algo:subroutine} (which is the running mean of $V^{\rref }_{h+1}$), we reach
\begin{align}
 & r_{h}(s_{h}^{k},a_{h}^{k})+V_{h+1}^{k^{n}}(s_{h+1}^{k^{n}})-V_{h+1}^{\rref,k^{n}}(s_{h+1}^{k^{n}})+\mu_{h}^{\re,k^{n}+1}(s_{h}^{k},a_{h}^{k})+b_{h}^{\rref,k^{n}+1}-Q_{h}^{\star}(s_{h}^{k},a_{h}^{k})\nonumber\\
 & \qquad=V_{h+1}^{k^{n}}(s_{h+1}^{k^{n}})-V_{h+1}^{\rref,k^{n}}(s_{h+1}^{k^{n}})+\frac{\sum_{i=1}^{n}V_{h+1}^{\rref,k^{i}}(s_{h+1}^{k^{i}})}{n}-P_{h,s_{h}^{k},a_{h}^{k}}V_{h+1}^{\star}+b_{h}^{\rref,k^{n}+1}\label{eq:dejavu}\\
 & \qquad=P_{h,s_{h}^{k},a_{h}^{k}}\left\{ V_{h+1}^{k^{n}}-V_{h+1}^{\rref,k^{n}}\right\} +\frac{\sum_{i=1}^{n}P_{h,s_{h}^{k},a_{h}^{k}}\big(V_{h+1}^{\rref,k^{i}}\big)}{n}-P_{h,s_{h}^{k},a_{h}^{k}}V_{h+1}^{\star}+b_{h}^{\rref,k^{n}+1}+\xi_{h}^{k^{n}},\nonumber\\
	& \qquad=P_{h,s_{h}^{k},a_{h}^{k}}\bigg\{ V_{h+1}^{k^{n}}-V_{h+1}^{\star}+\frac{\sum_{i=1}^{n}\big(V_{h+1}^{\rref,k^{i}}-V_{h+1}^{\rref,k^{n}}\big)}{n}\bigg\}+b_{h}^{\rref,k^{n}+1}+\xi_{h}^{k^{n}}.\label{eq:dejavu-135}
\end{align}
Here, we have introduced the following quantity
\begin{equation}
\xi^{k^n}_h \coloneqq  \big(P^{k^n}_{h} - P_{h, s_h^k,a_h^k} \big)\big(V^{k^n}_{h+1} - V^{\rref , k^n}_{h+1} \big) +  \frac{1}{n}\sum_{i=1}^n \big( P^{k^i}_{h} - P_{h, s_h^k,a_h^k} \big) V^{\rref , k^i}_{h+1} , 
	\label{eq:defn-xi-kh-123}
\end{equation}
with the notation $P^{k}_{h}$ defined in \eqref{eq:P-hk-defn-s}. 
Putting \eqref{eq:dejavu-135} and \eqref{equ:lemma4 vr 1} together leads to the following decomposition
\begin{align}
Q^{\rref , k+1}_h(s_h^k,a_h^k) &- Q^{\star}_h(s_h^k,a_h^k) = \eta^{N_h^k}_0 \Big(Q^{\rref , 1}_h(s_h^k,a_h^k) - Q^{\star}_h(s_h^k,a_h^k) \Big) \nonumber\\
	&\qquad+ \sum_{n = 1}^{N_h^k} \eta^{N_h^k}_n \bigg\{ P_{h, s_h^k,a_h^k} \bigg( V^{k^n}_{h+1} - V^{\star}_{h+1} 
	+ \frac{\sum_{i =1}^n \big(V^{\rref , k^i}_{h+1} - V^{\rref , k^n}_{h+1}\big)}{n} \bigg) 
	+ b^{\rref , k^n+1}_h + \xi^{k^n}_h \bigg\}. 
	\label{equ:concise update}
\end{align}

\paragraph{Step 2: two key quantities for lower bounding $Q^{\rref , k+1}_h(s_h^k,a_h^k) - Q^{\star}_h(s_h^k,a_h^k)$.}

In order to develop a lower bound on $Q^{\rref , k+1}_h(s_h^k,a_h^k) - Q^{\star}_h(s_h^k,a_h^k)$ based on the decomposition \eqref{equ:concise update}, 
we make note of several simple facts as follows.
\begin{itemize}
	\item[(i)] The initialization satisfies $Q^{\rref , 1}_h(s_h^k,a_h^k) - Q^{\star}_h(s_h^k,a_h^k) \geq 0$.

	\item[(ii)] For any $1 \leq k^n \leq k$, one has
		\begin{equation}
			V^{k^n}_{h+1}\geq V^{\star}_{h+1},
			\label{eq:V-kn-V-star-dominate-135}
		\end{equation}
		owing to the induction hypotheses \eqref{eq:Q-kh-upper-bound-Qh-star} and \eqref{equ:v-upper-vstar} that hold up to $k$.

	\item[(iii)] For all $0 \leq i \leq n$ and any $s\in \cS$, one has 
		\begin{equation}
			V^{\rref , k^i}_{h+1}(s) - V^{\rref , k^n}_{h+1}(s) \geq 0,
			\label{eq:V-ref-monotone-135}
		\end{equation}
		which holds since the reference value $V^{\rref }_h(s)$ is monotonically non-increasing in view of the monotonicity of $V_h(s)$ in \eqref{eq:monotonicity_Vh} and the update rule in line~\ref{eq:line-number-16} of Algorithm~\ref{algo:vr-k}.

\end{itemize}
The above three facts taken collectively with \eqref{equ:concise update} allow one to drop several terms and yield
\begin{equation}
	Q^{\rref , k+1}_h(s_h^k,a_h^k) - Q_h^{\star}(s_h^k,a_h^k)  
	\geq \sum_{n = 1}^{N_h^k} \eta^{N_h^k}_n \big( b^{\rref , k^n+1}_h + \xi^{k^n}_h \big). 
	\label{eq:Qref-Q-diff-LB-123}
\end{equation}
In the sequel, we aim to establish $Q^{\rref , k+1}_h(s_h^k,a_h^k) \geq Q_h^{\star}(s_h^k,a_h^k)$ based on this inequality \eqref{eq:Qref-Q-diff-LB-123}.

As it turns out, if one could show that
\begin{equation}
	\bigg| \sum_{n = 1}^{N_h^k} \eta^{N_h^k}_n \xi_h^{k^n} \bigg| \le 
	\sum_{n = 1}^{N_h^k} \eta^{N_h^k}_n b^{\rref , k^n+1}_h ,
	\label{eq:claim-sum-eta-sum-b}
\end{equation}
then taking this together with \eqref{eq:Qref-Q-diff-LB-123} and the triangle inequality would immediately lead to the desired result
\begin{equation}
	Q^{\rref , k+1}_h(s_h^k,a_h^k) - Q_h^{\star}(s_h^k,a_h^k)  \geq  \sum_{n = 1}^{N_h^k} \eta^{N_h^k}_n  b^{\rref , k^n+1}_h  - \bigg| \sum_{n = 1}^{N_h^k} \eta^{N_h^k}_n  \xi^{k^n}_h \bigg|  \geq  0. 
\end{equation}
As a result, the remaining steps come down to justifying the claim \eqref{eq:claim-sum-eta-sum-b}.
In order to do so, we need to control the following two quantities (in view of \eqref{eq:defn-xi-kh-123})
\begin{subequations}
\label{eq:defn-I1-I2-step2}
\begin{align}
	I_1 &\defn  \sum_{n=1}^{N_h^k} \eta_n^{N_h^k} \big(P^{k^n}_{h} - P_{h, s_h^k,a_h^k} \big)\big(V^{k^n}_{h+1} - V^{\rref , k^n}_{h+1} \big),  \label{eq:defn-I1-step2} \\
	 I_2 &\defn \sum_{n=1}^{N_h^k} \frac{1}{n} \eta_n^{N_h^k} \sum_{i=1}^n \big(P^{k^i}_{h} - P_{h, s_h^k,a_h^k} \big)V^{\rref , k^i}_{h+1} 
	 \label{eq:defn-I2-step2}
\end{align}
\end{subequations}
separately, which constitutes the next two steps. As will be seen momentarily, these two terms can be controlled in a similar fashion using Freedman's inequality.

\paragraph{Step 3: controlling $I_1$.}
In the following, we intend to invoke Lemma~\ref{lemma:martingale-union-all} to control the term $I_1$ defined in \eqref{eq:defn-I1-step2}. 
To begin with, consider any $(N,h) \in [K] \times [H]$, and introduce
\begin{align}
	W_{h+1}^i \coloneqq V^{i}_{h+1} - V^{\rref , i}_{h+1}
	\qquad \text{and} \qquad
	u_h^i(s,a, N) \coloneqq \eta_{N_h^i(s,a)}^N \geq 0. \label{I1-wu-step3}
\end{align}
Accordingly, we can derive and define
\begin{align}
	\|W_{h+1}^i\|_\infty \leq \|V^{\rref , i}_{h+1}\|_{\infty} + \|V^{i}_{h+1}\|_{\infty} \leq 2H \eqqcolon C_{\mathrm{w}}, \label{I1-cw}
\end{align}
and
\begin{align}
	 \max_{N, h, s, a\in [K] \times[H]\times \cS\times \cA} \eta_{N_{h}^{i}(s,a)}^{N} \leq \frac{2H}{N} \eqqcolon C_{\mathrm{u}}, \label{I1-cu}
\end{align}
where the last inequality follows since (according to Lemma~\ref{lemma:property of learning rate} and the definition in \eqref{equ:learning rate notation}) 
\begin{align*}
	\eta_{N_{h}^{i}(s,a)}^{N} & \leq\frac{2H}{N},\qquad\text{if }1\le N_{h}^{i}(s,a)\le N;\\
	\eta_{N_{h}^{i}(s,a)}^{N} & =0,\qquad\quad ~\text{if }N_{h}^{i}(s,a)>N.
\end{align*}
Moreover, observed from \eqref{eq:sum-eta-n-N}, we have 
\begin{align}
	0 \leq \sum_{n=1}^N u_h^{k_h^n(s,a)}(s,a, N) = \sum_{n=1}^N \eta_{n}^N \leq 1 \label{eq:sum-uh-k-s-a-123}
\end{align}
holds for all $(N,s,a) \in [K] \times \cS\times \cA$.
Therefore, choosing $(N,s,a)=(N^k_h,s_h^k, a_h^k)$ and  
applying Lemma~\ref{lemma:martingale-union-all} with the quantities \eqref{I1-wu-step3} 
implies that, with probability at least $1-\delta$,  
\begin{align}
	|I_1| &= \left|\sum_{n=1}^{N_h^k} \eta_n^{N_h^k} \big(P^{k^n}_{h} - P_{h, s_h^k,a_h^k} \big)\big(V^{k^n}_{h+1} - V^{\rref , k^n}_{h+1} \big)\right| = \left|\sum_{i=1}^k X_i(s_h^k,a_h^k, h, N_h^k)\right| \nonumber\\
	&\lesssim \sqrt{C_{\mathrm{u}} \log^2\frac{SAT}{\delta}}\sqrt{\sum_{n = 1}^{N_h^k} u_h^{k^n}(s_h^k,a_h^k, N_h^k) \Var_{h, s_h^k,a_h^k} \big(W_{h+1}^{k^n}  \big)} + \left(C_{\mathrm{u}} C_{\mathrm{w}} + \sqrt{\frac{C_{\mathrm{u}}}{N}} C_{\mathrm{w}}\right) \log^2\frac{SAT}{\delta} \nonumber \\
	& \asymp \sqrt{\frac{H}{N_h^k}\log^2\frac{SAT}{\delta}}\sqrt{\sum_{n = 1}^{N_h^k} \eta^{N_h^k}_n \Var_{h, s_h^k,a_h^k} \big(V^{k^n}_{h+1} - V^{\rref , k^n}_{h+1} \big) } + \frac{H^2\log^2\frac{SAT}{\delta}}{N_h^k} \label{lemma1:equ4-sub} \\
	&  \lesssim \sqrt{\frac{H}{N_h^k}\log^2\frac{SAT}{\delta}}\sqrt{\sigma^{\adv, k^{N_h^k}+1}_h(s_h^k,a_h^k) - \big(\mu^{\adv, k^{N_h^k}+1}_h(s_h^k,a_h^k)\big)^2} + \frac{H^2\log^2\frac{SAT}{\delta}}{(N_h^k)^{3/4}}, 
	\label{lemma1:equ4}
\end{align}
where the proof of the last inequality \eqref{lemma1:equ4} needs additional explanation and is postponed to Appendix~\ref{sec:proof:eq:var-V-Vref} to streamline the presentation.

\paragraph{Step 4: controlling $I_2$.}

Next, we turn attention to the quantity $I_2$ defined in \eqref{eq:defn-I2-step2}. 
Rearranging terms in the definition \eqref{eq:defn-I2-step2}, we are left with
\begin{align*}
I_2 = \sum_{n = 1}^{N_h^k} \eta^{N_h^k}_n\frac{\sum_{i =1}^n \big(P^{k^i}_{h} - P_{h, s_h^k,a_h^k} \big)V^{\rref , k^i}_{h+1}}{n} 
	 = \sum_{i = 1}^{N_h^k} \left( \sum_{n = i}^{N_h^k} \frac{\eta^{N_h^k}_n}{n} \right) \big(P^{k^i}_{h} - P_{h, s_h^k,a_h^k}\big)V^{\rref , k^i}_{h+1}, 
\end{align*}
which can again be controlled by invoking Lemma~\ref{lemma:martingale-union-all}. To do so, we abuse the notation by taking
\begin{align}
	W_{h+1}^i \coloneqq V^{\rref, i}_{h+1} 
	\qquad \text{and} \qquad
	u_h^i(s,a, N) \coloneqq \sum_{n=N_h^i(s,a)}^N \frac{\eta_n^N}{n} \geq 0. \label{I2-wu-step4}
\end{align}
These quantities satisfy
\begin{align}
	  \big\|W_{h+1}^i\big\|_\infty \leq \big\|V^{\rref , i}_{h+1}\big\|_\infty \leq H \eqqcolon C_{\mathrm{w}} \label{I2-cw-step4}
\end{align}
and, according to Lemma~\ref{lemma:property of learning rate}, 
\begin{align}
	 \max_{N, h, s, a\in [K] \times[H]\times \cS\times \cA} \sum_{n=N_h^i(s,a)}^N \frac{\eta_n^N}{n} \leq \sum_{n = 1}^{N} \frac{\eta^{N}_n}{n} \le \frac{2}{N} \eqqcolon C_{\mathrm{u}} . \label{I2-cu-step4}
\end{align}

Then it is readily seen from \eqref{I2-cu-step4} that  
\begin{align}
	0 \leq \sum_{n=1}^N u_h^{k_h^n(s,a)}(s,a, N) \leq \sum_{n=1}^{N}\frac{2}{N}\leq 2. \label{eq:sum-uh-k-s-a-123}
\end{align}
holds for all $(N,s,a) \in [K] \times \cS\times \cA$.

With the above relations in mind, Taking $(N,s,a)=(N_h^k,s_h^k,a_h^k)$ and 
applying Lemma~\ref{lemma:martingale-union-all} w.r.t.~the quantities \eqref{I2-wu-step4} reveals that
\begin{align}
| I_2 |&= \bigg|\sum_{i = 1}^{N_h^k} \sum_{n = i}^{N_h^k} \frac{\eta^{N_h^k}_n}{n} \big(P^{k^i}_{h} - P_{h, s_h^k,a_h^k}\big)V^{\rref , k^i}_{h+1}\bigg| = \bigg|\sum_{i=1}^k X_i(s_h^k,a_h^k,h,N_h^k)\bigg|\\
	& \lesssim \sqrt{C_{\mathrm{u}} \log^2\frac{SAT}{\delta}}\sqrt{\sum_{n = 1}^{N_h^k} u_h^{k^n}(s_h^k,a_h^k, N_h^k) \Var_{h, s_h^k,a_h^k} \big(W_{h+1}^{k^n}  \big)} + \left(C_{\mathrm{u}} C_{\mathrm{w}} + \sqrt{\frac{C_{\mathrm{u}}}{N}} C_{\mathrm{w}}\right) \log^2\frac{SAT}{\delta} \nonumber\\
	&  \lesssim \sqrt{\frac{1}{N_h^k}\log^2\frac{SAT}{\delta}}\sqrt{\frac{1}{N_h^k}\sum_{n = 1}^{N_h^k} \Var_{h, s_h^k,a_h^k}\big(V^{\rref , k^n}_{h+1}\big)} + \frac{H}{N_h^k} \log^2\frac{SAT}{\delta} \label{lemma1:equ5-sub}\\
	&  \lesssim \sqrt{\frac{1}{N_h^k}\log^2\frac{SAT}{\delta}}\sqrt{ \sigma^{\re , k^{N_h^k}+1}_h(s_h^k,a_h^k) - \big(\mu^{\re , k^{N_h^k}+1}_h(s_h^k,a_h^k) \big)^2} + \frac{H}{(N_h^k)^{3/4}}\log^2\frac{SAT}{\delta}
	\label{lemma1:equ5}
\end{align}
with probability exceeding $1- \delta$, where the proof of the last inequality \eqref{lemma1:equ5} is deferred to Appendix~\ref{sec:proof:eq:var-Vref} in order to streamline presentation.

\paragraph{Step 5: combining the above bounds.}

Summing up the results in \eqref{lemma1:equ4} and \eqref{lemma1:equ5}, 
we arrive at an upper bound on $\big|\sum_{n = 1}^{N_h^k} \eta^{N_h^k}_n \xi_h^{k^n} \big|$ as follows:
\begin{align}
	\bigg| \sum_{n = 1}^{N_h^k} \eta^{N_h^k}_n \xi_h^{k^n} \bigg|  &\leq |I_1| +|I_2| \notag\\
	& \lesssim~\sqrt{\frac{H}{N_h^k}\log^2\frac{SAT}{\delta}}\sqrt{\sigma^{\adv, k^{N_h^k}+1}_h(s_h^k,a_h^k) - \big(\mu^{\adv, k^{N_h^k}+1}_h(s_h^k,a_h^k) \big)^2} \nonumber \\
	&\qquad + \sqrt{\frac{1}{N_h^k}\log^2\frac{SAT}{\delta}}\sqrt{\sigma^{\re , k^{N_h^k}+1}_h(s_h^k,a_h^k) - \big(\mu^{\re , k^{N_h^k}+1}_h(s_h^k,a_h^k) \big)^2} + \frac{H^2\log^2\frac{SAT}{\delta}}{(N_h^k)^{3/4}}  \nonumber\\
	& \leq \sumb^{\rref , k^{N_h^k}+1}_h(s_h^k,a_h^k) + \cb \frac{H^2\log^2\frac{SAT}{\delta}}{(N_h^k)^{3/4}} 
	\label{lemma1:equ6}
\end{align}
for some sufficiently large constant $\cb>0$, 
where the last line follows from the definition of $\sumb^{\rref , k^{N_h^k}+1}_h(s_h^k,a_h^k)$ in line~\ref{eq:line-number-19} of Algorithm~\ref{algo:subroutine}. 

In order to establish the desired bound \eqref{eq:claim-sum-eta-sum-b}, we still need to control the sum $\sum_{n = 1}^{N_h^k} \eta^{N_h^k}_nb^{\rref , k^n+1}_h$. 
Towards this end,
the definition of $b^{\rref , k^n+1}_h$ (resp.~$\delta_h^{\rref}$) in line~\ref{line:bonus_2} (resp.~line~\ref{eq:line-delta}) of Algorithm~\ref{algo:subroutine} yields
\begin{align}
 & b_{h}^{\rref,k^{n}+1}=\Big(1-\frac{1}{\eta_{n}}\Big)\sumb_{h}^{\rref,k^{n}}(s_{h}^{k},a_{h}^{k})
 +\frac{1}{\eta_{n}}\sumb_{h}^{\rref,k^{n}+1}(s_{h}^{k},a_{h}^{k})+\frac{\cb}{n^{3/4}}H^{2}\log^{2}\frac{SAT}{\delta}.
	\label{eq:bh-ref-kn-identity}
\end{align}
This taken collectively with the definition \eqref{equ:learning rate notation} of $\eta_n^N$ allows us to expand
\begin{align}
& \sum_{n = 1}^{N_h^k} \eta^{N_h^k}_nb^{\rref , k^n+1}_h \nonumber \\
&=\sum_{n = 1}^{N_h^k} \eta_n \prod_{i = n+1}^{N_h^k}(1-\eta_i) \left( \Big(1-\frac{1}{\eta_n}\Big) \sumb^{\rref , k^n}_h(s_h^k,a_h^k) + \frac{1}{\eta_n} \sumb^{\rref , k^n+1}_h(s_h^k,a_h^k) \right) + \cb \sum_{n=1}^{N_h^k}\frac{ \eta_n^{N_h^k} }{n^{3/4}}  H^2\log^2\frac{SAT}{\delta}\nonumber \\
&=\sum_{n = 1}^{N_h^k}  \prod_{i = n+1}^{N_h^k}(1-\eta_i) \left( - \big(1-\eta_n\big) \sumb^{\rref , k^n}_h(s_h^k,a_h^k) + \sumb^{\rref , k^n+1}_h(s_h^k,a_h^k) \right) + \cb \sum_{n=1}^{N_h^k}\frac{ \eta_n^{N_h^k} }{n^{3/4}}  H^2\log^2\frac{SAT}{\delta}\nonumber \\
&= \sum_{n = 1}^{N_h^k} \left(\prod_{i = n+1}^{N_h^k}(1-\eta_i) \sumb^{\rref , k^n+1}_h(s_h^k,a_h^k) - \prod_{i = n}^{N_h^k}(1-\eta_i) \sumb^{\rref , k^n}_h(s_h^k,a_h^k)\right) + \cb \sum_{n=1}^{N_h^k}\frac{ \eta_n^{N_h^k} }{n^{3/4}}  H^2\log^2\frac{SAT}{\delta}  \nonumber \\
&\overset{\mathrm{(i)}}{=} \sum_{n=1}^{N_h^k}  \prod_{i = n+1}^{N_h^k}(1-\eta_i) \sumb^{\rref , k^n+1}_h(s_h^k,a_h^k)  
- \sum_{n=2}^{N_h^k} \prod_{i = n}^{N_h^k}(1-\eta_i) \sumb^{\rref , k^n}_h(s_h^k,a_h^k) + \cb \sum_{n=1}^{N_h^k}\frac{ \eta_n^{N_h^k} }{n^{3/4}}  H^2\log^2\frac{SAT}{\delta} \nonumber \\
&\overset{\mathrm{(ii)}}{=}\sum_{n=1}^{N_h^k}  \prod_{i = n+1}^{N_h^k}(1-\eta_i) \sumb^{\rref , k^n+1}_h(s_h^k,a_h^k)  - \sum_{n=1}^{N_h^k-1} \prod_{i = n+1}^{N_h^k}(1-\eta_i) \sumb^{\rref , k^n+1}_h(s_h^k,a_h^k) + \cb \sum_{n=1}^{N_h^k}\frac{ \eta_n^{N_h^k} }{n^{3/4}}  H^2\log^2\frac{SAT}{\delta} \nonumber \\
&=\sumb^{\rref , k^{N_h^k}+1}_h(s_h^k,a_h^k) + \cb \sum_{n=1}^{N_h^k}\frac{ \eta_n^{N_h^k} }{n^{3/4}}  H^2\log^2\frac{SAT}{\delta} . 
	\label{eq:eta-b-ref-sum-identity-135}
\end{align}
Here,  
(i) is valid due to the fact that $\sumb_h^{\rref , k^1}(s_h^k,a_h^k) = 0$;
 (ii) follows from the fact that 
\begin{align*}
\sum_{n=2}^{N_{h}^{k}}\prod_{i=n}^{N_{h}^{k}}(1-\eta_{i})\sumb_{h}^{\rref,k^{n}}(s_{h}^{k},a_{h}^{k}) & =\sum_{n=1}^{N_{h}^{k}-1}\prod_{i=n+1}^{N_{h}^{k}}(1-\eta_{i})\sumb_{h}^{\rref,k^{n+1}}(s_{h}^{k},a_{h}^{k})\\
 & =\sum_{n=1}^{N_{h}^{k}-1}\prod_{i=n+1}^{N_{h}^{k}}(1-\eta_{i})\sumb_{h}^{\rref,k^{n}+1}(s_{h}^{k},a_{h}^{k}), 
\end{align*}
where the first relation can be seen by replacing $n$ with $n+1$, and the last relation holds true since the state-action pair 
$(s_h^k,a_h^k)$ has not been visited at step $h$ between the $(k^n +1)$-th episode and the $(k^{n+1} - 1)$-th episode.
Combining the above identity \eqref{eq:eta-b-ref-sum-identity-135} with the following property 
(see Lemma~\ref{lemma:property of learning rate}) $$\frac{1}{(N_h^k)^{3/4}} \le \sum_{n = 1}^{N_h^k} \frac{\eta^{N_h^k}_n}{n^{3/4}} \le \frac{2}{(N_h^k)^{3/4}},$$  
we can immediately demonstrate that
\begin{equation}
	\sumb^{\rref , k^{N_h^k}+1}_h(s_h^k,a_h^k) + \cb\frac{H^2\log^2\frac{SAT}{\delta}}{(N_h^k)^{3/4}} \le \sum_{n = 1}^{N_h^k} \eta^{N_h^k}_nb^{\rref , k^n+1}_h \le \sumb^{\rref , k^{N_h^k}+1}_h(s_h^k,a_h^k) + 2\cb\frac{H^2\log^2\frac{SAT}{\delta}}{(N_h^k)^{3/4}}. 
	\label{lemma1:equ10}
\end{equation}
Taking  \eqref{lemma1:equ6} and \eqref{lemma1:equ10} collectively demonstrates that
\begin{equation}
	\bigg| \sum_{n = 1}^{N_h^k} \eta^{N_h^k}_n \xi^{k^n}_h \bigg| 
	\le \sumb^{\rref , k^{N_h^k}+1}_h(s_h^k,a_h^k) + \cb\frac{H^2\log^2\frac{SAT}{\delta}}{(N_h^k)^{3/4}} \le \sum_{n = 1}^{N_h^k} \eta^{N_h^k}_n b^{\rref , k^n+1}_h \label{lemma1:equ7}
\end{equation}
as claimed in \eqref{eq:claim-sum-eta-sum-b}.  We have thus concluded the proof of Lemma~\ref{lem:Qk-lower} based on the argument in Step 2.

\subsubsection{Proof of the inequality \eqref{lemma1:equ4}}\label{sec:proof:eq:var-V-Vref}

In order to establish the inequality \eqref{lemma1:equ4},
it suffices to look at the following term 
\begin{align}
	I_3 := \sum_{n = 1}^{N_h^k} \eta^{N_h^k}_n\Var_{h, s_h^k,a_h^k}\big(V^{k^n}_{h+1} - {V}^{\rref, k^n}_{h+1}\big) - {\sigma}^{\adv, k^{N_h^k}+1}_h(s_h^k,a_h^k) + \big({\mu}^{\adv, k^{N_h^k}+1}_h(s_h^k,a_h^k) \big)^2,
	\label{eq:defn-I3-lem-equ4}
\end{align}
which forms the main content of this subsection.

First of all,  the update rules of $\mu^{\adv, k^{n+1}}_h$ and $\sigma^{\adv, k^{n+1}}_h $  in lines~\ref{line:advmu_h}-\ref{line:advsigma_h} of Algorithm~\ref{algo:subroutine} tell us that
\begin{align*}
\mu^{\adv, k^{n+1}}_h(s_h^k,a_h^k) &= \mu^{\adv, k^n+1}_h(s_h^k,a_h^k) 
	= (1-\eta_n)\mu^{\adv, k^n}_h(s_h^k,a_h^k) + \eta_n \big( V^{k^n}_{h+1}(s^{k^n}_{h+1}) - V^{\rref, k^n}_{h+1}(s^{k^n}_{h+1}) \big) ,\\ 
\sigma^{\adv, k^{n+1}}_h(s_h^k,a_h^k) &= \sigma^{\adv, k^n+1}_h(s_h^k,a_h^k) 
	= (1-\eta_n)\sigma^{\adv, k^n}_h(s_h^k,a_h^k) + \eta_n \big(V^{k^n}_{h+1}(s^{k^n}_{h+1}) - V^{\rref, k^n}_{h+1}(s^{k^n}_{h+1}) \big)^2.
\end{align*} 
Applying this relation recursively and invoking the definitions of  $\eta_n^{N}$ (resp.~$P_h^{k}$) 
in \eqref{equ:learning rate notation} (resp.~\eqref{eq:P-hk-defn-s}) give
\begin{subequations} \label{eq:recursion_mu_sigma_adv}
\begin{align}
	\mu_{h}^{\adv,k^{N_{h}^{k}}+1}(s_{h}^{k},a_{h}^{k}) & \overset{\mathrm{(i)}}{=}\sum_{n=1}^{N_{h}^{k}}\eta_{n}^{N_{h}^{k}}\big(V_{h+1}^{k^{n}}(s_{h+1}^{k^{n}})-{V}_{h+1}^{\rref,k^{n}}(s_{h+1}^{k^{n}})\big)=\sum_{n=1}^{N_{h}^{k}}\eta_{n}^{N_{h}^{k}}P_{h}^{k^{n}}\big(V_{h+1}^{k^{n}}-{V}_{h+1}^{\rref,k^{n}}\big),\\
{\sigma}_{h}^{\adv,k^{N_{h}^{k}}+1}(s_{h}^{k},a_{h}^{k}) & \overset{\mathrm{(ii)}}{=} \sum_{n=1}^{{N_{h}^{k}}}\eta_{n}^{N_{h}^{k}}\big(V_{h+1}^{k^{n}}(s_{h+1}^{k^{n}})-{V}_{h+1}^{\rref,k^{n}}(s_{h+1}^{k^{n}})\big)^{2}=\sum_{n=1}^{{N_{h}^{k}}}\eta_{n}^{N_{h}^{k}}P_{h}^{k^{n}}\big(V_{h+1}^{k^{n}}-{V}_{h+1}^{\rref,k^{n}}\big)^{2}. 
\end{align}
\end{subequations}
Recognizing that $\sum_{n=1}^{{N_{h}^{k}}}\eta_{n}^{N_{h}^{k}}=1$ (see \eqref{eq:sum-eta-n-N}), we can immediately apply Jensen's inequality to the expressions (i) and (ii) to yield 
\begin{align}\label{equ:positive-of-ref}
{\sigma}_h^{\adv, k^{N_h^k}+1}(s_h^k,a_h^k) \geq \Big(\mu_h^{\adv, k^{N_h^k}+1}(s_h^k,a_h^k)\Big)^2.
\end{align}
Further, in view of the definition \eqref{lemma1:equ2}, we have
\begin{equation*}
\Var_{h, s_h^k,a_h^k}\big(V_{h+1}^{k^n} - {V}_{h+1}^{\rref, k^n} \big) 
= P_{h, s_h^k,a_h^k} \big(V_{h+1}^{k^n} - {V}_{h+1}^{\rref, k^n}\big)^{2} - \left(P_{h, s_h^k,a_h^k} \big(V_{h+1}^{k^n} - {V}_{h+1}^{\rref, k^n} \big) \right)^2,
\end{equation*}
which allows one to decompose and bound $I_3$ as follows 
\begin{align}
I_3 & = \sum_{n = 1}^{N_h^k} \eta^{N_h^k}_n P_{h, s_h^k,a_h^k} \big(V_{h+1}^{k^n} - {V}_{h+1}^{\rref, k^n}\big)^{ 2} - \sum_{n = 1}^{{N_h^k}} \eta_n^{N_h^k} P_h^{k^n} \big(V_{h+1}^{k^n} - {V}_{h+1}^{\rref, k^n} \big)^2  \nonumber \\
&\qquad \qquad+ \Bigg(\sum_{n = 1}^{{N_h^k}} \eta_n^{N_h^k} P^{k^n}_{h} \big(V^{k^n}_{h+1} - {V}^{\rref, k^n}_{h+1} \big)\Bigg)^2  - \sum_{n = 1}^{{N_h^k}} \eta_n^{N_h^k} \Big(P_{h, s_h^k,a_h^k} \big(V^{k^n}_{h+1} - {V}^{\rref, k^n}_{h+1} \big)\Big)^2 \nonumber \\
	& \le  \underbrace{ \Bigg|\sum_{n = 1}^{{N_h^k}} \eta_n^{N_h^k} \big( P^{k^n}_{h}-P_{h, s_h^k,a_h^k} \big) \big(V^{k^n}_{h+1} 
	- {V}^{\rref, k^n}_{h+1} \big)^{ 2} \Bigg| }_{\eqqcolon I_{3,1}}  \nonumber \\
	&\qquad \qquad+ \underbrace{ \bigg(\sum_{n = 1}^{{N_h^k}} \eta_n^{N_h^k} P^{k^n}_{h} \big(V^{k^n}_{h+1} - {V}^{\rref, k^n}_{h+1} \big)\bigg)^2  - \sum_{n = 1}^{{N_h^k}} \eta_n^{N_h^k} \Big(P_{h, s_h^k,a_h^k} \big(V^{k^n}_{h+1} - {V}^{\rref, k^n}_{h+1} \big)\Big)^2}_{\eqqcolon I_{3,2}} . 
	\label{equ:lemma4 vr 2}
\end{align}
It then boils down to controlling the above two terms in \eqref{equ:lemma4 vr 2} separately.

\paragraph{Step 1: bounding $I_{3,1}$.} To upper bound the term $I_{3,1}$ in \eqref{equ:lemma4 vr 2}, we resort to Lemma~\ref{lemma:martingale-union-all} by setting
\begin{align}
	W_{h+1}^i \coloneqq \big(V^{i}_{h+1} - {V}^{\rref, i}_{h+1} \big)^{ 2}
	\qquad \text{and} \qquad
	u_h^i(s,a, N) \coloneqq \eta_{N_h^i(s,a)}^N. 
	\label{I31-wu}
\end{align}
It is easily seen that
\begin{align}
	\|W_{h+1}^i\|_\infty \leq  \Big( \big\| V^{\rref , i}_{h+1} \big\|_{\infty} + \big\|V^{i}_{h+1} \big\|_\infty \Big)^2 
	 \leq 4H^2 \eqqcolon C_{\mathrm{w}}, \label{I31-cw}
\end{align}
and it follows from \eqref{I1-cu} that
\begin{align}
	  \max_{N, h, s, a\in [K] \times[H]\times \cS\times \cA} \eta_{N_{h}^{i}(s,a)}^{N} \leq \frac{2H}{N} \eqqcolon C_{\mathrm{u}}.
	  \label{I31-cu}
\end{align}
Armed with the properties \eqref{I31-cw} and \eqref{I31-cu} and recalling \eqref{eq:sum-uh-k-s-a-123},
we can invoke Lemma~\ref{lemma:martingale-union-all} w.r.t.~\eqref{I31-wu} and set $(N,s,a)=(N_h^k, s_h^k, a_h^k)$ to yield
\begin{align}
	I_{3,1} &= \Bigg|\sum_{n = 1}^{{N_h^k}} \eta_n^{N_h^k} \big(P^{k^n}_{h}-P_{h, s_h^k,a_h^k} \big)\big(V^{k^n}_{h+1} - {V}^{\rref, k^n}_{h+1}\big)^2\Bigg| = \left|\sum_{i=1}^k X_i(s_h^k,a_h^k,h, N_h^k)\right| \nonumber\\
	&\lesssim \sqrt{C_{\mathrm{u}} \log^2\frac{SAT}{\delta}}\sqrt{\sum_{n = 1}^{N_h^k} u_h^{k^n}(s_h^k,a_h^k, N_h^k) \Var_{h, s_h^k,a_h^k} \big(W_{h+1}^{k^n}  \big)} + \left(C_{\mathrm{u}} C_{\mathrm{w}} + \sqrt{\frac{C_{\mathrm{u}}}{N}} C_{\mathrm{w}}\right) \log^2\frac{SAT}{\delta} \nonumber \\
	& \lesssim \sqrt{\frac{H}{N_h^k}\log^2\frac{SAT}{\delta}}\sqrt{\sum_{n = 1}^{N_h^k} \eta^{N_h^k}_n \Var_{h, s_h^k,a_h^k} \Big( \big(V^{k^n}_{h+1} - V^{\rref , k^n}_{h+1} \big)^2 \Big) } + \frac{H^3\log^2\frac{SAT}{\delta}}{N_h^k} \nonumber \\
	&\lesssim \sqrt{\frac{H^5}{{N_h^k}}\log^2\frac{SAT}{\delta}} + \frac{H^3}{{N_h^k}}\log^2\frac{SAT}{\delta} \label{equ:lemma4 vr 3}
\end{align}
with probability at least $1-\delta$. 
Here, the last inequality results from the fact $\sum_{n = 1}^{N_h^k} \eta^{N_h^k}_n \leq 1$ (see \eqref{eq:sum-eta-n-N}) and the following trivial result:
\begin{align}
	\Var_{h, s_h^k,a_h^k} \Big(\big(V^{k^n}_{h+1} - {V}^{\rref, k^n}_{h+1}\big)^2\Big)
	\leq \big\|\big(V^{k^n}_{h+1} - {V}^{\rref, k^n}_{h+1} \big)^{ 4} \big\|_\infty 
	\leq 16H^4.
\end{align}

\paragraph{Step 2: bounding $I_{3,2}$.} Jensen's inequality tells us that
\begin{align*}
\Bigg(\sum_{n=1}^{{N_{h}^{k}}}\eta_{n}^{N_{h}^{k}}P_{h,s_{h}^{k},a_{h}^{k}} \big(V_{h+1}^{k^{n}}-{V}_{h+1}^{\rref,k^{n}} \big)\Bigg)^{2} 
	& = \Bigg(\sum_{n=1}^{{N_{h}^{k}}} \big( \eta_{n}^{N_{h}^{k}}\big)^{1/2} \cdot \big( \eta_{n}^{N_{h}^{k}}\big)^{1/2} P_{h,s_{h}^{k},a_{h}^{k}} \big(V_{h+1}^{k^{n}}-{V}_{h+1}^{\rref,k^{n}} \big)\Bigg)^{2} \\
& \le\left\{ \sum_{n=1}^{{N_{h}^{k}}}\eta_{n}^{N_{h}^{k}}\right\} \left\{ \sum_{n=1}^{{N_{h}^{k}}}\eta_{n}^{N_{h}^{k}}\Big(P_{h,s_{h}^{k},a_{h}^{k}} \big(V_{h+1}^{k^{n}}-{V}_{h+1}^{\rref,k^{n}}\big)\Big)^{2}\right\} \\ 
 & \leq \sum_{n=1}^{{N_{h}^{k}}}\eta_{n}^{N_{h}^{k}}\Big(P_{h,s_{h}^{k},a_{h}^{k}}\big(V_{h+1}^{k^{n}}-{V}_{h+1}^{\rref,k^{n}}\big)\Big)^{2},
\end{align*}
where the last line arises from \eqref{eq:sum-eta-n-N}. 
Substitution into $I_{3,2}$ (cf.~\eqref{equ:lemma4 vr 2}) gives 
\begin{align}
	I_{3,2} &\le \bigg(\sum_{n = 1}^{{N_h^k}} \eta_n^{N_h^k} P^{k^n}_{h}\big(V^{k^n}_{h+1} - {V}^{\rref, k^n}_{h+1}\big) \bigg)^2 
	- \bigg(\sum_{n = 1}^{{N_h^k}} \eta_n^{N_h^k} P_{h, s_h^k,a_h^k}\big(V^{k^n}_{h+1} - {V}^{\rref, k^n}_{h+1}\big)\bigg)^2 \nonumber\\
	& = \bigg\{\sum_{n = 1}^{{N_h^k}} \eta_n^{N_h^k} \big(P^{k^n}_{h}-P_{h, s_h^k,a_h^k}\big) \big(V^{k^n}_{h+1} - {V}^{\rref, k^n}_{h+1}\big)\bigg\}   \bigg\{ \sum_{n = 1}^{{N_h^k}} \eta_n^{N_h^k} \big(P^{k^n}_{h}+P_{h, s_h^k,a_h^k}\big) \big(V^{k^n}_{h+1} - {V}^{\rref, k^n}_{h+1}\big)\bigg\}  . \label{equ:lemma4vr4}
\end{align}
In what follows, we would like to use this relation to show that
\begin{equation} \label{eq:bound_I32}
	I_{3,2} \leq C_{32}  \bigg\{ \sqrt{\frac{H^5}{{N_h^k}}\log^2\frac{SAT}{\delta}} + \frac{H^3}{{N_h^k}}\log^2\frac{SAT}{\delta}  \bigg\}
\end{equation}
for some universal constant $C_{32}>0$. 

If $I_{3,2}\leq 0$, then \eqref{eq:bound_I32} holds true trivially. Consequently, it is sufficient to study the case where $I_{3,2}> 0$.   
To this end, we first note that the term in the first pair of curly brakets of \eqref{equ:lemma4vr4} is exactly $I_1$ (see \eqref{eq:defn-I1-step2}), which can be bounded by recalling \eqref{lemma1:equ4-sub}:
\begin{align} 
|I_1| 
	& \lesssim \sqrt{\frac{H}{N_h^k}\log^2\frac{SAT}{\delta}}\sqrt{\sum_{n = 1}^{N_h^k} \eta^{N_h^k}_n \Var_{h, s_h^k,a_h^k} \big(V^{k^n}_{h+1} - V^{\rref , k^n}_{h+1} \big) } + \frac{H^2\log^2\frac{SAT}{\delta}}{N_h^k} \nonumber \\
& \lesssim \sqrt{\frac{H^3}{N_h^k}\log^2\frac{SAT}{\delta}}\sqrt{\sum_{n = 1}^{N_h^k} \eta^{N_h^k}_n  } + \frac{H^2\log^2\frac{SAT}{\delta}}{N_h^k} \nonumber \\	
 &\lesssim   \sqrt{\frac{H^3}{{N_h^k}}\log^2\frac{SAT}{\delta}} + \frac{H^2}{{N_h^k}}\log^2\frac{SAT}{\delta},
	\label{eq:first-curly-bracket-bound}
\end{align}
with probability at least $1-\delta$. 
Here, the second inequality arises from  the following property
\begin{align}
	\Var_{h, s_h^k,a_h^k} \Big(V^{k^n}_{h+1} - {V}^{\rref, k^n}_{h+1}\Big)  
	&\leq \big\|\big(V^{k^n}_{h+1} - {V}^{\rref, k^n}_{h+1}\big)^{ 2} \big\|_\infty \leq 4H^2,
\end{align}
whereas the last inequality \eqref{eq:first-curly-bracket-bound} holds as a result of the fact $\sum_{n = 1}^{N_h^k} \eta^{N_h^k}_n \leq 1$ (see \eqref{eq:sum-eta-n-N}). 

Moreover, the term in the second pair of curly brakets of \eqref{equ:lemma4vr4} can be bounded straightforwardly as follows
\begin{align}
	& \bigg|\sum_{n = 1}^{{N_h^k}} \eta_n^{N_h^k} \big(P^{k^n}_{h}+P_{h, s_h^k,a_h^k}\big) \big(V^{k^n}_{h+1} - {V}^{\rref, k^n}_{h+1}\big) \bigg|  \notag\\
	& \qquad \le \sum_{n = 1}^{{N_h^k}} \eta_n^{N_h^k} \Big( \big\|P^{k^n}_{h}\big\|_1 +\big\|P_{h, s_h^k,a_h^k}\big\|_1 \Big) \big\|V^{k^n}_{h+1} - {V}^{\rref, k^n}_{h+1} \big\|_{\infty} \leq 2H,
	\label{eq:second-curly-bracket-bound}
\end{align}
where we have made use of the property \eqref{eq:sum-eta-n-N}, as well as the elementary facts $\big\|V^{k^n}_{h+1} - V^{\rref , k^n}_{h+1}\big\|_{\infty}\leq H$ and $\big\|P^{k^n}_{h}\big\|_1 =\big\| P_{h, s_h^k,a_h^k}\big\|_1=1$.
Substituting the above two results \eqref{eq:first-curly-bracket-bound} and \eqref{eq:second-curly-bracket-bound} back into \eqref{equ:lemma4vr4}, we arrive at the bound \eqref{eq:bound_I32} as long as $I_{3,2}>0$. 
Putting all cases together, we have established the claim \eqref{eq:bound_I32}.

\paragraph{Step 3: putting all this together.} 
To finish up, plugging the bounds \eqref{equ:lemma4 vr 3} and \eqref{eq:bound_I32} into \eqref{equ:lemma4 vr 2}, we can conclude that
\begin{align*} 
	I_3 \leq I_{3,1} + I_{3,2} 
	\leq C_3\bigg\{  \sqrt{\frac{H^5}{{N_h^k}}\log^2\frac{SAT}{\delta}} + \frac{H^3}{{N_h^k}}\log^2\frac{SAT}{\delta} \bigg\} 
\end{align*}
for some constant $C_3>0$. 
This together with the definition \eqref{eq:defn-I3-lem-equ4} of $I_3$ results in
\begin{align*}
 & \sum_{n=1}^{N_{h}^{k}}\eta_{n}^{N_{h}^{k}}\Var_{h,s_{h}^{k},a_{h}^{k}}\big(V_{h+1}^{k^{n}}-{V}_{h+1}^{\rref,k^{n}}\big)\\
	& \qquad\leq \Big\{ {\sigma}_{h}^{\adv,k^{N_{h}^{k}}+1}(s_{h}^{k},a_{h}^{k})-\big({\mu}_{h}^{\adv,k^{N_{h}^{k}}+1}(s_{h}^{k},a_{h}^{k})\big)^{2} \Big\}
	+ C_3 \left(\sqrt{\frac{H^{5}}{{N_{h}^{k}}}\log^{2}\frac{SAT}{\delta}}+\frac{H^{3}}{{N_{h}^{k}}}\log^{2}\frac{SAT}{\delta}\right),
\end{align*}
which combined with the elementary inequality $\sqrt{u+v}\leq \sqrt{u}+\sqrt{v}$ for any $u,v\geq 0$ and \eqref{equ:positive-of-ref} yields 
\begin{align*}
 & \bigg\{\sum_{n=1}^{N_{h}^{k}}\eta_{n}^{N_{h}^{k}}\Var_{h,s_{h}^{k},a_{h}^{k}}\big(V_{h+1}^{k^{n}}-{V}_{h+1}^{\rref,k^{n}}\big)\bigg\}^{1/2}\\
 & \qquad\lesssim \Big\{{\sigma}_{h}^{\adv,k^{N_{h}^{k}}+1}(s_{h}^{k},a_{h}^{k})-\big({\mu}_{h}^{\adv,k^{N_{h}^{k}}+1}(s_{h}^{k},a_{h}^{k})\big)^{2}\Big\}^{1/2}+\frac{H^{5/4}}{\big(N_{h}^{k}\big)^{1/4}}\log^{1/2}\frac{SAT}{\delta}+\frac{H^{3/2}}{\big(N_{h}^{k}\big)^{1/2}}\log\frac{SAT}{\delta}. 
\end{align*}
Substitution into \eqref{lemma1:equ4-sub} establishes the desired result \eqref{lemma1:equ4}.

\subsubsection{Proof of the inequality \eqref{lemma1:equ5}}
\label{sec:proof:eq:var-Vref}

In order to prove the inequality \eqref{lemma1:equ5}, it suffices to look at the following term
\begin{equation}
I_4 \defn  \frac{1}{N_h^k}\sum_{n = 1}^{{N_h^k}} \Var_{h, s_h^k,a_h^k}({V}^{\rref, k^n}_{h+1}) - \Big( {\sigma}^{\re, k^{N_h^k}+1}_h(s_h^k,a_h^k) -\big({\mu}^{\re, k^{N_h^k}+1}_h(s_h^k,a_h^k) \big)^2 \Big) .
\label{eq:defn-I4-lem-equ4}
\end{equation}
In view of the update rules of ${\mu}^{\re, k^{n+1}}_h$ and ${\sigma}^{\re, k^{n+1}}_h $  in lines~\ref{line:refmu_h}-\ref{line:refsigma_h} of Algorithm~\ref{algo:subroutine}, we have
\begin{align*}
{\mu}^{\re, k^{n+1}}_h(s_h^k,a_h^k) &= {\mu}^{\re, k^n+1}_h(s_h^k,a_h^k) = \left(1-\frac{1}{n}\right){\mu}^{\re, k^n}_h(s_h^k,a_h^k) + \frac{1}{n}{V}^{\rref, k^n}_{h+1}(s^{k^n}_{h+1}), \\
{\sigma}^{\re, k^{n+1}}_h(s_h^k,a_h^k) &= {\sigma}^{\re, k^n+1}_h(s_h^k,a_h^k) = \left(1-\frac{1}{n}\right){\sigma}^{\re, k^n}_h(s_h^k,a_h^k) + \frac{1}{n}\big( {V}^{\rref, k^n}_{h+1}(s^{k^n}_{h+1}) \big)^2,
\end{align*} 
Through simple recursion, these identities together with the definition \eqref{eq:P-hk-defn-s} of $P_h^k$ lead to
\begin{subequations}
\label{eq:recursion_mu_sigma_ref}
\begin{align}
	{\mu}_{h}^{\re,k^{N_{h}^{k}}+1}(s_{h}^{k},a_{h}^{k}) & \overset{\mathrm{(i)}}{=} \frac{1}{N_{h}^{k}}\sum_{n=1}^{N_{h}^{k}}{V}_{h+1}^{\rref,k^{n}}(s_{h+1}^{k^{n}})=\frac{1}{N_{h}^{k}}\sum_{n=1}^{N_{h}^{k}}P_{h}^{k^{n}}{V}_{h+1}^{\rref,k^{n}},\\
{\sigma}_{h}^{\re,k^{N_{h}^{k}}+1}(s_{h}^{k},a_{h}^{k}) & \overset{\mathrm{(ii)}}{=} \frac{1}{N_{h}^{k}}\sum_{n=1}^{{N_{h}^{k}}}\big({V}_{h+1}^{\rref,k^{n}}(s_{h+1}^{k^{n}})\big)^{2}=\frac{1}{N_{h}^{k}}\sum_{n=1}^{{N_{h}^{k}}}P_{h}^{k^{n}}\big({V}_{h+1}^{\rref,k^{n}}\big)^{2},
\end{align}
\end{subequations}
The expressions (i) and (ii) combined with Jensen's inequality give
\begin{align}\label{equ:positive-of-adv-12}
{\sigma}_h^{\re, k^{N_h^k}+1}(s_h^k,a_h^k) \geq \Big(\mu_h^{\re, k^{N_h^k}+1}(s_h^k,a_h^k)\Big)^2.
\end{align}
Taking these together with the definition
\begin{equation*}
\Var_{h, s_h^k,a_h^k}({V}_{h+1}^{\rref, k^n}) = P_{h, s_h^k,a_h^k} \big({V}_{h+1}^{\rref, k^n} \big)^{2} - \big(P_{h, s_h^k,a_h^k}{V}_{h+1}^{\rref, k^n} \big)^2,
\end{equation*} 
we obtain
\begin{align}
I_4 & = \frac{1}{{N_h^k}}\sum_{n=1}^{N_h^k} \Big(  P_{h, s_h^k,a_h^k}({V}_{h+1}^{\rref, k^n})^{2} - \big(P_{h, s_h^k,a_h^k}{V}_{h+1}^{\rref, k^n} \big)^2 \Big) - \frac{1}{{N_h^k}}\sum_{n = 1}^{{N_h^k}} P_h^{k^n}\big({V}_{h+1}^{\rref, k^n}\big)^2 + \bigg(\frac{1}{{N_h^k}}\sum_{n=1}^{N_h^k} P_h^{k^n} {V}_{h+1}^{\rref, k^n}\bigg)^2 \nonumber \\
& = \underbrace{ \frac{1}{{N_h^k}}\sum_{n = 1}^{{N_h^k}} \big(P_{h, s_h^k,a_h^k} - P^{k^n}_{h} \big) \big({V}^{\rref, k^n}_{h+1}\big)^2 }_{\eqqcolon \, I_{4,1}} + \underbrace{ \bigg(\frac{1}{{N_h^k}}\sum_{n=1}^{N_h^k} P_h^{k^n} {V}_{h+1}^{\rref, k^n}\bigg)^2 - \frac{1}{{N_h^k}}\sum_{n = 1}^{{N_h^k}} \Big(P_{h, s_h^k,a_h^k}{V}^{\rref, k^n}_{h+1}\Big)^2}_{\eqqcolon\, I_{4,2}} . \label{equ:lemma4 vr 6}
\end{align}
In what follows,
we shall bound the terms $I_{4,1}$ and $I_{4,2}$ in \eqref{equ:lemma4 vr 6} separately.

\paragraph{Step 1: bounding $I_{4,1}$.}

The first term $I_{4,1}$ in \eqref{equ:lemma4 vr 6} can be bounded by means of Lemma~\ref{lemma:martingale-union-all} in an almost identical fashion as $I_{3,1}$ in \eqref{equ:lemma4 vr 3}. 
Specifically, let us set 
\begin{align*}
	W_{h+1}^i \coloneqq ({V}^{\rref, i}_{h+1})^{ 2}
	\qquad \text{and} \qquad
	u_h^i(s,a, N) \coloneqq \frac{1}{N}, 
\end{align*}
which clearly obey
\begin{align*}
   |u_h^i(s,a, N)| = \frac{1}{N} \eqqcolon C_{\mathrm{u}}
	\qquad \text{and} \qquad
	  \|W_{h+1}^i\|_\infty \leq H^2 \eqqcolon C_{\mathrm{w}}. 
\end{align*} 
It is easily verified that
\[
\sum_{n=1}^{N} u_{h}^{k^{n}(s,a)}(s,a, N)= \sum_{n=1}^{N}\frac{1}{N}=1
\]
holds for all $(N,s,a) \in [K] \times \cS\times\cA$.
Hence we can take $(N,s,a)=(N_h^k,s_h^k,a_h^k)$ 
and apply Lemma~\ref{lemma:martingale-union-all} to yield
\begin{align}\label{equ:lemma4 vr 7}
|I_{4,1}| &= \bigg| \frac{1}{{N_h^k}}\sum_{n = 1}^{{N_h^k}} \big(P^{k^n}_{h}-P_{h, s_h^k,a_h^k} \big)\big({V}^{\rref, k^n}_{h+1}\big)^2 \bigg|  = \left|\sum_{i=1}^k X_i(s_h^k,a_h^k, h, N_h^k)\right| \nonumber\\ 
&\lesssim \sqrt{C_{\mathrm{u}} \log^2\frac{SAT}{\delta}}\sqrt{\sum_{n = 1}^{N_h^k} u_h^{k^n}(s_h^k,a_h^k, N_h^k) \Var_{h, s_h^k,a_h^k} \big(W_{h+1}^{k^n}  \big)} + \left(C_{\mathrm{u}} C_{\mathrm{w}} + \sqrt{\frac{C_{\mathrm{u}}}{N}} C_{\mathrm{w}}\right) \log^2\frac{SAT}{\delta} \nonumber \\
&\lesssim  \sqrt{\frac{H^4\log^2 \frac{SAT}{\delta}}{{N_h^k}}} + \frac{H^2\log^2\frac{SAT}{\delta}}{{N_h^k}}
\end{align}
with probability at least $1-\delta$, where the last inequality results from the fact that
\[
	\Var_{h, s_h^k,a_h^k} \big(W_{h+1}^{k^n}  \big) \leq \big\| W_{h+1}^{k^n} \big\|_{\infty}^2 \leq C_{\mathrm{w}}^2 = H^4. 
\]

\paragraph{Step 2: bounding $I_{4,2}$.}

We now turn to the other term $I_{4,2}$ defined in \eqref{equ:lemma4 vr 6}. Towards this, we first make the observation that
\begin{equation}
\Bigg(\frac{1}{{N_h^k}}\sum_{n=1}^{N_h^k} P_{h, s_h^k,a_h^k} {V}_{h+1}^{\rref,k^n}\Bigg)^2 \le \frac{1}{{N_h^k}}\sum_{n = 1}^{{N_h^k}} \Big(P_{h, s_h^k,a_h^k}{V}^{\rref,k^n}_{h+1}\Big)^2 , 
\end{equation}
which follows from Jensen's inequality. Equipped with this relation, we can upper bound $I_{4,2}$ as follows
\begin{align}
	I_{4,2} &\leq \Bigg(\frac{1}{{N_h^k}}\sum_{n=1}^{N_h^k} P_h^{k^n} {V}_{h+1}^{\rref,k^n}\Bigg)^2  - \Bigg(\frac{1}{{N_h^k}}\sum_{n=1}^{N_h^k} P_{h, s_h^k,a_h^k} {V}_{h+1}^{\rref,k^n}\Bigg)^2 \nonumber \\
	&= \bigg\{ \frac{1}{{N_h^k}}\sum_{n=1}^{N_h^k} \big( P_h^{k^n} -P_{h, s_h^k,a_h^k}\big)  {V}_{h+1}^{\rref,k^n}\bigg\} 
	\bigg\{ \frac{1}{{N_h^k}}\sum_{n=1}^{N_h^k} \big( P_h^{k^n} +P_{h, s_h^k,a_h^k}\big)  {V}_{h+1}^{\rref,k^n}\bigg\}. 
	\label{equ:lemma4vr5}
\end{align}
In the following, we would like to apply this relation to prove 
\begin{equation} 
	I_{4,2} \leq      
	C_{42} \bigg( \sqrt{\frac{H^4}{{N_h^k}}\log^2\frac{SAT}{\delta}} + \frac{H^2}{{N_h^k}} \log^2\frac{SAT}{\delta} \bigg) 
	\label{eq:bound_I42}
\end{equation}
for some constant $C_{42}>0$.

When $I_{4,2}\leq 0$, the claim  \eqref{eq:bound_I42} holds trivially. As a result, we shall focus on the case where $I_{4,2} > 0$. 
Let us begin with the term in the first pair of curly brackets of \eqref{equ:lemma4vr5}. Towards this, let us abuse the notation and set
\begin{align*}
	W_{h+1}^i \coloneqq {V}^{\rref, i}_{h+1} \qquad \text{and} \qquad u_h^i(s,a, N) \coloneqq \frac{1}{N},
\end{align*}
which satisfy
\begin{align*}
	| u_h^i(s,a, N)| =\frac{1}{N} \eqqcolon C_{\mathrm{u}} \qquad \text{and} \qquad
	 \|W_{h+1}^i\|_\infty \leq H \eqqcolon C_{\mathrm{w}}. 
\end{align*} 
Akin to our argument for bounding $I_{4,1}$, 
invoking Lemma~\ref{lemma:martingale-union-all} and setting $(N,s,a)=(N_h^k,s_h^k,a_h^k)$ imply that
\begin{equation*}
\Bigg|\frac{1}{{N_h^k}}\sum_{n=1}^{N_h^k} (P_h^{k^n} - P_{h, s_h^k,a_h^k} ){V}_{h+1}^{\rref,k^n}\Bigg| \lesssim  \sqrt{\frac{H^2\log^2\frac{SAT}{\delta}}{{N_h^k}}} + \frac{H\log^2\frac{SAT}{\delta}}{{N_h^k}} 
\end{equation*}
with probability at least $1-\delta$. In addition,
the term  in the second pair of curly brackets of \eqref{equ:lemma4vr5} can be bounded straightforwardly by
\begin{equation*}
\Bigg|\frac{1}{{N_h^k}}\sum_{n=1}^{N_h^k} \big( P_h^{k^n} +P_{h, s_h^k,a_h^k}\big)  {V}_{h+1}^{\rref,k^n}\Bigg| 
\le \frac{1}{{N_h^k}}\sum_{n = 1}^{{N_h^k}}  \big( \big\|P^{k^n}_{h}\big\|_1 
	+\big\|P_{h, s_h^k,a_h^k}\big\|_1 \big) \big\|{V}_{h+1}^{\rref,k^n}\big\|_{\infty} \leq 2H,
\end{equation*}
where we have used $\big\|{V}_{h+1}^{\rref,k^n}\big\|_{\infty}\leq H$ and $\big\|P^{k^n}_{h}\big\|_1 =\big\| P_{h, s_h^k,a_h^k}\big\|_1=1$. Substituting the preceding facts into \eqref{equ:lemma4vr5} validates the bound \eqref{eq:bound_I42} as long as $I_{4,2} > 0$. 
We have thus finished the proof of the claim \eqref{eq:bound_I42}.

\paragraph{Step 3: putting all pieces together.} 

Combining the results  \eqref{equ:lemma4 vr 7} and \eqref{eq:bound_I42} with \eqref{equ:lemma4 vr 6} yields
\begin{align*} 
	I_4 \leq |I_{4,1}| + I_{4,2} \leq
	C_4 \bigg\{ \sqrt{\frac{H^4}{{N_h^k}}\log^2\frac{SAT}{\delta}} + \frac{H^2}{{N_h^k}}\log^2\frac{SAT}{\delta} \bigg\}
\end{align*}
for some constant $C_4>0$. 
This bound taken together with the definition \eqref{eq:defn-I4-lem-equ4} of $I_4$ gives
\begin{align*}
 \frac{1}{N_{h}^{k}}\sum_{n=1}^{{N_{h}^{k}}}\Var_{h,s_{h}^{k},a_{h}^{k}}({V}_{h+1}^{\rref,k^{n}})
	& \leq \Big\{ {\sigma}_{h}^{\re,k^{N_{h}^{k}}+1}(s_{h}^{k},a_{h}^{k})-\big({\mu}_{h}^{\re,k^{N_{h}^{k}}+1}(s_{h}^{k},a_{h}^{k})\big)^{2} \Big\} \\
	& \qquad +C_4\bigg\{ \sqrt{\frac{H^{4}}{N_{h}^{k}}\log^{2}\frac{SAT}{\delta}}+\frac{H^{2}}{N_{h}^{k}}\log^{2}\frac{SAT}{\delta} \bigg\} .
\end{align*}
Invoke the elementary inequality $\sqrt{u+v}\leq \sqrt{u}+\sqrt{v}$ for any $u,v\geq 0$ and use the property \eqref{equ:positive-of-adv-12} to obtain
\begin{align*}
 & \bigg(\frac{1}{N_{h}^{k}}\sum_{n=1}^{{N_{h}^{k}}}\Var_{h,s_{h}^{k},a_{h}^{k}}({V}_{h+1}^{\rref,k^{n}})\bigg)^{1/2}\\
 & \qquad \lesssim \Big\{{\sigma}_{h}^{\re,k^{N_{h}^{k}}+1}(s_{h}^{k},a_{h}^{k})-\big({\mu}_{h}^{\re,k^{N_{h}^{k}}+1}(s_{h}^{k},a_{h}^{k})\big)^{2}\Big\}^{1/2}+\frac{H}{(N_{h}^{k})^{1/4}}\log^{1/2}\frac{SAT}{\delta}+\frac{H}{(N_{h}^{k})^{1/2}}\log\frac{SAT}{\delta}.
\end{align*}
Substitution into \eqref{lemma1:equ5-sub} directly establishes the desired result \eqref{lemma1:equ5}.

\subsection{Proof of Lemma~\ref{lem:Qk-lcb}}
\label{proof:lem:Qk-lcb}

\subsubsection{Proof of the inequalities \eqref{eq:Qk-lcb-upper}}

Suppose that we can verify the following inequality: 
\begin{equation}
	\label{eq:lcb-lower}
	Q^{\LCB, k}_h(s, a) \le Q_h^{\star}(s, a) \qquad\text{for all } (s, a, k ,h) \in  \cS \times \cA \times [K] \times [H] . 
\end{equation}
which in turn yields
\begin{align}
	\max_{a}Q_{h}^{\LCB,k}(s,a)\leq\max_{a}Q_{h}^{\star}(s,a)=V_{h}^{\star}(s)\qquad\text{for all }(k,h,s)\in[K]\times[H]\times\mathcal{S}.
	\label{eq:lcb-lower-135}
\end{align}
In addition, the construction of $V_{h}^{\LCB,k}$ (see line~\ref{eq:line-number-13-k} of Algorithm~\ref{algo:vr-k}) 
allows us to show that
\[
	V_{h}^{\LCB,k+1}(s) \leq  \max\Big\{ \max_{j:j\leq k+1}  \max_a Q^{\LCB, j}_h(s, a) , ~ \max_{j:j\leq k}V_h^{\LCB,j}(s) \Big\}. 
\]
This taken together with the initialization $V_{h}^{\LCB,1}=0$ and a simple induction argument yields
\begin{align}
	V_{h}^{\LCB,k}(s)\leq V_{h}^{\star}(s)\qquad\text{for all }(k,h,s)\in[K]\times[H]\times\mathcal{S}.
	\label{eq:lcb-V-lower-135}
\end{align}
As a consequence, everything comes down to proving the claim \eqref{eq:lcb-lower}, which we shall accomplish by induction.

\paragraph{Base case.}
Given our initialization, we have $$Q^{\LCB, 1}_h(s, a) - Q^{\star}_h(s,a) = 0 - Q^{\star}_h(s,a) \leq 0,$$
and hence the claim \eqref{eq:lcb-lower} holds trivially when $k=1$.

\paragraph{Induction step.} 

Suppose now that the claim \eqref{eq:lcb-lower} holds all the way up to $k$ for all $(s,a,h)$,
and we would like to validate it for the $(k+1)$-th episode as well. 
Towards this end, 
recall that the state-action pair $(s_h^k,a_h^k)$ is visited in the $k$-th episode at time step $h$;
this means that $Q_{h}^{\LCB }(s_h^k,a_h^k)$ is updated once we collect samples in the $k$-th episode,
with all other entries $Q_{h}^{\LCB }$ frozen.  
It thus suffices to verify that
\[
	Q^{\LCB , k+1}_h(s_h^k,a_h^k) \le Q_h^{\star}(s_h^k,a_h^k).
\]
In what follows, we shall adopt the short-hand notation (see also Section~\ref{sec:additional-notation})
\begin{align*}
	N_h^k = N_h^k(s_h^k,a_h^k)
	\qquad \text{and} \qquad
	k^n=k_h^n(s_h^k,a_h^k)
\end{align*}
which will be used throughout this subsection as long as it is clear from the context.

The update rule of $Q^{\LCB, k}_h$ (cf.~line~\ref{line:dadv4} of Algorithm~\ref{algo:subroutine}) and the Bellman optimality equation in \eqref{eq:bellman_opt_hk} tell us the following identities: 
\begin{align*}
	Q_{h}^{\LCB,k+1}(s_h^k,a_h^k) & = Q_{h}^{\LCB,k^{N_{h}^{k}}+1}(s_h^k,a_h^k) \\
& =(1-\eta_{N_{h}^{k}})Q_{h}^{\LCB,k^{N_{h}^{k}}}(s_h^k,a_h^k)+\eta_{N_{h}^{k}}\Big(r_h(s_h^k,a_h^k)+V_{h+1}^{\LCB,k^{N_{h}^{k}}}(s_{h+1}^{k^{N_{h}^{k}}})-b_{h}^{k^{N_{h}^{k}}}\Big),\\
Q_{h}^{\star}(s_h^k,a_h^k) & =(1-\eta_{N_{h}^{k}})Q_{h}^{\star}(s_h^k,a_h^k)+\eta_{N_{h}^{k}}Q_{h}^{\star}(s_h^k,a_h^k)\\
 & =(1-\eta_{N_{h}^{k}})Q_{h}^{\star}(s_h^k,a_h^k)+\eta_{N_{h}^{k}}\Big(r(s_h^k,a_h^k)+P_{h,s_h^k,a_h^k}V_{h+1}^{\star}\Big),
\end{align*}
which taken collectively lead to the following identity
\begin{align*}
&Q^{\LCB, k+1}_h(s_h^k,a_h^k) - Q_h^{\star}(s_h^k,a_h^k) = Q^{\LCB, k^{N_h^k}+1}_h(s_h^k,a_h^k) - Q_h^{\star}(s_h^k,a_h^k) \\
&\qquad  =(1-\eta_{N_h^k})\Big(Q^{\LCB, k^{N_h^k} }_h(s_h^k,a_h^k) - Q_h^{\star}(s_h^k,a_h^k)\Big)
+\eta_{N_h^k}\Big(V^{\LCB, k^{N_h^k}}_{h+1}(s^{k^{N_h^k}}_{h+1}) - P_{h, s_h^k,a_h^k}V^{\star}_{h+1} - b_h^{k^{N_h^k}} \Big) \\
&\qquad  =(1-\eta_{N_h^k})\Big(Q^{\LCB, k^{N_h^k-1}+1 }_h(s_h^k,a_h^k) - Q_h^{\star}(s_h^k,a_h^k)\Big)
+\eta_{N_h^k}\Big(V^{\LCB, k^{N_h^k}}_{h+1}(s^{k^{N_h^k}}_{h+1}) - P_{h, s_h^k,a_h^k}V^{\star}_{h+1} - b_h^{k^{N_h^k}} \Big) .
\end{align*}
Recall the definitions of $\eta_0^{N}$ and $\eta_n^{N}$ in \eqref{equ:learning rate notation}. 
Applying the above relation recursively and making use of the decomposition of $Q_h^{\star}(s_h^k,a_h^k)$ in \eqref{eq:decomp_Qhstar_eta}
result in
\begin{align} 
	& Q^{\LCB, k+1}_h(s_h^k,a_h^k)  - Q_h^{\star}(s_h^k,a_h^k) \nonumber\\
&\qquad = \eta^{N_h^k}_0 \left(Q^{\LCB, 1}_h(s_h^k,a_h^k) - Q_h^{\star}(s_h^k,a_h^k)\right) 
	 + \sum_{n = 1}^{N_h^k} \eta^{N_h^k}_n \left(V^{\LCB, k^n}_{h+1}(s^{k^n}_{h+1}) - P_{h, s_h^k,a_h^k}V^{\star}_{h+1} - b_h^{k^n} \right) \nonumber \\
	 &\qquad \leq \sum_{n = 1}^{N_h^k} \eta^{N_h^k}_n \left(V^{\LCB, k^n}_{h+1}(s^{k^n}_{h+1}) - V^{\star}_{h+1}(s^{k^n}_{h+1}) + \big(P_h^{k^n}  - P_{h, s_h^k,a_h^k}\big)V^{\star}_{h+1} - b_h^{k^n} \right),
	 \label{eq:lcb_bah}
\end{align}
where the inequality follows from the initialization $Q^{\LCB, 1}_h(s_h^k,a_h^k) = 0\leq Q_h^{\star}(s_h^k,a_h^k)$ and the definition of 
$P_h^{k}$ in \eqref{eq:P-hk-defn-s}. To continue, we invoke a result established in \citet[proof of Lemma 4.3]{jin2018q}, which guarantees that with probability at least $1- \delta$,
\begin{equation*}
 \sum_{n = 1}^{N_h^k} \eta^{N_h^k}_n \left(P_h^{k^n}  - P_{h, s_h^k,a_h^k}\right)V^{\star}_{h+1} \lesssim 
	\sqrt{ \frac{H^3 \log(\frac{SAT}{\delta})}{N_h^k} } \leq \sum_{n = 1}^{N_h^k} \eta^{N_h^k}_n b_h^{k^n},
\end{equation*}
provided that $\cb$ is some sufficiently large constant. 
Substituting the above relation into \eqref{eq:lcb_bah} implies that 
\begin{equation}
 Q^{\LCB, k+1}_h(s_h^k,a_h^k) - Q_h^{\star}(s_h^k,a_h^k) \leq \sum_{n = 1}^{N_h^k} \eta^{N_h^k}_n \left(V^{\LCB, k^n}_{h+1}(s^{k^n}_{h+1}) - V^{\star}_{h+1}(s^{k^n}_{h+1}) \right) \leq 0,
\end{equation}
where the last inequality follows from the induction hypothesis
 \begin{equation*}
	 V^{\LCB, j}_{h+1}(s) \leq V^{\star}_{h+1}(s) \qquad \text{for all } s \in \cS  \text{ and } j\leq k.
\end{equation*}
The proof is thus completed by induction.

\subsubsection{Proof of the inequality \eqref{eq:main-lemma}}

The proof of \eqref{eq:main-lemma} essentially follows the same arguments of \citet[Lemma~4.2]{yang2021q} (see also \citet[Lemma~C.7]{jin2018qarxiv}), an algebraic result leveraging certain relations w.r.t.~the Q-value estimates. Accounting for the  difference between our algorithm and the one in \cite{yang2021q}, we paraphrase \citet[Lemma~4.2]{yang2021q} into the following form that is convenient for our purpose.
\begin{lemma}[paraphrased from Lemma~4.2 in \cite{yang2021q}]\label{lem:yang-lem}
Assume there exists a constant $\cb>0$ such that for all $(s,a,k,h) \in \cS\times \cA \times [K] \times [H]$, it holds that 
\begin{align} \label{eq:Qk-ucb-lcb}
	0 &\leq Q^{k+1}_h(s, a) - Q^{\LCB, k+1}_h(s, a)  \notag\\
	& \le \eta^{N^{k}_h (s, a)}_0 H 
	 + \sum_{n = 1}^{N^{k}_h(s, a)} \eta^{N^{k}_h (s, a)}_n \left(V^{k^n}_{h+1}(s^{k^n}_{h+1}) - V^{\LCB, k^n}_{h+1}(s^{k^n}_{h+1}) \right) + 4 \cb \sqrt{\frac{H^3\log \frac{SAT}{\delta}}{N_h^k(s,a)}}.
\end{align}
Consider any $\varepsilon \in (0,H]$.	
Then for all $\beta =1,\ldots, \left \lceil \log_2 \frac{H}{\varepsilon} \right \rceil$, one has 
\begin{equation}\label{eq:yang-lemma}
	\Bigg|\sum_{h=1}^H\sum_{k=1}^K \mathds{1}\left( Q^{k}_h(s_h^k, a_h^k) - Q^{\LCB, k}_h(s_h^k, a_h^k)\in \big[2^{\beta-1}\varepsilon, 2^\beta \varepsilon \big) \right) \Bigg| \lesssim \frac{H^6SA\log\frac{SAT}{\delta}}{4^\beta \varepsilon^2}.
\end{equation}
\end{lemma}

We first show how to justify \eqref{eq:main-lemma} if 
the inequality \eqref{eq:yang-lemma} holds. 
As can be seen, the fact \eqref{eq:yang-lemma} immediately leads to 
\begin{align}
\sum_{h=1}^H \sum_{k=1}^K \mathds{1}\left(Q^{k}_h(s_h^k, a_h^k) - Q^{\LCB, k}_h(s_h^k, a_h^k) > \varepsilon \right) \lesssim \sum_{\beta =1}^{\left \lceil \log_2 \frac{H}{\varepsilon} \right \rceil} \frac{H^6SA\log\frac{SAT}{\delta}}{4^{\beta} \varepsilon^2} \leq \frac{H^6SA\log\frac{SAT}{\delta}}{2\varepsilon^2} 
\end{align}
as desired.

We now return to justify the claim \eqref{eq:yang-lemma}, towards which it suffices to demonstrate that \eqref{eq:Qk-ucb-lcb} holds. 
Lemma~\ref{lem:Qk-lower} and Lemma~\ref{lem:Qk-lcb} directly verify the left-hand side of \eqref{eq:Qk-ucb-lcb} since
\begin{equation}
	Q_h^{k}(s,a) \geq Q_h^{\star}(s, a) \geq Q_h^{\LCB, k}(s,a) \qquad 
	\text{for all } (s, a, k, h) \in  \cS \times \cA \times [K]\times [H].
\end{equation}
The remainder of the proof is thus devoted to justifying the upper bound on $Q^{k+1}_h(s, a) - Q^{\LCB, k+1}_h(s, a)$ in \eqref{eq:Qk-ucb-lcb}.
In view of the update rule in line~\ref{eq:line-number-11-k} of Algorithm~\ref{algo:vr-k}, we have the following basic fact 
$$Q^{k+1}_h(s, a) \leq Q^{\UCB, k+1}_h(s,a).$$  
This enables us to obtain 
\begin{align}
Q^{k+1}_h(s, a) - Q^{\LCB, k+1}_h(s, a) \le Q^{\UCB, k+1}_h(s, a) - Q^{\LCB, k+1}_h(s, a)=  Q^{\UCB, k^{N_h^k}+1}_h(s, a) - Q^{\LCB, k^{N_h^k}+1}_h(s, a),
	\label{eq:Qk-Q-LCB-UB-Q-UCB}
\end{align}
where we abbreviate $$N_h^k =N_h^k(s,a)$$ throughout this subsection as long as it is clear from the context. 
Making use of the update rules of $Q^{\UCB, k}_h$ and $Q^{\LCB, k}_h$ in line~\ref{line:dadv3} and line~\ref{line:dadv4} of Algorithm~\ref{algo:subroutine}, we reach
\begin{align*}
 & Q_{h}^{\UCB,k^{N_{h}^{k}}+1}(s,a)-Q_{h}^{\LCB,k^{N_{h}^{k}}+1}(s,a)\\
 & =(1-\eta_{N_{h}^{k}})Q_{h}^{\UCB,k^{N_{h}^{k}}}(s,a)+\eta_{N_{h}^{k}}\Bigg(r_h(s,a)+V_{h+1}^{k^{N_{h}^{k}}}(s_{h+1}^{k^{N_{h}^{k}}})+\cb\sqrt{\frac{H^{3}\log\frac{SAT}{\delta}}{N_{h}^{k}}}\Bigg)\\
 & \qquad-(1-\eta_{N_{h}^{k}})Q_{h}^{\LCB,k^{N_{h}^{k}}}(s,a)-\eta_{N_{h}^{k}}\Bigg(r_h(s,a)+V_{h+1}^{\LCB,k^{N_{h}^{k}}}(s_{h+1}^{k^{N_{h}^{k}}})-\cb\sqrt{\frac{H^{3}\log\frac{SAT}{\delta}}{N_{h}^{k}}}\Bigg)\\
 & =(1-\eta_{N_{h}^{k}})\Big(Q_{h}^{\UCB,k^{N_{h}^{k}}}(s,a)-Q_{h}^{\LCB,k^{N_{h}^{k}}}(s,a)\Big)
	+\eta_{N_{h}^{k}}\Bigg(V_{h+1}^{k^{N_{h}^{k}}}(s_{h+1}^{k^{N_{h}^{k}}})-V_{h+1}^{\LCB,k^{N_{h}^{k}}}(s_{h+1}^{k^{N_{h}^{k}}})+2\cb\sqrt{\frac{H^{3}\log\frac{SAT}{\delta}}{N_{h}^{k}}}\Bigg) \\
 & =(1-\eta_{N_{h}^{k}})\Big(Q_{h}^{\UCB,k^{N_{h}^{k}-1}+1}(s,a)-Q_{h}^{\LCB,k^{N_{h}^{k}}}(s,a)\Big)
	+\eta_{N_{h}^{k}}\Bigg(V_{h+1}^{k^{N_{h}^{k}}}(s_{h+1}^{k^{N_{h}^{k}}})-V_{h+1}^{\LCB,k^{N_{h}^{k}}}(s_{h+1}^{k^{N_{h}^{k}}})+2\cb\sqrt{\frac{H^{3}\log\frac{SAT}{\delta}}{N_{h}^{k}}}\Bigg).
\end{align*}
Applying this relation recursively leads to the desired result 
\begin{align*}
	& Q^{\UCB, k^{N_h^k}+1}_h(s, a) - Q^{\LCB, k^{N_h^k}+1}_h(s, a) 	\\
& \quad = \eta^{N^{k}_h}_0 \Big(Q_{h}^{\UCB,1}(s,a)-Q_{h}^{\LCB,1}(s,a)\Big) 
	 + \sum_{n = 1}^{N^{k}_h} \eta^{N^{k}_h}_n \Bigg(V^{k^n}_{h+1}(s^{k^n}_{h+1}) - V^{\LCB, k^n}_{h+1}(s^{k^n}_{h+1}) + 2\cb \sqrt{\frac{H^3\log \frac{SAT}{\delta}}{n}} \Bigg)\\
	 &\quad \leq \eta^{N^{k}_h}_0  H 
	 + \sum_{n = 1}^{N^{k}_h} \eta^{N^{k}_h}_n \left(V^{k^n}_{h+1}(s^{k^n}_{h+1}) - V^{\LCB, k^n}_{h+1}(s^{k^n}_{h+1}) \right) + 4\cb \sqrt{\frac{H^3\log \frac{SAT}{\delta}}{N_h^k}}.
\end{align*}
Here, the last line is valid due to the property $0\leq  Q_{h}^{\LCB,1}(s,a) \leq Q_{h}^{\UCB,1} (s,a)\leq H$ and the following fact
\begin{equation*}
	\sum_{n = 1}^{N^{k}_h} \eta^{N^{k}_h}_n \cb \sqrt{\frac{H^3\log \frac{SAT}{\delta}}{N_h^k}}   \leq 2\cb \sqrt{\frac{H^3\log \frac{SAT}{\delta}}{N_h^k}},
\end{equation*}
which is an immediate consequence of the elementary property $\sum_{n=1}^{N} \frac{\eta_n^{N}}{\sqrt{n}} \leq \frac{2}{\sqrt{N}}$ 
(see Lemma~\ref{lemma:property of learning rate}).
This combined with \eqref{eq:Qk-Q-LCB-UB-Q-UCB} establishes the condition \eqref{eq:Qk-ucb-lcb},
thus concluding the proof of the inequality \eqref{eq:main-lemma}.

\subsection{Proof of Lemma~\ref{lem:Vr_properties}}

\subsubsection{Proof of the inequality \eqref{eq:close-ref-v}}

Consider any state $s$ that has been visited at least once during the $K$ episodes.  
Throughout this proof, we shall adopt the shorthand notation
\[
	k^i=k_h^i(s),
\]
which denotes the index of the episode in which state $s$ is visited for the $i$-th time at step $h$. 
Given that $V_h (s)$ and $V_h^{\rref } (s)$ are only updated during the episodes with indices coming from $\{i \mid 1\leq k^i\leq K\}$, 
it suffices to show that for any $s$ and the corresponding $1\leq k^i \leq K$, the claim \eqref{eq:close-ref-v} holds in the sense that
\begin{align}
	\big| V_h^{k^i+1} (s) - V_h^{\rref, k^i+1} (s) \big| \leq 2. 
	\label{eq:close-ref-v-ki}
\end{align}
Towards this end, we look at three scenarios separately.

\paragraph{Case 1.} Suppose that $k^i$ obeys 
\begin{align}
	&V_h^{k^i+1} (s) - V_h^{\LCB, k^i+1} (s) > 1  \label{eq:condition-ki-update-1}
	\end{align}
	or

\begin{align}
	&V_h^{k^i+1} (s) - V_h^{\LCB, k^i+1} (s) \leq 1 \qquad \text{and} \qquad
	u_{\mathrm{ref}}^{k^i}(s) = \mathsf{True}  \label{eq:condition-ki-update-2}
\end{align}
The above conditions correspond to the ones in line~\ref{eq:line-number-15-k} and line~\ref{eq:line-number-17-k} of Algorithm~\ref{algo:vr-k} (meaning that $V^{\rref}_h$ is updated during the $k^i$-th episode), thus resulting in
\[
	V_h^{k^i+1} (s) = V_h^{\rref, k^i+1} (s). 
\]
This clearly satisfies \eqref{eq:close-ref-v-ki}.

\paragraph{Case 2.} 
Suppose that $k^{i_0}$ is the first time such that \eqref{eq:condition-ki-update-1} and \eqref{eq:condition-ki-update-2} are violated, namely,
\begin{align}
	{i_0} \coloneqq \min\left\{ j\mid V_{h}^{k^{j}+1}(s)-V_{h}^{\LCB,k^{j}+1}(s)\leq1 ~\text{ and }~ u_{\mathrm{ref}}^{k^j}(s) = \mathsf{False}\right\}.
	\label{eq:defn-first-i-violate-condition}
\end{align}
We make three observations.
\begin{itemize}
\item The definition \eqref{eq:defn-first-i-violate-condition} taken together with the update rules (lines~\ref{eq:line-number-15-k}-\ref{eq:line-number-18-k} of Algorithm~\ref{algo:vr-k}) reveals that  
$V_h^{\rref}$ has been updated in the $k^{i_0-1}$-th episode,  thus indicating that
\begin{align}
	V_{h}^{\rref,k^{i_0}}(s)=V_{h}^{\rref,k^{i_0-1}+1}(s)=V_{h}^{k^{i_0-1}+1}(s)=V_{h}^{k^{i_0}}(s). 
	\label{eq:lcb-final-value0}
\end{align}
\item Additionally, note that under the definition \eqref{eq:defn-first-i-violate-condition}, $	
V_h^{\rref} (s)$ is not updated during the $k^{i_0}$-th episode, namely, 
\begin{align}\label{eq:lcb-final-value}
	V_h^{\rref, k^{i_0}+1} (s) = V_h^{\rref, k^{i_0}} (s) .
\end{align}
\item
The definition of $k^{i_0}$ indicates that either \eqref{eq:condition-ki-update-1} or \eqref{eq:condition-ki-update-2} is satisfied in the previous episode $k^i = k^{i_0 -1}$ in which $s$ was visited. If \eqref{eq:condition-ki-update-1} is satisfied, then lines~\ref{eq:line-number-15-k}-\ref{eq:line-number-16-k} in Algorithm~\ref{algo:vr-k} tell us that
\begin{align}
	\mathsf{True} = u_{\mathrm{ref}}^{k^{i_0-1}+1}(s) = u_{\mathrm{ref}}^{k^{i_0}}(s),
\end{align}
which, however, contradicts the assumption $u_{\mathrm{ref}}^{k^{i_0}}(s) = \mathsf{False}$ in \eqref{eq:defn-first-i-violate-condition}. Therefore, in the $k^{i_0-1}$-th episode, \eqref{eq:condition-ki-update-2} is satisfied, thus leading to
\begin{align}\label{eq:lcb-final-value2}
	V_h^{k^{i_0}} (s) - V_h^{\LCB, k^{i_0}} (s)  = V_h^{k^{i_0-1}+1} (s) - V_h^{\LCB, k^{i_0-1}+1} (s) \leq 1.
\end{align}
\end{itemize}

We see from \eqref{eq:lcb-final-value0}, \eqref{eq:lcb-final-value} and \eqref{eq:lcb-final-value2} that 
\begin{align}
	V^{\rref, k^{i_0}+1}_h(s) -V^{k^{i_0} + 1 }_h(s) 
		&=  V^{\rref, k^{i_0}}_h(s) -V^{k^{i_0} + 1 }_h(s) 
		=  V^{k^{i_0}}_h(s) -V^{k^{i_0} + 1}_h(s) \label{eq:ref-final-value10}\\
		& \overset{\mathrm{(i)}}{\leq} V^{k^{i_0}}_h(s) -V^{\LCB, k^{i_0}}_h(s)\overset{\mathrm{(ii)}}{\leq} 1, 
		\label{eq:ref-final-value1}
\end{align}
where (i) holds since $V^{k^{i_0} + 1}_h(s) \geq V^\star_h(s) \geq V^{\LCB, k^{i_0}}_h(s)$, and (ii) follows from \eqref{eq:lcb-final-value2}.
In addition, we make note of the fact that
\begin{align}
	V^{\rref, k^{i_0}+1}_h(s) -V^{k^{i_0} + 1 }_h(s) =  V^{k^{i_0}}_h(s) -V^{k^{i_0} + 1}_h(s) \geq 0,
\end{align}
which follows from \eqref{eq:ref-final-value10} and the monotonicity of $V_h^k(s)$ in $k$.
With the above results in place, we arrive at the advertised bound \eqref{eq:close-ref-v-ki} when $i=i_0$.

\paragraph{Case 3.}  Consider any $i>i_0$. It is easily verified that
\begin{align}
	\label{eq:not_update_cond}
	V_h^{k^i+1} (s) - V_h^{\LCB, k^i+1} (s) \leq 1 \qquad \text{and} \qquad u_{\mathrm{ref}}^{k^i}(s) = \mathsf{False}.
\end{align}
It then follows that
\begin{align}
V_{h}^{\rref,k^{i}+1}(s) & \overset{(\mathrm{i})}{\leq}V_{h}^{\rref,k^{i_{0}}+1}(s)\overset{(\mathrm{ii})}{\leq}V_{h}^{k^{i_{0}}+1}(s)+1\overset{(\mathrm{iii})}{\leq}V_{h}^{\LCB,k^{i_{0}}+1}(s)+2 \notag\\
 & \overset{(\mathrm{iv})}{\leq}V_{h}^{\star}(s)+2\overset{(\mathrm{v})}{\leq}V_{h}^{k^{i}+1}(s)+2.
\end{align}
Here, (i) holds due to the monotonicity of $V_{h}^{\rref}$ and $V_{h}^{k}$ (see line~\ref{eq:line-number-13-k} of Algorithm~\ref{algo:vr-k}), (ii) is a consequence of \eqref{eq:ref-final-value1}, 
(iii) comes from the definition \eqref{eq:defn-first-i-violate-condition} of $i_0$, 
(iv) arises since $V_h^{\LCB}$ is a lower bound on $V^{\star }_h$ (see Lemma~\ref{lem:Qk-lcb}),
whereas (v) is valid since  $V^{k^i + 1}_h(s) \geq V^{\star}_h(s)$ (see  Lemma~\ref{lem:Qk-lower}).  
In addition, in view of the monotonicity of $V_{h}^{k}$ (see line~\ref{eq:line-number-13-k} of Algorithm~\ref{algo:vr-k}) 
and the update rule in line \ref{eq:line-number-16-k} of Algorithm~\ref{algo:vr-k}, we know that
\[
	V_{h}^{\rref,k^{i}+1}(s)\geq V_{h}^{k^{i}+1}(s).
\]
The preceding two bounds taken collectively demonstrate that
\[
	0\leq V_{h}^{\rref,k^{i}+1}(s)- V_{h}^{k^{i}+1}(s) \leq 2, 
\]
thus justifying \eqref{eq:close-ref-v-ki} for this case.

\bigskip \noindent
Therefore, we have established \eqref{eq:close-ref-v-ki}---and hence \eqref{eq:close-ref-v}---for all cases.

\subsubsection{Proof of the inequality \eqref{eq:Vr_lazy}}

Suppose that 
\begin{align}
	V^{\rref, k}_{h}(s^{k}_{h}) - V^{\rref, K}_{h}(s^{k}_{h}) \neq 0
	\label{eq:assumption-Vref-k-Vref-K-neq}
\end{align}
holds for some $k< K$. Then there are two possible scenarios to look at: 
\begin{itemize}
	\item[(a)] {\em Case 1: the condition in line~\ref{eq:line-number-15-k} and line~\ref{eq:line-number-17-k} of Algorithm~\ref{algo:vr-k} are violated at step $h$ of the $k$-th episode.} This means that we have
	\begin{equation}
	V_{h}^{k+1}(s_h^k) - V_{h}^{\LCB, k+1}(s_{h}^k) \leq 1 \quad \text{ and } \quad u_{\mathrm{ref}}^k(s_h^k) = \mathsf{False}
			\label{eq:Vh-k-V-LCB-k-test-135}
	\end{equation}
	in this case. Then for any $k' > k$, one necessarily has
\begin{align}
	\begin{cases}
		V_{h}^{k'}(s_h^k) - V_{h}^{\LCB, k'}(s_{h}^k)
		\leq V_{h}^{k+1}(s_h^k) - V_{h}^{\LCB, k+1}(s_{h}^k) \leq 1, \\
		u_{\mathrm{ref}}^{k'}(s_h^k) = u_{\mathrm{ref}}^k(s_h^k) = \mathsf{False},
	\end{cases}
		\label{eq:check-condition-kprime-135}
\end{align}
where the first property makes use of the monotonicity of $V_{h}^{k}$ and $V_{h}^{\LCB, k}$ (see \eqref{eq:monotonicity_Vh} and line~\ref{eq:line-number-13-k} of Algorithm~\ref{algo:vr-k}).  
In turn, Condition \eqref{eq:check-condition-kprime-135} implies that $V_h^{\rref}$ will no longer be updated after the $k$-th episode (see line~\ref{eq:line-number-15-k} of Algorithm~\ref{algo:vr-k}), thus indicating that 
\begin{align}
	V^{\rref, k}_{h}(s^{k}_{h}) = V^{\rref, {k+1}}_{h}(s^{k}_{h}) = \cdots = V^{\rref, K}_{h}(s^{k}_{h}). 
\end{align}
This, however, contradicts the assumption \eqref{eq:assumption-Vref-k-Vref-K-neq}.

\item[(b)] {\em Case 2: the condition in either line~\ref{eq:line-number-15-k} or line~\ref{eq:line-number-17-k} of Algorithm~\ref{algo:vr-k} is satisfied at step $h$ of the $k$-th episode.} If this occurs, then the update rule in line~\ref{eq:line-number-15-k} of Algorithm~\ref{algo:vr-k} implies that
\begin{equation}\label{eq:Vh-k-V-LCB-k-test-139}
	V_{h}^{k+1}(s_h^k) - V_{h}^{\LCB, k+1}(s_{h}^k) > 1,
\end{equation}
or
\begin{align}\label{eq:Vh-k-V-LCB-k-test-148}
	&V_h^{k+1} (s_h^k) - V_h^{\LCB, k+1} (s_h^k) \leq 1 \qquad \text{and} \quad
	u_{\mathrm{ref}}^k(s_h^k) = \mathsf{True}.
\end{align}

\end{itemize}
To summarize, the above argument demonstrates that \eqref{eq:assumption-Vref-k-Vref-K-neq} can only occur if either \eqref{eq:Vh-k-V-LCB-k-test-139} or \eqref{eq:Vh-k-V-LCB-k-test-148} holds.

With the above observation in place, we can proceed with the following decomposition: 
\begin{align} 
 & \sum_{h=1}^{H}\sum_{k=1}^{K}\left(V_{h}^{\rref,k}(s_{h}^{k})-V_{h}^{\rref,K}(s_{h}^{k})\right)=\sum_{h=1}^{H}\sum_{k=1}^{K}\left(V_{h}^{\rref,k}(s_{h}^{k})-V_{h}^{\rref,K}(s_{h}^{k})\right)\mathds{1}\left(V_{h}^{\rref,k}(s_{h}^{k})-V_{h}^{\rref,K}(s_{h}^{k})\neq0\right)\notag\\
 &\le\sum_{h=1}^{H}\sum_{k=1}^{K}\left(V_{h}^{\rref,k}(s_{h}^{k})-V_{h}^{\rref,K}(s_{h}^{k})\right)\mathds{1}\left( V_h^{k+1} (s_h^k) - V_h^{\LCB, k+1} (s_h^k) \leq 1 \text{ and } 
	u_{\mathrm{ref}}^k(s_h^k) = \mathsf{True}\right)\notag\\
 & \qquad+\underbrace{\sum_{h=1}^{H}\sum_{k=1}^{K}\left(V_{h}^{k}(s_{h}^{k})-V_{h}^{\LCB,k}(s_{h}^{k})\right)\mathds{1}\left(V_{h}^{k}(s_{h}^{k})-V_{h}^{\LCB,k}(s_{h}^{k})>1\right)}_{\eqqcolon \omega} .
	\label{eq:upper-ref-error-00}
\end{align}
Regarding the first term in \eqref{eq:upper-ref-error-00}, it is readily seen that for all $s\in \cS$,
\begin{align}
	\sum_{k=1}^{K}\mathds{1}\left( V_h^{k+1} (s ) - V_h^{\LCB, k+1} (s ) \leq 1 \text{ and }
	u_{\mathrm{ref}}^k(s ) = \mathsf{True}\right) \leq 1, \label{eq:line18-appear-once}
\end{align}
which arises since, for each $s\in \cS$, the above condition is satisfied in at most one episode, 
owing to the monotonicity property of $V_h, V_h^{\LCB}$ and the update rule for $u_{\mathrm{ref}}$ in \eqref{eq:line-number-17-k}. 
As a result, one has
\begin{align*}
 & \sum_{h=1}^{H}\sum_{k=1}^{K}\left(V_{h}^{\rref,k}(s_{h}^{k})-V_{h}^{\rref,K}(s_{h}^{k})\right)\mathds{1}\left(V_{h}^{k+1}(s_{h}^{k})-V_{h}^{\LCB,k+1}(s_{h}^{k})\leq1\text{ and }u_{\mathrm{ref}}^{k}(s_{h}^{k})=\mathsf{True}\right)\notag\\
 & \qquad\leq H\sum_{h=1}^{H}\sum_{k=1}^{K}\mathds{1}\left(V_{h}^{k+1}(s_{h}^{k})-V_{h}^{\LCB,k+1}(s_{h}^{k})\leq1\text{ and }u_{\mathrm{ref}}^{k}(s_{h}^{k})=\mathsf{True}\right)\notag\\
 & \qquad=H\sum_{h=1}^{H}\sum_{  s \in\mathcal{S}}  \sum_{k=1}^{K} \mathds{1}\left(V_{h}^{k+1}(s)-V_{h}^{\LCB,k+1}(s)\leq1\text{ and }u_{\mathrm{ref}}^{k}(s)=\mathsf{True}\right)  \notag\\
 & \qquad\leq H\sum_{h=1}^{H}\sum_{s\in\mathcal{S}}1=H^{2}S, 
\end{align*}
where the first inequality holds since $\|V_{h}^{\rref,k}-V_{h}^{\rref,K}\|_\infty\leq H$. 
Substitution into \eqref{eq:upper-ref-error-00} yields
\begin{align} 
  \sum_{h=1}^{H}\sum_{k=1}^{K}\left(V_{h}^{\rref,k}(s_{h}^{k})-V_{h}^{\rref,K}(s_{h}^{k})\right)
 & \leq H^2S+\omega. 
	\label{eq:upper-ref-error}
\end{align}

To complete the proof, it boils down to bounding the term $\omega$ defined in \eqref{eq:upper-ref-error-00}. To begin with, note that
\begin{equation*}
V^{\rref, K}_{h}(s^{k}_{h})  \geq V^\star_{h}(s^{k}_{h}) \geq V^{\LCB, k}_{h}(s_{h}^k), 
\end{equation*}
where we make use of the optimism of $V_{h}^{\rref,K}(s^{k}_{h})$ stated in Lemma~\ref{lem:Qk-lower} (cf.~\eqref{eq:Qk-lower}) and the pessimism of $V^{\LCB}_{h}$ in Lemma~\ref{lem:Qk-lcb} (see \eqref{eq:Qk-lcb-upper}). As a result, we can obtain
\begin{align}
\omega & = \sum_{h=1}^H \sum_{k=1}^K \left(V^{k}_{h}(s_{h}^k) - V^{\LCB, k}_{h}(s_{h}^k)\right)\mathds{1}\left(V^{k}_{h}(s_{h}^k) - V^{\LCB, k}_{h}(s_{h}^k) > 1\right) \nonumber \\
&\leq \sum_{h=1}^H\sum_{k=1}^K  \left(Q^{k}_{h}(s_{h}^k, a_{h}^k) - Q^{\LCB, k}_{h}(s_{h}^k, a_{h}^k)\right)\mathds{1}\left(Q^{k}_{h}(s_{h}^k, a_{h}^k) - Q^{\LCB, k}_{h}(s_{h}^k, a_{h}^k) > 1\right) ,\label{eq:ref-error-upper-final}
\end{align}
where the second line arises from the properties $V^{k}_{h}(s_{h}^k) = Q^{k}_{h}(s_{h}^k, a_{h}^k)$ (given that $a_h^k= \arg\max_a Q_h^k(s_h^k, a)$)  as well as the following fact (see line~\ref{eq:line-number-13-k} of Algorithm~\ref{algo:vr-k}) 
\[
	V^{\LCB, k}_{h}(s_{h}^k) \geq \max_a Q^{\LCB, k}_{h}(s_{h}^k, a) \geq Q^{\LCB, k}_h(s_{h}^k, a_{h}^k).
\]
Further, let us make note of the following elementary identity 
\begin{equation*}
	Q^{k}_{h}(s_{h}^k, a_{h}^k) - Q^{\LCB, k}_{h}(s_{h}^k, a_{h}^k)  = \int_0^{\infty} \mathds{1}\Big(Q^{k}_{h}(s_{h}^k, a_{h}^k) - Q^{\LCB, k}_{h}(s_{h}^k, a_{h}^k) > t \Big) \mathrm{d}t.
\end{equation*}
This allows us to obtain
\begin{align}
\omega & \leq\sum_{h=1}^{H}\sum_{k=1}^{K}\left\{ \int_{0}^{\infty}\mathds{1}\big(Q_{h}^{k}(s_{h}^{k},a_{h}^{k})-Q_{h}^{\LCB,k}(s_{h}^{k},a_{h}^{k})>t\big)\mathrm{d}t\right\} \ind\Big(Q_{h}^{k}(s_{h}^{k},a_{h}^{k})-Q_{h}^{\LCB,k}(s_{h}^{k},a_{h}^{k})>1\Big)\nonumber\\
 & =\int_{1}^{H}\sum_{h=1}^{H}\sum_{k=1}^{K}\mathds{1}\Big(Q_{h}^{k}(s_{h}^{k},a_{h}^{k})-Q_{h}^{\LCB,k}(s_{h}^{k},a_{h}^{k})>t\Big)\mathrm{d}t \notag\\
 & \lesssim\int_{1}^{H}\frac{H^{6}SA\log\frac{SAT}{\delta}}{t^{2}}\mathrm{d}t\lesssim H^{6}SA\log\frac{SAT}{\delta},
\label{eq:bound-ref-error-upper-final}
\end{align}
where the last line follows from the property \eqref{eq:main-lemma} in Lemma~\ref{lem:Qk-lcb}.
Combining the above bounds \eqref{eq:ref-error-upper-final} and \eqref{eq:bound-ref-error-upper-final} 
with \eqref{eq:upper-ref-error} establishes 
\begin{align} 
& \sum_{h=1}^H\sum_{k=1}^K \left(V^{\rref, k}_{h}(s^{k}_{h}) - V^{\rref, K}_{h}(s^{k}_{h})\right)  \notag\\
	& \leq H^2S +  \sum_{h=1}^H\sum_{k=1}^K  \left(Q^{k}_{h}(s_{h}^k, a_{h}^k) - Q^{\LCB, k}_{h}(s_{h}^k, a_{h}^k)\right)\mathds{1}\left(Q^{k}_{h}(s_{h}^k, a_{h}^k) - Q^{\LCB, k}_{h}(s_{h}^k, a_{h}^k) > 1\right) \notag\\
& \leq H^{6}SA\log\frac{SAT}{\delta} \notag
\end{align}
as claimed.

\section{Proof of Lemma~\ref{lem:Qk-upper}}\label{proof:lem:Qk-upper}

For notational simplicity, we shall adopt the short-hand notation
$$	k^n = k_h^n(s_h^k, a_h^k)  $$
throughout this section. 
A starting point for proving this lemma is the upper bound already derived in \eqref{eq:Vk-Qk},  
and we intend to further bound the first term on the right-hand side of \eqref{eq:Vk-Qk}. 
Recalling the expression of $Q^{\rref, k+1}_h(s_h^k, a_h^k)$ in \eqref{equ:lemma4 vr 1} and \eqref{eq:dejavu}, 
we can derive
\begin{align}
 & Q_{h}^{\rref,k}(s_{h}^{k},a_{h}^{k})-Q_{h}^{\star}(s_{h}^{k},a_{h}^{k})=Q_{h}^{\rref,k^{N_{h}^{k-1}(s_{h}^{k},a_{h}^{k})}+1}(s_{h}^{k},a_{h}^{k})-Q_{h}^{\star}(s_{h}^{k},a_{h}^{k})\\
 & \qquad=\eta_{0}^{N_{h}^{k-1}(s_{h}^{k},a_{h}^{k})}\left(Q_{h}^{\rref,1}(s_{h}^{k},a_{h}^{k})-Q_{h}^{\star}(s_{h}^{k},a_{h}^{k})\right)+\sum_{n=1}^{N_{h}^{k-1}(s_{h}^{k},a_{h}^{k})}\eta_{n}^{N_{h}^{k-1}(s_{h}^{k},a_{h}^{k})}b_{h}^{\rref,k^{n}+1}\nonumber\\
 & \qquad\qquad+\sum_{n=1}^{N_{h}^{k-1}(s_{h}^{k},a_{h}^{k})}\eta_{n}^{N_{h}^{k-1}(s_{h}^{k},a_{h}^{k})}\bigg(V_{h+1}^{k^{n}}(s_{h+1}^{k^{n}})-V_{h+1}^{\rref,k^{n}}(s_{h+1}^{k^{n}})+\frac{1}{n}\sum_{i=1}^{n}V_{h+1}^{\rref,k^{i}}(s_{h+1}^{k^{i}})-P_{h,s_{h}^{k},a_{h}^{k}}V_{h+1}^{\star}\bigg)\nonumber\\
 & \qquad\le\eta_{0}^{N_{h}^{k-1}(s_{h}^{k},a_{h}^{k})}H+B_{h}^{\rref,k}(s_{h}^{k},a_{h}^{k})+\frac{2\cb H^{2}}{{\big(N_{h}^{k-1}(s_{h}^{k},a_{h}^{k})\big)}^{3/4}}\log\frac{SAT}{\delta}\nonumber\\
 & \qquad\qquad+\sum_{n=1}^{N_{h}^{k-1}(s_{h}^{k},a_{h}^{k})}\eta_{n}^{N_{h}^{k-1}(s_{h}^{k},a_{h}^{k})}\bigg(V_{h+1}^{k^{n}}(s_{h+1}^{k^{n}})-V_{h+1}^{\rref,k^{n}}(s_{h+1}^{k^{n}})+\frac{1}{n}\sum_{i=1}^{n}V_{h+1}^{\rref,k^{i}}(s_{h+1}^{k^{i}})-P_{h,s_{h}^{k},a_{h}^{k}}V_{h+1}^{\star}\bigg),\nonumber 
\end{align} 
where the last line follows from \eqref{lemma1:equ10} with $B^{\rref, k^{N_h^{k-1}}+1}_h = B^{\rref, k}_h$  and the initialization $Q^{\rref, 1}_h(s_h^k, a_h^k) = H$.
Summing over all $1\leq k\leq K$ gives
\begin{align}
&\sum_{k = 1}^{K}  \Big( Q^{\rref, k}_h(s^k_h, a^k_h) - Q^{\star}_h(s^k_h, a^k_h) \Big) \nonumber \\
&\qquad  \le \sum_{k = 1}^{K} \bigg(H\eta^{N_h^{k-1}(s_h^k,a_h^k) }_0 + B^{\rref, k}_h(s^k_h, a^k_h) + \frac{2\cb H^2}{\big(N_h^{k-1}(s_h^k,a_h^k) \big)^{3/4}} \log\frac{SAT}{\delta}\bigg) \nonumber \\
& \qquad \qquad+ \sum_{k = 1}^{K} \sum_{n = 1}^{N_h^{k-1}(s_h^k,a_h^k) } \eta^{N_h^{k-1}(s_h^k,a_h^k) }_n \bigg(V^{k^n}_{h+1}(s^{k^n}_{h+1}) - V^{\rref, k^n}_{h+1}(s^{k^n}_{h+1}) + \frac{\sum_{i=1}^n V^{\rref, k^i}_{h+1}(s^{k^i}_{h+1})}{n} - P_{h, s_h^k, a_h^k}V^{\star}_{h+1} \bigg) \nonumber \\
& \qquad \le \sum_{k = 1}^{K} \bigg(H\eta^{N_h^{k-1}(s_h^k,a_h^k) }_0 + B^{\rref, k}_h(s^k_h, a^k_h) + \frac{2\cb H^2}{\big(N_h^{k-1}(s_h^k,a_h^k) \big)^{3/4}} \log\frac{SAT}{\delta}\bigg) \nonumber\\
&\qquad \qquad + \sum_{k = 1}^{K}\sum_{n = 1}^{N_h^{k-1}(s_h^k,a_h^k) } \eta^{N_h^{k-1}(s_h^k,a_h^k) }_n \Big(V^{k^n}_{h+1}(s^{k^n}_{h+1}) - V^{\star}_{h+1}(s^{k^n}_{h+1}) \Big)  \nonumber \\
&  \qquad \qquad+ \sum_{k = 1}^{K} \sum_{n = 1}^{N_h^{k-1}(s_h^k,a_h^k) } \eta^{N_h^{k-1}(s_h^k,a_h^k) }_n 
\bigg(V^{\star}_{h+1}(s^{k^n}_{h+1}) - V^{\rref, k^n}_{h+1}(s^{k^n}_{h+1}) + \frac{1}{n}\sum_{i=1}^n V^{\rref, k^i}_{h+1}(s^{k^i}_{h+1}) - P_{h, s_h^k, a_h^k}V^{\star}_{h+1} \bigg). \label{lemma1:equ8} 
\end{align}
Next, we control each term in \eqref{lemma1:equ8} separately.
\begin{itemize}

\item Regarding the first term of \eqref{lemma1:equ8}, we make two observations. To begin with, 
\begin{align}\label{eq:eta0eee}
	\sum_{k = 1}^{K}\eta^{N_h^{k-1}(s^k_h, a^k_h) }_0 
	\leq \sum_{(s,a) \in \mathcal{S}\times \mathcal{A}} \sum_{n = 0}^{N^{K - 1}_h(s,a)}\eta^{n}_0 \leq SA ,
\end{align}
where the last inequality follows since $\eta_0^n = 0$ for all $n>0$ (see \eqref{equ:learning rate notation}).
Next, it is also observed that
\begin{align}
	\sum_{k = 1}^{K}  \frac{1}{\big(N^{k-1}_h(s^k_h, a^k_h) \big)^{3/4}} &= \sum_{(s,a) \in \mathcal{S}\times \mathcal{A}} \sum_{n=1}^{N_h^{K- 1}(s,a) } \frac{1}{n^{3/4}} \notag \\
	&\leq \sum_{(s,a) \in \mathcal{S}\times \mathcal{A}} 4 \big(N_h^{K - 1}(s,a)\big)^{1/4} \leq 4(SA)^{3/4}K^{1/4}, \label{eq:n3/4-bound}
\end{align}
where the last inequality comes from Holder's inequality 
\[
	\sum_{(s,a) \in \mathcal{S}\times \mathcal{A}}  \big(N_h^{K-1}(s,a) \big)^{1/4} 
		\leq \Bigg[\sum_{(s,a) \in \mathcal{S}\times \mathcal{A}} 1\Bigg]^{3/4}  \Bigg[\sum_{(s,a) \in \mathcal{S}\times \mathcal{A}}  N_h^{K - 1}(s,a) \Bigg]^{1/4}  \leq (SA)^{3/4}K^{1/4}.
\]
Combine the above bounds to yield
\begin{align}
 & \sum_{k=1}^{K}\bigg(H\eta_{0}^{N_{h}^{k-1}(s_{h}^{k},a_{h}^{k})}+B_{h}^{\rref,k}(s_{h}^{k},a_{h}^{k})+\frac{2\cb H^{2}}{\big(N_{h}^{k-1}(s_{h}^{k},a_{h}^{k})\big)^{3/4}}\log\frac{SAT}{\delta}\bigg) \notag\\
 & \qquad\leq HSA+\sum_{k=1}^{K}B_{h}^{\rref,k}(s_{h}^{k},a_{h}^{k})+8\cb(SA)^{3/4}K^{1/4}H^{2}\log\frac{SAT}{\delta}.
\end{align}

\item We now turn to the second term of \eqref{lemma1:equ8}. A little algebra gives
\begin{align}
 & \sum_{k=1}^{K}\sum_{n=1}^{N_{h}^{k-1}(s_{h}^{k},a_{h}^{k})}\eta_{n}^{N_{h}^{k-1}(s_{h}^{k},a_{h}^{k})}\big(V_{h+1}^{k^{n}}(s_{h+1}^{k^{n}})-V_{h+1}^{\star}(s_{h+1}^{k^{n}})\big)\notag\\
 & \quad =\sum_{l=1}^{K}\sum_{N=N_{h}^{l}(s_{h}^{l},a_{h}^{l})}^{N_{h}^{K-1}(s_{h}^{l},a_{h}^{l})}\eta_{N_{h}^{l}(s_{h}^{l},a_{h}^{l})}^{N}\big(V_{h+1}^{l}(s_{h+1}^{l})-V_{h+1}^{\star}(s_{h+1}^{l})\big)\nonumber\\
 & \quad \le\left(1+\frac{1}{H}\right)\sum_{l=1}^{K}\big(V_{h+1}^{l}(s_{h+1}^{l})-V_{h+1}^{\star}(s_{h+1}^{l})\big)\nonumber\\
 & \quad =\left(1+\frac{1}{H}\right)\left[\sum_{k=1}^{K}\big(V_{h+1}^{k}(s_{h+1}^{k})-V_{h+1}^{\pi^{k}}(s_{h+1}^{k})\big)-\sum_{k=1}^{K}\big(V_{h+1}^{\star}(s_{h+1}^{k})-V_{h+1}^{\pi^{k}}(s_{h+1}^{k})\big)\right] . 
\label{equ:remove-v-star}
\end{align}
Here, the second line replaces $k^n$ (resp.~$n$) with $l$ (resp.~$N_h^l(s_h^l, a_h^l)$), 
the third line is due to the property $\sum_{N = n}^{\infty} \eta_n^N \le 1+ 1/H$ (see Lemma~\ref{lemma:property of learning rate}),
while the last relation replaces $l$ with $k$ again.

\item
When it comes to the last term of \eqref{lemma1:equ8}, we can derive
\begin{align*}
&\sum_{k = 1}^{K} \sum_{n = 1}^{N_h^{k-1}(s_h^k,a_h^k) } \eta^{N_h^{k-1}(s_h^k,a_h^k) }_n \Bigg( V^{\star}_{h+1}(s^{k^n}_{h+1}) - V^{\rref, k^n}_{h+1}(s^{k^n}_{h+1}) + \frac{1}{n} \sum_{i = 1}^n V^{\rref, k^i}_{h+1}(s^{k^i}_{h+1}) - P_{h, s_h^k, a_h^k}V^{\star}_{h+1} \Bigg) \nonumber \\
& = \sum_{k = 1}^K\sum_{n = 1}^{N_h^{k-1}(s_h^k,a_h^k)} \eta^{N^{k -1}_h(s^k_h, a^k_h) }_n 
\Bigg( \big(P^{k^n}_{h} -  P_{h, s^k_h, a^k_h} \big) \big(V^{\star}_{h+1} -  V^{\rref, k^n}_{h+1} \big) + \frac{1}{n} \sum_{i=1}^n \big(V^{\rref, k^i}_{h+1}(s^{k^i}_{h+1}) - P_{h, s^k_h, a^k_h}V^{\rref, k^n}_{h+1}\big)\Bigg) \nonumber \\
&= \sum_{k = 1}^{K} \sum_{N = N^{k}_h(s_h^k, a_h^k)}^{N^{K- 1}_h(s_h^k, a_h^k) }\eta^{N}_{N^{k}_h(s_h^k, a_h^k)} 
	\Bigg( \big(P^{k}_{h} -  P_{h, s^k_h, a^k_h} \big) \big(V^{\star}_{h+1} -  V^{\rref, k}_{h+1} \big)  + \frac{ \sum_{i =1}^{N^{k}_h(s_h^k, a_h^k)} \left(V^{\rref, k^i}_{h+1}(s^{k^i}_{h+1}) - P_{h, s^k_h, a^k_h}V^{\rref, k}_{h+1} \right) }{N^{k}_h(s_h^k, a_h^k)} \Bigg). 
\end{align*}
Here, the first equality holds since $V^{\star}_{h+1}(s^{k^n}_{h+1}) - V^{\rref, k^n}_{h+1}(s^{k^n}_{h+1}) = P_{h}^{k^n} \big(V^{\star}_{h+1} - V^{\rref, k^n}_{h+1} \big)$ (in view of the definition of $P_h^k$ in \eqref{eq:P-hk-defn-s}),  the second equality can be seen via simple rearrangement of the terms, 
while in the last line we replace $k^n$ (resp.~$n$) with $k$ (resp.~$N_h^k(s_h^k, a_h^k)$).  

\end{itemize}

\noindent Taking the above bounds together with \eqref{lemma1:equ8} and \eqref{eq:Vk-Qk}, we can rearrange terms to reach
\begin{align}
&\sum_{k = 1}^K \big(V^k_h(s^k_h) - V_h^{\pi^k}(s^k_h)\big) \notag\\
& \le \left(1+\frac{1}{H}\right) \sum_{k = 1}^K\big(V^{k}_{h+1}(s^{k}_{h+1}) - V^{\pi^k}_{h+1}(s^k_{h+1})\big)  + \sum_{k = 1}^{K}B^{\rref, k}_h(s^k_h, a^k_h)  \notag\\
&\; + HSA + 8\cb H^2 (SA)^{3/4} K^{1/4} \log\frac{SAT}{\delta}  + \sum_{k = 1}^K \big(P_{h, s^k_h, a^k_h} - P^k_h \big) \big(V^{\star}_{h+1} - V^{\pi^k}_{h+1} \big) \notag\\
	&\; + \sum_{k = 1}^{K} \sum_{N = N^{k}_h(s_h^k, a_h^k)}^{N^{K- 1}_h(s_h^k, a_h^k) }\eta^{N}_{N^{k}_h(s_h^k, a_h^k)} \Bigg[ \big(P^{k}_{h} -  P_{h, s^k_h, a^k_h} \big) \big(V^{\star}_{h+1} -  V^{\rref, k}_{h+1} \big)  + \frac{\sum_{i=1}^{N^{k}_h(s_h^k, a_h^k)} \left(V^{\rref, k^i}_{h+1}(s^{k^i}_{h+1}) - P_{h, s^k_h, a^k_h}V^{\rref, k}_{h+1} \right)}{N^{k}_h(s_h^k, a_h^k)} \Bigg] ,
	\label{eq:recursion-12345}
\end{align}
where we have dropped the term $- \frac{1}{H} \sum_{k}\big(V_{h+1}^{\star}(s_{h+1}^{k})-V_{h+1}^{\pi^{k}}(s_{h+1}^{k})\big)$ 
owing to the fact that $V_{h+1}^{\star}\geq V_{h+1}^{\pi^{k}}$.

Thus far, we have established a crucial connection between $\sum_{k = 1}^K \big(V^k_h(s^k_h) - V_h^{\pi^k}(s^k_h)\big)$ at step $h$ 
and  $\sum_{k = 1}^K\big(V^{k}_{h+1}(s^{k}_{h+1}) - V^{\pi^k}_{h+1}(s^k_{h+1})\big)$ at step $h+1$. 
Clearly, the term $V^{k}_{h+1}(s^{k}_{h+1}) - V^{\pi^k}_{h+1}(s^k_{h+1})$ can be further bounded in the same manner. 
As a result, by recursively applying the above relation \eqref{eq:recursion-12345} 
 over the time steps $h=1,2,\cdots, H$ and using the terminal condition $V^{k}_{H+1} = V^{\pi^k}_{H+1} = 0$, 
we can immediately arrive at the advertised bound in Lemma~\ref{lem:Qk-upper}.

\section{Proof of Lemma~\ref{lemma:bound_of_everything}}
\label{sec-proof-lemma:everything}

\subsection{Bounding the term $\mathcal{R}_1$}\label{sec-proof-lemma:extra-term}

 First of all, let us look at the first two terms of $\mathcal{R}_1$ in \eqref{eq:expression_R1}. 
Recognizing the following elementary inequality
\begin{equation}
	\label{equ:algebra property}
	\left(1+\frac{1}{H}\right)^{h-1} 
	\leq \left(1+\frac{1}{H}\right)^{H} \leq e  
	\quad \text{for all } h = 1, 2,\cdots, H+1,
\end{equation}
we are allowed to upper bound the first two terms in \eqref{eq:expression_R1} as follows:  
\begin{align}
	& \sum_{h = 1}^H \left(1+\frac{1}{H}\right)^{h-1} 
	\left\{ HSA +  8\cb H^2  (SA)^{3/4} K^{1/4}\log\frac{SAT}{\delta} \right\} 
	\lesssim H^2SA + H^3  (SA)^{3/4} K^{1/4}\log\frac{SAT}{\delta}\nonumber \\
	&\qquad \lesssim H^{4.5}SA\log^2\frac{SAT}{\delta} + \sqrt{H^3SAK} = H^{4.5}SA\log^2\frac{SAT}{\delta} + \sqrt{H^2SAT}, 
	\label{eq:secf:equ-num1}
\end{align}
where the last inequality can be shown using the AM-GM inequality as follows:
\begin{align*}
H^{3}(SA)^{3/4}K^{1/4}\log\frac{SAT}{\delta}=\Big(H^{9/4}\sqrt{SA}\log\frac{SAT}{\delta}\Big)\cdot(H^{3}SAK)^{1/4}\leq H^{4.5}SA\log^{2}\frac{SAT}{\delta}+\sqrt{H^{3}SAK}.
\end{align*}

We are now left with the last term of $\mathcal{R}_1$ in \eqref{eq:expression_R1}.
Towards this, we resort to Lemma~\ref{lemma:martingale-union-all2} by setting
\begin{align*}
	W_{h+1}^i \coloneqq {V}^{\star}_{h+1} - V_{h+1}^{\pi^k}
	\qquad \text{and} \qquad
	c_h \coloneqq \left(1+\frac{1}{H}\right)^{h-1}.  
\end{align*}
In view of \eqref{equ:algebra property} and the property $H\geq V^{\star}(s) \geq V^{\pi} (s)\geq 0$, we see that
\begin{align*}
   0\leq c_h \leq e,
	\qquad W_{h+1}^i \geq 0, 
	\qquad \text{and} \qquad
	  \|W_{h+1}^i\|_\infty \leq H \eqqcolon C_{\mathrm{w}}. 
\end{align*}
Therefore, applying Lemma~\ref{lemma:martingale-union-all2} yields
\begin{align}
	&\left|\sum_{h = 1}^H \left(1+\frac{1}{H}\right)^{h-1}\sum_{k = 1}^K \big(P_{h, s^k_h, a^k_h} - P^k_h \big) \big(V^{\star}_{h+1} - V^{\pi^k}_{h+1} \big)\right|  = \left|\sum_{h = 1}^H \sum_{k = 1}^K Y_{k,h} \right| \nonumber\\
	&\qquad \lesssim \sqrt{T C_{\mathrm{w}}^{2} \log\frac{1}{\delta}}  + C_{\mathrm{w}} \log\frac{1}{\delta} 
	= \sqrt{H^2T\log\frac{1}{\delta} } + H \log\frac{1}{\delta} \label{eq:secf:equ-num2}
\end{align}
with probability exceeding $1-\delta$.

Combining \eqref{eq:secf:equ-num1} and \eqref{eq:secf:equ-num2} with the definition \eqref{eq:expression_R1} of $\mathcal{R}_1$  immediately leads to the claimed bound.

\subsection{Bounding the term $\mathcal{R}_2$}\label{sec-proof-lemma:exploration}

In view of the definition of $B^{\rref, k}_h(s^k_h, a^k_h)$ in line~\ref{eq:line-number-19} of Algorithm~\ref{algo:subroutine}, we can decompose $\mathcal{R}_2$ (cf.~\eqref{eq:expression_R2})  as follows: 
\begin{align}
\mathcal{R}_2	&= \sum_{h = 1}^H \left(1+\frac{1}{H}\right)^{h-1}\cb\sqrt{H\log\frac{SAT}{\delta} }\sum_{k = 1}^K\sqrt{\frac{\sigma^{\adv, k}_h(s^k_h, a^k_h) - \big(\mu^{\adv, k}_h(s^k_h, a^k_h)\big)^2}{N_h^k(s_h^k, a_h^k)}} \nonumber \\
	 &\qquad + \sum_{h = 1}^H \left(1+\frac{1}{H}\right)^{h-1} \cb\sqrt{\log\frac{SAT}{\delta}} \sum_{k = 1}^K \sqrt{\frac{\sigma^{\re, k}_h(s^k_h, a^k_h) - \big(\mu^{\re, k}_h(s^k_h, a^k_h)\big)^2}{N_h^k(s_h^k, a_h^k)}} \nonumber \\
	 & \lesssim   \sqrt{H\log\frac{SAT}{\delta} } \sum_{h = 1}^H\sum_{k = 1}^K\sqrt{\frac{\sigma^{\adv, k}_h(s^k_h, a^k_h) - \big(\mu^{\adv, k}_h(s^k_h, a^k_h)\big)^2}{N_h^k(s_h^k, a_h^k)}}   \nonumber \\
	 & \qquad +\sqrt{\log\frac{SAT}{\delta} } \sum_{h = 1}^H    \sum_{k = 1}^K \sqrt{\frac{\sigma^{\re, k}_h(s^k_h, a^k_h) - \big(\mu^{\re, k}_h(s^k_h, a^k_h)\big)^2}{N_h^k(s_h^k, a_h^k)}} , \label{equ:bound of b_hat}
\end{align}
where the last relation holds due to \eqref{equ:algebra property}. 
In what follows, we intend to bound these two terms separately.

\paragraph{Step 1: upper bounding the first term in \eqref{equ:bound of b_hat}.}

Towards this, we make the observation that
\begin{align}
\sum_{k = 1}^K\sqrt{\frac{\sigma^{\adv, k}_h(s^k_h, a^k_h) - \big(\mu^{\adv, k}_h(s^k_h, a^k_h) \big)^2}{N^{k}_h(s^k_h, a^k_h)}} 
&\le \sum_{k = 1}^K\sqrt{\frac{\sigma^{\adv, k}_h(s^k_h, a^k_h)}{N^{k}_h(s^k_h, a^k_h)}} \nonumber \\
&= \sum_{k = 1}^K\sqrt{\frac{\sum_{n = 1}^{N^{k}_h(s^k_h, a^k_h)} \eta^{N^{k}_h(s^k_h, a^k_h)}_n\big(V^{k^n}_{h+1}(s^{k^n}_{h+1}) - V^{\rref, k^n}_{h+1}(s^{k^n}_{h+1})\big)^2}{N^{k}_h(s^k_h, a^k_h)}} , \label{eq:proof of var of var prelim}
\end{align}
where the second line follows from the update rule of $\sigma_h^{\adv, k}$  in \eqref{eq:recursion_mu_sigma_adv}.
Combining the relation 
$ | V^{k}_{h+1}(s_h^k) - V^{\rref, k}_{h+1}(s_h^k) | \le 2$ (cf.~\eqref{eq:close-ref-v}) and the property $\sum_{n = 1}^{N^{k}_h(s^k_h, a^k_h)} \eta^{N^{k}_h(s^k_h, a^k_h)}_n \leq 1$ (cf.~\eqref{eq:sum-eta-n-N}) with \eqref{eq:proof of var of var prelim} yields
\begin{align}
	\sum_{k = 1}^K\sqrt{\frac{\sigma^{\adv, k}_h(s^k_h, a^k_h) - \big(\mu^{\adv, k}_h(s^k_h, a^k_h) \big)^2}{N^{k}_h(s^k_h, a^k_h)}} 
	&\le \sum_{k = 1}^K\sqrt{\frac{4}{N^{k}_h(s^k_h, a^k_h)}} \le 2\sqrt{SAK}. \label{equ: proof of var of var}
\end{align}
Here, the last inequality holds due to the following fact:
\begin{align}
	\sum_{k = 1}^K\sqrt{\frac{1}{N^{k}_h(s^k_h, a^k_h)}} &= \sum_{(s, a)\in \cS\times \cA} \sum_{n = 1}^{N^{K}_h(s, a)} \sqrt{\frac{1}{n}} \le 2 \sum_{(s, a)\in \cS\times \cA} \sqrt{N^{K}_h(s, a)} \notag \\
	& \le  2  \sqrt{ \sum_{(s, a)\in \cS\times \cA} 1 } \cdot \sqrt{\sum_{(s,a)\in \cS\times \cA}N^{K}_h(s, a)}  =  2\sqrt{SAK},
	\label{equ:sak bound}
\end{align}
where the last line arises from  Cauchy-Schwarz  and the basic fact that $\sum_{(s,a)}N^{K}_h(s, a) = K$.

\paragraph{Step 2: upper bounding the second term in \eqref{equ:bound of b_hat}.}

Recalling the update rules of $\mu_h^{\re, k}$ and $\sigma_h^{\re, k}$ in \eqref{eq:recursion_mu_sigma_ref}, we have
\begin{align}
&\sum_{k = 1}^K\sqrt{\frac{\sigma^{\re, k}_h(s^k_h, a^k_h) - \big(\mu^{\re, k}_h(s^k_h, a^k_h)\big)^2}{N_h^k(s_h^k, a_h^k)}}\ \nonumber \\
&= \sum_{k = 1}^K\sqrt{\frac{1}{N^{k}_h(s^k_h, a^k_h)}} \underbrace{ \sqrt{ \frac{\sum_{n = 1}^{N^{k}_h(s^k_h, a^k_h)}\big(V^{\rref, k^n}_{h+1}(s^{k^n}_{h+1}) \big)^2}{N^{k}_h(s^k_h, a^k_h)} - \bigg(\frac{\sum_{n = 1}^{N^{k}_h(s^k_h, a^k_h)}V^{\rref, k^n}_{h+1}(s^{k^n}_{h+1})}{N^{k}_h(s^k_h, a^k_h)}\bigg)^2 }  }_{=: J_h^k}. \label{equ:adv-var1}
\end{align}
Additionally, the quantity $J_h^k$ defined in \eqref{equ:adv-var1} obeys
\begin{align}
(J_h^k)^2 & \leq\frac{\sum_{n = 1}^{N^{k}_h(s^k_h, a^k_h)}\big(V^{\rref, k^n}_{h+1}(s^{k^n}_{h+1})\big)^2 - \big(V^{\star}_{h+1}(s^{k^n}_{h+1}) \big)^2 }{N^{k}_h(s^k_h, a^k_h)} + \frac{\sum_{n = 1}^{N^{k}_h(s^k_h, a^k_h)}\big(V^{\star}_{h+1}(s^{k^n}_{h+1}) \big)^2}{N^{k}_h(s^k_h, a^k_h)}  - \Bigg(\frac{\sum_{n = 1}^{N^{k}_h(s^k_h, a^k_h)}V^{\star}_{h+1}(s^{k^n}_{h+1})}{N^{k}_h(s^k_h, a^k_h)}\Bigg)^2 \nonumber\\
\le& \underbrace{ \frac{\sum_{n = 1}^{N^{k}_h(s^k_h, a^k_h)}2H \big(V^{\rref, k^n}_{h+1}(s^{k^n}_{h+1}) - V^{\star}_{h+1}(s^{k^n}_{h+1}) \big)}{N^{k}_h(s^k_h, a^k_h)} }_{=:J_1}+ \underbrace{ \frac{\sum_{n = 1}^{N^{k}_h(s^k_h, a^k_h)}\big(V^{\star}_{h+1}(s^{k^n}_{h+1})\big)^2}{N^{k}_h(s^k_h, a^k_h)} - \Bigg(\frac{\sum_{n = 1}^{N^{k}_h(s^k_h, a^k_h)}V^{\star}_{h+1}(s^{k^n}_{h+1})}{N^{k}_h(s^k_h, a^k_h)}\Bigg)^2}_{=: J_2}, \label{equ:adv-var2}
\end{align}
which arises from the fact that $H\geq V^{\rref, k^n}_{h+1} \ge V^{\star}_{h+1}\geq 0$ for all $k^n \leq K$ and hence
\begin{align*}
\big(V_{h+1}^{\rref,k^{n}}(s_{h+1}^{k^{n}})\big)^{2}-\big(V_{h+1}^{\star}(s_{h+1}^{k^{n}})\big)^{2} & =\big(V_{h+1}^{\rref,k^{n}}(s_{h+1}^{k^{n}})+V_{h+1}^{\star}(s_{h+1}^{k^{n}})\big)\big(V_{h+1}^{\rref,k^{n}}(s_{h+1}^{k^{n}})-V_{h+1}^{\star}(s_{h+1}^{k^{n}})\big)\\
 & \le2H\big(V_{h+1}^{\rref,k^{n}}(s_{h+1}^{k^{n}})-V_{h+1}^{\star}(s_{h+1}^{k^{n}})\big).
\end{align*}
With \eqref{equ:adv-var2} in mind, 
we shall proceed to bound each term in \eqref{equ:adv-var2} separately. 

\begin{itemize}
\item The first term $J_1$ can be straightforwardly bounded as follows
\begin{align}
J_1 & = \frac{2H}{N_h^k(s_h^k,a_h^k)} \Bigg( \sum_{n = 1}^{N^{k}_h(s_h^k, a_h^k)}\Big(V^{\rref, k^n}_{h+1}(s^{k^n}_{h+1}) - V^{\star}_{h+1}(s^{k^n}_{h+1})\Big)\ind\Big(V^{\rref, k^n}_{h+1}(s^{k^n}_{h+1}) - V^{\star}_{h+1}(s^{k^n}_{h+1}) \leq 3\Big) + \Phi^{k}_h(s^k_h, a^k_h) \Bigg) \nonumber \\
&\le 6H +\frac{2H}{N_h^k(s_h^k,a_h^k)} \Phi^{k}_h(s^k_h, a^k_h), \label{equ:error until k-234}
\end{align}
where $\Phi^{k}_h(s^k_h, a^k_h)$ is defined as
\begin{equation} \label{eq:def_Phik}
\Phi^{k}_h(s_h^k, a_h^k) \defn \sum_{n = 1}^{N^{k}_h(s_h^k, a_h^k)}\Big(V^{\rref, k^n}_{h+1}(s^{k^n}_{h+1}) - V^{\star}_{h+1}(s^{k^n}_{h+1})\Big)\ind\Big(V^{\rref, k^n}_{h+1}(s^{k^n}_{h+1}) - V^{\star}_{h+1}(s^{k^n}_{h+1}) > 3\Big).
\end{equation}
\item When it comes to the second term $J_2$, we claim that 
\begin{align} \label{eq:var-Vstar}
J_2 \lesssim \Var_{h, s_h^k, a_h^k}(V^{\star}_{h+1}) + H^2\sqrt{\frac{\log\frac{SAT}{\delta}}{N_h^k(s_h^k,a_h^k)}},
\end{align}
which will be justified in Appendix~\ref{sec:proof:eq:var-Vstar}. 
\end{itemize}

Plugging \eqref{equ:error until k-234} and \eqref{eq:var-Vstar} into \eqref{equ:adv-var2} and \eqref{equ:adv-var1} 
allows one to demonstrate that
\begin{align}
&\sum_{k = 1}^K\sqrt{\frac{\sigma^{\re, k}_h(s^k_h, a^k_h) - \big(\mu^{\re, k}_h(s^k_h, a^k_h)\big)^2}{N_h^k(s_h^k, a_h^k)}} \nonumber \\
& \lesssim \sum_{k = 1}^K\sqrt{\frac{1}{N^{k}_h(s^k_h, a^k_h)}}\sqrt{H+\frac{H\Phi^{k}_h(s^k_h, a^k_h)}{N^{k}_h(s^k_h, a^k_h)} + \Var_{h, s^k_h, a^k_h}(V^{\star}_{h+1}) + H^2\sqrt{\frac{\log\frac{SAT}{\delta}}{N^{k}_h(s^k_h, a^k_h)}}} \nonumber\\
&\leq \sum_{k = 1}^K \Bigg( \sqrt{\frac{H }{N^{k}_h(s^k_h, a^k_h)}}  + \frac{ \sqrt{H\Phi^{k}_h(s^k_h, a^k_h)}}{N^{k}_h(s^k_h, a^k_h)} + \sqrt{ \frac{ \Var_{h, s^k_h, a^k_h}(V^{\star}_{h+1})}{N^{k}_h(s^k_h, a^k_h)}}  + \frac{H\log^{1/4}\frac{SAT}{\delta}}{ \big(N^{k}_h(s^k_h, a^k_h) \big)^{3/4}} \Bigg) \nonumber\\
& \lesssim \sqrt{HSAK} +\sum_{k=1}^K \frac{ \sqrt{H\Phi^{k}_h(s^k_h, a^k_h)}}{N^{k}_h(s^k_h, a^k_h)}+ \sum_{k = 1}^K\sqrt{\frac{\Var_{h, s^k_h, a^k_h}(V^{\star}_{h+1})}{N^{k}_h(s^k_h, a^k_h)}} + H(SA)^{3/4} \left(K\log\frac{SAT}{\delta} \right)^{1/4} ,
 \label{equ:bound of var}
\end{align}
where the last line follows from \eqref{equ:sak bound} and \eqref{eq:n3/4-bound}.

\paragraph{Step 3: putting together the preceding results.}

Finally,  the above results in \eqref{equ: proof of var of var} and \eqref{equ:bound of var} taken collectively with \eqref{equ:bound of b_hat} lead to
\begin{align*}
\mathcal{R}_2 &\lesssim \sqrt{H^3SAK\log\frac{SAT}{\delta}}  + \sum_{h = 1}^H \sqrt{\log\frac{SAT}{\delta}} \sum_{k = 1}^K \sqrt{\frac{\sigma^{\re, k}_h(s^k_h, a^k_h) - \big(\mu^{\re, k}_h(s^k_h, a^k_h)\big)^2}{N_h^k(s_h^k, a_h^k)}} \\
 &\lesssim \sqrt{H^3SAK\log\frac{SAT}{\delta}} + H^2(SA)^{3/4}K^{1/4}\log^{5/4}\frac{SAT}{\delta}  + \sqrt{\log\frac{SAT}{\delta} }\sum_{h=1}^H\sum_{k = 1}^K\sqrt{\frac{\Var_{h, s^k_h, a^k_h}(V^{\star}_{h+1})}{N^{k}_h(s^k_h, a^k_h)}} \\
 &\qquad +\sqrt{H \log\frac{SAT}{\delta} }\sum_{h=1}^H \sum_{k=1}^K \frac{ \sqrt{\Phi^{k}_h(s^k_h, a^k_h)}}{N^{k}_h(s^k_h, a^k_h)} \\
&\overset{\mathrm{(i)}}{\lesssim} \sqrt{H^3SAK\log\frac{SAT}{\delta}} + H^2(SA)^{3/4}K^{1/4}\log^{5/4}\frac{SAT}{\delta}  + H^4 SA\log^2\frac{SAT}{\delta}  \\
&\overset{\mathrm{(ii)}}{\lesssim} \sqrt{H^3SAK\log\frac{SAT}{\delta}} + H^4 SA\log^2\frac{SAT}{\delta} = \sqrt{H^2SAT\log\frac{SAT}{\delta}} + H^4 SA\log^2\frac{SAT}{\delta} .
\end{align*} 
Here, (i) holds due to the following two claimed inequalities
\begin{align} 
\sum_{h = 1}^H \sum_{k = 1}^K\sqrt{\frac{\Var_{h, s^k_h, a^k_h}(V^{\star}_{h+1})}{N^{k}_h(s^k_h, a^k_h)}} & \lesssim \sqrt{H^2SAT\log\frac{SAT}{\delta}} +  H^{4}SA\log\frac{SAT}{\delta},\label{eq:var-Vstar-sum}  \\
\sum_{h=1}^H \sum_{k=1}^K \frac{ \sqrt{\Phi^{k}_h(s^k_h, a^k_h)}}{N^{k}_h(s^k_h, a^k_h)} & \lesssim H^{7/2}SA\log^{3/2}\frac{SAT}{\delta}, \label{eq:bound-ref-error-2}
\end{align}
whose proofs are postponed to
to Appendix~\ref{sec:proof:eq:var-Vstar-sum} and Appendix~\ref{sec:proof:eq:bound-ref-error-2}, respectively. 
Additionally, the inequality (ii) above is valid since
\begin{align*}
 & H^{2}(SA)^{3/4}K^{1/4}\log^{5/4}\frac{SAT}{\delta}=\bigg(H^{5/4}(SA)^{1/2}\log\frac{SAT}{\delta}\bigg)\cdot\bigg(H^{3}SAK\log\frac{SAT}{\delta}\bigg)^{1/4}\\
 & \quad\lesssim H^{2.5}SA\log^{2}\frac{SAT}{\delta}+\sqrt{H^{3}SAK\log\frac{SAT}{\delta}}=H^{2.5}SA\log^{2}\frac{SAT}{\delta}+\sqrt{H^{2}SAT\log\frac{SAT}{\delta}}
\end{align*}
due to the Cauchy-Schwarz inequality. This concludes the proof of the advertised upper bound on $\mathcal{R}_2$.

\subsubsection{Proof of the inequality~\eqref{eq:var-Vstar}}\label{sec:proof:eq:var-Vstar}

Akin to the proof of $I_4^1$ in \eqref{equ:lemma4 vr 7}, let 
\begin{align*}
	W_{h+1}^i \coloneqq ({V}^{\star}_{h+1})^{ 2} 
	\qquad \text{and} \qquad
	u_h^i(s,a, N) \coloneqq \frac{1}{N}.
\end{align*}
By observing and setting
\begin{align*}
C_{\mathrm{u}} \coloneqq \frac{1}{N},\qquad  \|W_{h+1}^i\|_\infty \leq H^2 \eqqcolon C_{\mathrm{w}} ,
\end{align*} 
we can apply Lemma~\ref{lemma:martingale-union-all} to yield
\begin{align*}
	\bigg| \frac{1}{{N_h^k}}\sum_{n = 1}^{{N_h^k}} \big( V^{\star}_{h+1}(s^{k^n}_{h+1}) \big)^2 - P_{h, s_h^k, a_h^k}(V^{\star}_{h+1})^2 \bigg| = \bigg| \frac{1}{{N_h^k}}\sum_{n = 1}^{{N_h^k}} \big(P^{k^n}_h - P_{h, s_h^k, a_h^k} \big) \big(V^{\star}_{h+1} \big)^2  \bigg|
	\lesssim H^2\sqrt{\frac{\log^2\frac{SAT}{\delta}}{{N_h^k}}}
\end{align*}
with probability at least $1-\delta$.
Similarly, by applying the trivial bound $\|V_{h+1}^{\star}\|_{\infty} \leq H$ and Lemma~\ref{lemma:martingale-union-all}, we can obtain
\begin{align*}
\bigg| \frac{1}{{N_h^k}}\sum_{n = 1}^{{N_h^k}} V^{\star}_{h+1}(s^{k^n}_{h+1}) - P_{h, s_h^k, a_h^k}V^{\star}_{h+1} \bigg| 
	= \bigg| \frac{1}{{N_h^k}}\sum_{n = 1}^{{N_h^k}} \big(P^{k^n}_h - P_{h, s_h^k, a_h^k} \big)V^{\star}_{h+1} \bigg|
	\lesssim H\sqrt{\frac{\log\frac{SAT}{\delta}}{{N_h^k}}}
\end{align*}
with probability at least $1-\delta$.

Recalling from \eqref{lemma1:equ2} the definition 
$$\Var_{h, s_h^k, a_h^k}(V^{\star}_{h+1}) = P_{h, s_h^k, a_h^k}(V^{\star}_{h+1})^2 -  \big(P_{h, s_h^k, a_h^k} V^{\star}_{h+1} \big)^2 ,$$ 
we can use the preceding two bounds and the triangle inequality to show that: 
\begin{align*} 
&  \Bigg|\frac{1}{{N_h^k}}\sum_{n = 1}^{{N_h^k}}  V^{\star}_{h+1}(s^{k^n}_{h+1})^2 - \Bigg(\frac{1}{{N_h^k}}\sum_{n = 1}^{{N_h^k}} V^{\star}_{h+1}(s^{k^n}_{h+1}) \Bigg)^2 - \Var_{h, s_h^k, a_h^k}(V^{\star}_{h+1}) \Bigg|\\
&\quad\leq \Bigg|\frac{1}{{N_h^k}}\sum_{n = 1}^{{N_h^k}} V^{\star}_{h+1}(s^{k^n}_{h+1})^2 - P_{h, s_h^k, a_h^k}(V^{\star}_{h+1})^2 \Bigg| 
	+ \Bigg| \bigg(\frac{1}{{N_h^k}}\sum_{n = 1}^{{N_h^k}} V^{\star}_{h+1}(s^{k^n}_{h+1})\bigg)^2 - (P_{h, s_h^k, a_h^k}V^{\star}_{h+1})^2 \Bigg| \\
&\quad\lesssim H^2\sqrt{\frac{\log\frac{SAT}{\delta}}{{N_h^k}}} + \Bigg| \frac{1}{N_h^k}\sum_{n = 1}^{N_h^k} V^{\star}_{h+1}(s^{k^n}_{h+1})  - P_{h, s_h^k, a_h^k}V^{\star}_{h+1} \Bigg| \cdot  \Bigg| \frac{1}{{N_h^k}}\sum_{n = 1}^{{N_h^k}} V^{\star}_{h+1}(s^{k^n}_{h+1})  + P_{h, s_h^k, a_h^k}V^{\star}_{h+1} \Bigg| \\
&\quad\lesssim H^2\sqrt{\frac{\log\frac{SAT}{\delta}}{{N_h^k}}}
\end{align*}
with probability at least $1-\delta$, 
where the last line also makes use of the fact that $\|V_{h+1}^{\star}\|_{\infty} \leq H$.

\subsubsection{Proof of the inequality~\eqref{eq:var-Vstar-sum}}\label{sec:proof:eq:var-Vstar-sum}

To begin with, we make the observation that 
\begin{align*}
  \sum_{k = 1}^K\sqrt{\frac{\Var_{h, s^k_h, a^k_h}(V^{\star}_{h+1})}{N^{k}_h(s^k_h, a^k_h)}} & = \sum_{(s, a)\in\cS\times\cA} \sum_{n = 1}^{N^{K}_h(s, a)} \sqrt{\frac{\Var_{h, s, a}(V^{\star}_{h+1})}{n}} \leq 2\sum_{(s, a)\in\cS\times\cA}   \sqrt{N^{K}_h(s, a)\Var_{h, s, a}(V^{\star}_{h+1})},
\end{align*}
which relies on the fact that $\sum_{n = 1}^{N} 1/{\sqrt{n}} \leq 2\sqrt{N}$. 
It then follows that
\begin{align}
\sum_{h = 1}^H \sum_{k = 1}^K\sqrt{\frac{\Var_{h, s^k_h, a^k_h}(V^{\star}_{h+1})}{N^{k}_h(s^k_h, a^k_h)}} &\leq 2 \sum_{h = 1}^H \sum_{(s, a)\in\cS\times\cA}   \sqrt{N^{K}_h(s, a)\Var_{h, s, a}(V^{\star}_{h+1})}\nonumber \\
&\leq 2\sqrt{\sum_{h = 1}^H \sum_{(s, a)\in\cS\times\cA} 1}  \cdot \sqrt{ \sum_{h = 1}^H \sum_{(s, a)\in\cS\times\cA}   N^{K}_h(s, a)\Var_{h, s, a}(V^{\star}_{h+1})} \nonumber \\
&= 2 \sqrt{HSA}\sqrt{\sum_{h = 1}^H \sum_{k = 1}^K \Var_{h, s^k_h, a^k_h}(V^{\star}_{h+1})}\, , \label{equ:bound-var-v-optimal}
\end{align}
where the second inequality invokes the Cauchy-Schwarz inequality.

The rest of the proof is then dedicated to bounding \eqref{equ:bound-var-v-optimal}. 
Towards this end, we first decompose 
\begin{align}
\sum_{h = 1}^H \sum_{k = 1}^K \Var_{h, s^k_h, a^k_h}(V^{\star}_{h+1}) &\leq 
\sum_{h = 1}^H \sum_{k = 1}^K \Var_{h, s^k_h, a^k_h} \big( V^{\pi^k}_{h+1} \big) + \sum_{h = 1}^H \sum_{k = 1}^K \left|\Var_{h, s^k_h, a^k_h}(V^{\star}_{h+1}) - \Var_{h, s^k_h, a^k_h}(V^{\pi^k}_{h+1})\right| \nonumber \\
&\overset{\mathrm{(ii)}}{\lesssim} HT + H^3\log \frac{SAT}{\delta} + \sum_{h = 1}^H \sum_{k = 1}^K \left|\Var_{h, s^k_h, a^k_h}(V^{\star}_{h+1}) - \Var_{h, s^k_h, a^k_h}(V^{\pi^k}_{h+1})\right| , \label{equ:var-vsatr-1}
\end{align}
where (ii) follows directly from \citet[Lemma C.5]{jin2018qarxiv}.
The second term on the right-hand side of \eqref{equ:var-vsatr-1} can be bounded as follows 
\begin{align}
 & \sum_{h=1}^{H}\sum_{k=1}^{K}\left|\Var_{h,s_{h}^{k},a_{h}^{k}}(V_{h+1}^{\star})-\Var_{h,s_{h}^{k},a_{h}^{k}}(V_{h+1}^{\pi^{k}})\right|\nonumber\\
 & =\sum_{h=1}^{H}\sum_{k=1}^{K}\left|P_{h,s_{h}^{k},a_{h}^{k}}(V_{h+1}^{\star})^{2}-\big(P_{h,s_{h}^{k},a_{h}^{k}}V_{h+1}^{\star}\big)^{2}-P_{h,s_{h}^{k},a_{h}^{k}}(V_{h+1}^{\pi^{k}})^{2}+\big(P_{h,s_{h}^{k},a_{h}^{k}}V_{h+1}^{\pi^{k}}\big)^{2}\right| \notag\\
 & \leq\sum_{h=1}^{H}\sum_{k=1}^{K}\bigg\{\left|P_{h,s_{h}^{k},a_{h}^{k}}\Big(\big(V_{h+1}^{\star}-V_{h+1}^{\pi^{k}}\big)\big(V_{h+1}^{\star}+V_{h+1}^{\pi^{k}}\big)\Big)\right|+\Big|\big(P_{h,s_{h}^{k},a_{h}^{k}}V_{h+1}^{\star}\big)^{2}-\big(P_{h,s_{h}^{k},a_{h}^{k}}V_{h+1}^{\pi^{k}}\big)^{2}\Big|\bigg\} \notag\\
 & \overset{(\mathrm{i})}{\leq}4H\sum_{h=1}^{H}\sum_{k=1}^{K}P_{h,s_{h}^{k},a_{h}^{k}}\big(V_{h+1}^{\star}-V_{h+1}^{\pi^{k}}\big)  \notag\\
 & =4H\sum_{h=1}^{H}\sum_{k=1}^{K}\left\{ V_{h+1}^{\star}(s_{h+1}^{k})-V_{h+1}^{\pi^{k}}(s_{h+1}^{k})+\big(P_{h,s_{h}^{k},a_{h}^{k}}-P_{h}^{k}\big)\big(V_{h+1}^{\star}-V_{h+1}^{\pi^{k}}\big)\right\}  \nonumber\\
	& \overset{(\mathrm{ii})}{\leq} 4H \sum_{h=1}^{H}\sum_{k=1}^{K}\big(\phi_{h+1}^{k}+\delta_{h+1}^{k}\big) 
 \overset{(\mathrm{iii})}{\lesssim}H^{2}\sqrt{T\log\frac{SAT}{\delta}}+ H^4\sqrt{SAT\log\frac{SAT}{\delta}} + H^4SA \notag\\
	& \asymp H^4\sqrt{SAT\log\frac{SAT}{\delta}} + H^4SA,
	\label{equ:var-vsatr-2}
\end{align}
where we define
\begin{align}
	 \delta_{h+1}^k \defn V_{h+1}^{\UCB,k}(s_{h+1}^k) - V_{h+1}^{\pi^k}(s_{h+1}^k),  \qquad \phi_{h+1}^k \defn \big(P_{h, s_h^k, a_h^k} -P_{h}^k\big) \big(V_{h+1}^\star - V_{h+1}^{\pi^k}\big). 
\end{align}
We shall take a moment to explain how we derive \eqref{equ:var-vsatr-2}. The inequality (i) holds by observing that $V_{h+1}^{\star}-V_{h+1}^{\pi^{k}}\geq 0$ and 
\begin{align*}
\left|P_{h,s_{h}^{k},a_{h}^{k}}\Big(\big(V_{h+1}^{\star}-V_{h+1}^{\pi^{k}}\big)\big(V_{h+1}^{\star}+V_{h+1}^{\pi^{k}}\big)\Big)\right| & \leq P_{h,s_{h}^{k},a_{h}^{k}}\big(V_{h+1}^{\star}-V_{h+1}^{\pi^{k}}\big)\Big(\big\| V_{h+1}^{\star}\big\|_{\infty}+\big\| V_{h+1}^{\pi^{k}}\big\|_{\infty}\big)\\
 & \leq2HP_{h,s_{h}^{k},a_{h}^{k}}\big(V_{h+1}^{\star}-V_{h+1}^{\pi^{k}}\big),\\
\bigg|\big(P_{h,s_{h}^{k},a_{h}^{k}}V_{h+1}^{\star}\big)^{2}-\big(P_{h,s_{h}^{k},a_{h}^{k}}V_{h+1}^{\pi^{k}}\big)^{2}\bigg| & \leq\Big|P_{h,s_{h}^{k},a_{h}^{k}}\big(V_{h+1}^{\star}-V_{h+1}^{\pi^{k}}\big)\Big| \cdot \Big|P_{h,s_{h}^{k},a_{h}^{k}}\big(V_{h+1}^{\star}+V_{h+1}^{\pi^{k}}\big)\Big|\\
 & \leq2HP_{h,s_{h}^{k},a_{h}^{k}}\big(V_{h+1}^{\star}-V_{h+1}^{\pi^{k}}\big);
\end{align*}
(ii) is valid since $V_{h+1}^{\UCB}\geq V_{h+1}^{\star}$;  
and (iii) results from the following two bounds:
\begin{subequations}
\begin{align}
	\sum_{h=1}^{H}\sum_{k=1}^{K}\delta_{h+1}^{k} & \lesssim H^{3}\sqrt{SAT\log\frac{SAT}{\delta}} + H^3 SA, \label{eq:sum-delta-k-bound}\\
\sum_{h=1}^{H}\sum_{k=1}^{K}\phi_{h+1}^{k} & \lesssim H\sqrt{T\log\frac{SAT}{\delta}},
\end{align}
\end{subequations}
which come respectively from \citet[Eqn.~(C.13)]{jin2018qarxiv} and the argument for \citet[Eqn.~(C.12)]{jin2018qarxiv}.\footnote{Note that the notation $\delta_h^k$ used in \citet[Section C.2]{jin2018qarxiv} and the one in the proof of \citet[Theorem 1]{jin2018qarxiv} are different; here, we need to adopt the notation used in the proof of \citet[Theorem 1]{jin2018qarxiv}.}

As a consequence, substituting \eqref{equ:var-vsatr-1} and \eqref{equ:var-vsatr-2} into \eqref{equ:bound-var-v-optimal}, 
we reach 
\begin{align*}
\sum_{h=1}^{H}\sum_{k=1}^{K}\sqrt{\frac{\Var_{h,s_{h}^{k},a_{h}^{k}}(V_{h+1}^{\star})}{N_{h}^{k}(s_{h}^{k},a_{h}^{k})}} & \lesssim\sqrt{HSA}\sqrt{HT+H^{4}\sqrt{SAT\log\frac{SAT}{\delta}} + H^4SA }\\
	& \lesssim\sqrt{H^{2}SAT}+H^{5/2}(SA)^{3/4}\Big(T\log\frac{SAT}{\delta}\Big)^{1/4} + H^{2.5} SA\\
 & =\sqrt{H^{2}SAT}+\Big(H^{2}SAT\log\frac{SAT}{\delta}\Big)^{1/4}\big(H^{4}SA\big)^{1/2} + H^{2.5} SA \\
 & \lesssim\sqrt{H^{2}SAT\log\frac{SAT}{\delta}}+H^{4}SA\log\frac{SAT}{\delta},
\end{align*}
where we have applied the basic inequality $2ab\leq a^2 + b^2$ for any $a,b\geq 0$.

\subsubsection{Proof of the inequality~\eqref{eq:bound-ref-error-2}}\label{sec:proof:eq:bound-ref-error-2}

First, it is observed that
\begin{align}
  \sum_{k = 1}^K\frac{\sqrt{\Phi_h^k(s^k_h, a^k_h)}}{N^{k}_h(s^k_h, a^k_h)}&   =\sum_{(s, a)\in\cS\times\cA} \sum_{n = 1}^{N^{K}_h(s, a)} \frac{\sqrt{\Phi_h^{k^n(s, a)}(s, a)}}{n} \nonumber \\
&\le\sum_{(s, a)\in\cS\times\cA}  \sqrt{\Phi_h^{N_h^K(s,a)}(s, a)}\log T 
	\le\sqrt{SA \sum_{(s, a)\in\cS\times\cA}  \Phi_h^{N_h^K(s,a)}(s, a)}\log T .\label{eq:phi_h_k_intermediate}
\end{align}
Here, the first inequality holds by the monotonicity property of $\Phi_h^{k}(s_h, a_h)$ with respect to $k$ (see its definition in \eqref{eq:def_Phik}) due to the same property of $V_{h+1}^{\rref,k}$, while the second inequality comes from Cauchy-Schwarz.

To continue, note that 
\begin{align}
 & \sum_{h=1}^{H}\sqrt{\sum_{(s,a)\in\cS\times\cA}\Phi_{h}^{N_{h}^{K}(s,a)}(s,a)}=\sum_{h=1}^{H}\sqrt{\sum_{k=1}^{K}\Big(V_{h+1}^{\rref,k}(s_{h+1}^{k})-V_{h+1}^{\star}(s_{h+1}^{k})\Big)\ind\Big(V_{h+1}^{\rref,k}(s_{h+1}^{k})-V_{h+1}^{\star}(s_{h+1}^{k})>3\Big)}\nonumber\\
 & \leq\sum_{h=1}^{H}\sqrt{\sum_{k=1}^{K}\Big(V_{h+1}^{k}(s_{h+1}^{k})+2-V_{h+1}^{\LCB,k}(s_{h+1}^{k})\Big)\ind\Big(V_{h+1}^{k}(s_{h+1}^{k})+2-V_{h+1}^{\LCB,k}(s_{h+1}^{k})>3\Big)}\nonumber\\
 & =\sum_{h=1}^{H}\sqrt{\sum_{k=1}^{K}\Big(V_{h+1}^{k}(s_{h+1}^{k})+2-V_{h+1}^{\LCB,k}(s_{h+1}^{k})\Big)\ind\Big(V_{h+1}^{k}(s_{h+1}^{k})-V_{h+1}^{\LCB,k}(s_{h+1}^{k})>1\Big)}\nonumber\\
 & \leq\sum_{h=1}^{H}\sqrt{\sum_{k=1}^{K}3\Big(V_{h+1}^{k}(s_{h+1}^{k})-V_{h+1}^{\LCB,k}(s_{h+1}^{k})\Big)\ind\Big(V_{h+1}^{k}(s_{h+1}^{k})-V_{h+1}^{\LCB,k}(s_{h+1}^{k})>1\Big)}\nonumber\\
 & \leq\sqrt{H}\sqrt{\sum_{h=1}^{H}\sum_{k=1}^{K}3\Big(V_{h+1}^{k}(s_{h+1}^{k})-V_{h+1}^{\LCB,k}(s_{h+1}^{k})\Big)\ind\Big(V_{h+1}^{k}(s_{h+1}^{k})-V_{h+1}^{\LCB,k}(s_{h+1}^{k})>1\Big)}, 
\end{align}
where the first inequality follows from Lemma~\ref{lem:Vr_properties} (cf.~\eqref{eq:close-ref-v}) and Lemma~\ref{lem:Qk-lcb} 
(so that $V_{h+1}^{\rref,k}(s_{h+1}^{k})-V_{h+1}^{\star}(s_{h+1}^{k})\leq V_{h+1}^{k}(s_{h+1}^{k})+2-V_{h+1}^{\LCB,k}(s_{h+1}^{k})$), the penultimate inequality holds since $1 \leq   V_{h+1}^{k}(s_{h+1}^{k})-V_{h+1}^{\LCB,k}(s_{h+1}^{k}) $ when $\ind\Big(V_{h+1}^{k}(s_{h+1}^{k})-V_{h+1}^{\LCB,k}(s_{h+1}^{k})>1\Big) \neq 0$, 
and the last inequality is a consequence of the Cauchy-Schwarz inequality.

Combining the above relation with \eqref{eq:ref-error-upper-final} and applying the triangle inequality, we can demonstrate that
\begin{align*}
&\sum_{h=1}^H \sqrt{  \sum_{(s, a)\in\cS\times\cA}  \Phi_h^{N_h^K(s,a)}(s, a)}  \\
	& \lesssim \sqrt{H}\sqrt{ \sum_{h=1}^H\sum_{k=1}^K  \left(Q^{k}_{h+1}(s_{h+1}^k, a_{h+1}^k) - Q^{\LCB, k}_{h+1}(s_{h+1}^k, a_{h+1}^k)\right)\mathds{1}\left(Q^{k}_{h+1}(s_{h+1}^k, a_{h+1}^k) - Q^{\LCB, k}_{h+1}(s_{h+1}^k, a_{h+1}^k) > 1\right)}\\
	&\lesssim \sqrt{H^7SA\log\frac{SAT}{\delta}},
\end{align*}
where the second inequality follows directly from \eqref{eq:Vr_lazy}, 
and the first inequality is valid since
\begin{align*}
V_{h+1}^{k}(s_{h+1}^{k})-V_{h+1}^{\LCB,k}(s_{h+1}^{k}) 
	& \leq Q_{h+1}^{k}(s_{h+1}^{k},a_{h+1}^{k})-Q_{h+1}^{\LCB,k}(s_{h+1}^{k},a_{h+1}^{k}). 
\end{align*}
Substitution into \eqref{eq:phi_h_k_intermediate} gives
\begin{align*}
\sum_{h=1}^{H}\sum_{k=1}^{K}\frac{\sqrt{\Phi_{h}^{k}(s_{h}^{k},a_{h}^{k})}}{N_{h}^{k}(s_{h}^{k},a_{h}^{k})} & \lesssim\left(\sqrt{SA}\log T\right)\cdot\sqrt{H^{7}SA\log\frac{SAT}{\delta}}\asymp H^{7/2}SA\log^{3/2}\frac{SAT}{\delta}, 
\end{align*}
thus concluding the proof.

\subsection{Bounding the term $\mathcal{R}_3$}\label{sec-proof-lemma:reference}

For notational convenience, we shall use the short-hand notation
\[
	k^i \defn k^i_h(s_h^k, a_h^k)
\]
whenever it is clear from the context. 
This allows us to decompose the expression of $\mathcal{R}_3$ in \eqref{eq:expression_R3} as follows
\[
	\mathcal{R}_3:= \underbrace{\sum_{h = 1}^H \sum_{k = 1}^K\lambda^{k}_h \big(P^k_{h} -  P_{h, s^k_h, a^k_h} \big) \big( V^{\star}_{h+1} - V^{\rref, k}_{h+1} \big)}_{\eqqcolon \mathcal{R}_3^1} + \underbrace{ \sum_{h = 1}^H \sum_{k = 1}^K\lambda^{k}_h \frac{\sum_{i \le N^{k}_h(s^k_h, a^k_h)} \big(V^{\rref, k^i}_{h+1}(s^{k^i}_{h+1}) - P_{h, s^k_h, a^k_h}V^{\rref, k}_{h+1}\big)}{N^{k}_h(s^k_h, a^k_h)} }_{\eqqcolon \mathcal{R}_3^2} 
\]
with
\begin{equation} 
	\label{eq:lambda_kh_bound}
	\lambda^{k}_h \defn \left(1+\frac{1}{H}\right)^{h-1}~\sum_{n = N^{k}_h(s^k_h, a^k_h)}^{N^{K- 1}_h(s^k_h, a^k_h) } \eta^{n}_{N^{k}_h(s^k_h, a^k_h)} 
	\le \left(1+\frac{1}{H}\right)^{h} \le \left(1+\frac{1}{H}\right)^{H} \leq e. 
\end{equation} 
Here, the first inequality in \eqref{eq:lambda_kh_bound} follows from the property $\sum_{N = n}^{\infty} \eta_n^N \le 1+ 1/H$ in Lemma~\ref{lemma:property of learning rate}, while the last inequality in \eqref{eq:lambda_kh_bound} results from \eqref{equ:algebra property}. 
In the sequel, we shall control each of these two terms separately.

\paragraph{Step 1: upper bounding $\mathcal{R}_3^1$.} 
We plan to control this term by means of Lemma~\ref{lemma:martingale-union-all2}. For notational simplicity, let us define
\begin{align*}
	N(s,a,h) \coloneqq N_h^{K-1}(s,a) 
\end{align*}
and set
\begin{align*}
	W_{h+1}^i \coloneqq V^{\rref, k}_{h+1} - V^{\star}_{h+1}
	\qquad \text{and} \qquad
	u_h^i(s_h^i,a_h^i) \coloneqq \lambda_h^i = \left(1+\frac{1}{H}\right)^{h-1}\sum_{n = N^{i}_h(s^i_h, a^i_h)}^{N(s^i_h, a^i_h,h)} \eta^{n}_{N^{i}_h(s^i_h, a^i_h)}.  
\end{align*}
Given the fact that $V^{\rref, k}_{h+1}(s),  V^{\star}_{h+1}(s)\in [0,H]$ and the condition \eqref{eq:lambda_kh_bound},
it is readily seen that
\begin{align*}
   \big|u_h^i(s_h^i,a_h^i) \big| \leq e \eqqcolon C_{\mathrm{u}}
	\qquad \text{and} \qquad
	  \big\| W_{h+1}^i \big\|_\infty \leq H \eqqcolon C_{\mathrm{w}}. 
\end{align*}
Apply Lemma~\ref{lemma:martingale-union-all2} to yield
\begin{align}
 & \left|\sum_{h=1}^{H}\sum_{k=1}^{K}\lambda_{h}^{k}\big(P_{h}^{k}-P_{h,s_{h}^{k},a_{h}^{k}}\big)\big(V_{h+1}^{\star}-V_{h+1}^{\rref,k}\big)\right|=\left|\sum_{h=1}^{H}\sum_{k=1}^{K}X_{k,h}\right|\nonumber\nonumber \\
 & \qquad\lesssim\sqrt{C_{\mathrm{u}}^{2}C_{\mathrm{w}}HSA\sum_{h=1}^{H}\sum_{i=1}^{K}\mathbb{E}_{i,h-1}\left[P_{h}^{i}W_{h+1}^{i}\right]\log\frac{K}{\delta}}+C_{\mathrm{u}}C_{\mathrm{w}}HSA\log\frac{K}{\delta}\nonumber\nonumber \\
 & \qquad\lesssim\sqrt{H^{2}SA\sum_{h=1}^{H}\sum_{k=1}^{K}\mathbb{E}_{i,h-1}\left[P_{h}^{k}\big(V_{h+1}^{\rref,k}-V_{h+1}^{\star}\big)\right]\log\frac{T}{\delta}}+H^{2}SA\log\frac{T}{\delta}\nonumber\nonumber \\
	& \qquad\asymp\sqrt{H^{2}SA \bigg\{ \sum_{h=1}^{H}\sum_{k=1}^{K}P_{h,s_{h}^{k},a_{h}^{k}}\big(V_{h+1}^{\rref,k}-V_{h+1}^{\star}\big) \bigg\} \log\frac{T}{\delta}}+H^{2}SA\log\frac{T}{\delta}
	\label{eq:bound-lambda-P-V-intermediate-135}
\end{align}
with probability at least $1- \delta/2$. 

It then comes down to controlling the sum $\sum_{h=1}^{H}\sum_{k=1}^{K}P_{h,s_{h}^{k},a_{h}^{k}}\big(V_{h+1}^{\rref,k}-V_{h+1}^{\star}\big)$. Towards this end, we first single out the following useful fact:
\begin{align}
	&\sum_{h=1}^H \sum_{k = 1}^K P^k_{h} \big( V^{\rref, k}_{h+1} -V_{h+1}^\star \big) \overset{\mathrm{(i)}}{\leq}  \sum_{h=1}^H \sum_{k = 1}^K P^k_{h} \big(V^{k}_{h+1} + 2 - V_{h+1}^\star \big)  \nonumber \\
	&\quad \leq 2HK + \sum_{h=1}^H\sum_{k = 1}^K \Big( V^{k}_{h+1}(s_{h+1}^k) -V_{h+1}^\star(s_{h+1}^k) \Big)  \overset{\mathrm{(ii)}}{\lesssim}  \sqrt{H^7SAK \log\frac{SAT}{\delta}}  + H^3SA + HK \label{equ:freed}
 \end{align}
 with probability at least $1-{\delta}/{4}$, 
 where (i) holds according to \eqref{eq:close-ref-v}, 
 and (ii) is valid since
 \begin{align*}
	 \sum_{h=1}^{H}\sum_{k=1}^{K}\Big(V_{h+1}^{k}(s_{h+1}^{k})-V_{h+1}^{\star}(s_{h+1}^{k})\Big) &\leq\sum_{h=1}^{H}\sum_{k=1}^{K}\Big(V_{h+1}^{\UCB,k}(s_{h+1}^{k})-V_{h+1}^{\pi^{k}}(s_{h+1}^{k})\Big) \\
	 & \lesssim \sqrt{H^{7}SAK\log\frac{SAT}{\delta}} + H^3SA,
 \end{align*}
where the first inequality follows since $V_{h+1}^{\UCB,k}\geq V_{h+1}^{k}$ and $V_{h+1}^{\star}\geq V_{h+1}^{\pi^{k}}$, and the second inequality comes from \eqref{eq:sum-delta-k-bound}. 
 Additionally, invoking Freedman's inequality (see Lemma~\ref{lemma:martingale-union-all2}) 
 with $c_h=1$ and $\widetilde{W}_h^i=V_{h+1}^{\rref,k}-V_{h+1}^{\star}$ (so that $0\leq \widetilde{W}_h^i(s)\leq H$)  directly leads to
\begin{align*}
\left|\sum_{h=1}^{H}\sum_{k=1}^{K}\big(P_{h}^{k}-P_{h,s_{h}^{k},a_{h}^{k}}\big)\big(V_{h+1}^{\rref,k}-V_{h+1}^{\star}\big)\right| & \lesssim\sqrt{TH^{2}\log \frac{1}{\delta}}+H\log\frac{1}{\delta}\asymp\sqrt{H^{3}K\log \frac{1}{\delta}}
\end{align*}
with probability at least $1-\delta/4$, which taken collectively with  \eqref{equ:freed} reveals that
\begin{align}
\sum_{h=1}^{H}\sum_{k=1}^{K}P_{s_{h}^{k},a_{h}^{k},h}\big(V_{h+1}^{\rref,k}-V_{h+1}^{\star}\big) & \leq\sum_{h=1}^{H}\sum_{k=1}^{K}P_{h}^{k}\big(V_{h+1}^{\rref,k}-V_{h+1}^{\star}\big)+\left|\sum_{h=1}^{H}\sum_{k=1}^{K}\big(P_{h}^{k}-P_{s_{h}^{k},a_{h}^{k},h}\big)\big(V_{h+1}^{\rref,k}-V_{h+1}^{\star}\big)\right| \notag\\
 & \lesssim\sqrt{H^{7}SAK\log\frac{SAT}{\delta}} + H^3SA +HK \label{equ:freed-234}
\end{align} 
with probability at least $1-\delta/2$. 
Substitution into \eqref{eq:bound-lambda-P-V-intermediate-135} then gives
\begin{align}
\Big|\sum_{h=1}^{H}\sum_{k=1}^{K}\lambda_{h}^{k}\big(P_{h}^{k}-P_{h,s_{h}^{k},a_{h}^{k}}\big)&\big(V_{h+1}^{\star}-V_{h+1}^{\rref,k}\big)\Big| \nonumber \\
&  \lesssim\sqrt{H^{2}SA\sum_{h=1}^{H}\sum_{k=1}^{K}P_{h,s_{h}^{k},a_{h}^{k}}\big(V_{h+1}^{\rref,k}-V_{h+1}^{\star}\big)\log\frac{T}{\delta}}+H^{2}SA\log\frac{T}{\delta}\nonumber\nonumber \\
 & \lesssim\sqrt{H^{2}SA\left(\sqrt{H^{7}SAK\log\frac{SAT}{\delta}} + H^3SA +HK\right)\log\frac{T}{\delta}}+H^{2}SA\log\frac{T}{\delta}\nonumber\\
 & \asymp\sqrt{H^{2}SA\left(H^{6}SA\log\frac{SAT}{\delta} + H^3SA +HK\right)\log\frac{T}{\delta}}+H^{2}SA\log\frac{T}{\delta}\nonumber\\
 & \lesssim\sqrt{H^{3}SAK\log\frac{SAT}{\delta}}+H^{4}SA\log\frac{SAT}{\delta} \nonumber \\
& =\sqrt{H^{2}SAT\log\frac{SAT}{\delta}}+H^{4}SA\log\frac{SAT}{\delta}
	\label{eq:step1-R31-bound-final}
\end{align}
with probability exceeding $1-\delta$, where the third line holds since (due to Cauchy-Schwarz)
\[
\sqrt{H^{7}SAK\log\frac{SAT}{\delta}}=\sqrt{H^{6}SA\log\frac{SAT}{\delta}}\sqrt{HK}\lesssim H^{6}SA\log\frac{SAT}{\delta}+HK.
\]

\paragraph{Step 2: upper bounding $\mathcal{R}_3^2$.}

We start by making the following observation: 
\begin{align}
\mathcal{R}_3^2 
& \le \sum_{h = 1}^H \sum_{k = 1}^K\frac{\lambda^{k}_h}{N^{k}_h(s^k_h, a^k_h)}\sum_{i \le N^{k}_h(s^k_h, a^k_h)} \big(V^{\rref, k^i}_{h+1}(s^{k^i}_{h+1}) - P_{h, s^k_h, a^k_h}V^{\rref, K}_{h+1} \big) \nonumber \\
& =\sum_{h = 1}^H \sum_{k = 1}^K\sum_{n = N^{k}_h(s^k_h, a^k_h)}^{N^{K - 1}_h(s^k_h, a^k_h)} \frac{\lambda^{k}_h}{n} \Big( V^{\rref, k}_{h+1}(s^{k}_{h+1}) - V^{\rref, K}_{h+1}(s^{k}_{h+1} ) + \big(P^k_{h} - P_{h, s^k_h, a^k_h} \big)V^{\rref, K}_{h+1} \Big) \nonumber \\
	& \leq  (e\log T) \sum_{h = 1}^H \sum_{k = 1}^K \big(V^{\rref, k}_{h+1}(s^{k}_{h+1}) - V^{\rref, K}_{h+1}(s^{k}_{h+1}) \big) + \sum_{h = 1}^H \sum_{k = 1}^K\sum_{n = N^{k}_h(s^k_h, a^k_h)}^{N^{K -1}_h(s^k_h, a^k_h)} \frac{\lambda^{k}_h}{n}\big( P^k_{h} - P_{h, s^k_h, a^k_h} \big)V^{\star}_{h+1} \nonumber\\
&\qquad +  \sum_{h = 1}^H \sum_{k = 1}^K \sum_{n = N^{k}_h(s^k_h, a^k_h)}^{N^{K -1}_h(s^k_h, a^k_h)} \frac{\lambda^{k}_h}{n} \big( P^k_{h} - P_{h, s^k_h, a^k_h} \big) \big( V^{\rref, K}_{h+1} -V_{h+1}^\star \big), \label{equ:reference-drift1}
\end{align}
where the first inequality comes from the monotonicity property $V^{\rref, k}_{h+1} \ge V^{\rref, k+1}_{h+1} \ge \cdots \geq V^{\rref, K}_{h+1}$, and the last line follows from  the facts that $\sum_{n = N^{k}_h(s^k_h, a^k_h)}^{N^{K - 1}_h(s^k_h, a^k_h)} \frac{1}{n} \leq \log T$ and $\lambda_h^k \leq e$ (cf.~\eqref{eq:lambda_kh_bound}).  In what follows, we shall control the three terms in \eqref{equ:reference-drift1} separately.

\begin{itemize}

\item 
The first term in \eqref{equ:reference-drift1} can be controlled by Lemma~\ref{lem:Vr_properties} (cf.~\eqref{eq:Vr_lazy}) as follows: 
\begin{equation}
\sum_{h=1}^H\sum_{k=1}^K  \big(V^{\rref, k}_{h+1}(s^{k}_{h+1}) - V^{\rref, K}_{h+1}(s^{k}_{h+1})\big) 
	\lesssim H^6SA\log\frac{SAT}{\delta} \label{equ:R32-term1-result}
\end{equation}
with probability at least $1- {\delta}/{3}$.

\item To control the second term in \eqref{equ:reference-drift1},  we abuse the notation by setting
\begin{align*}
	N(s,a,h) \coloneqq N_h^{K-1}(s,a)
\end{align*}
and
\begin{align*}
	W_{h+1}^i \coloneqq V^{\star}_{h+1},
	\qquad \text{and} \qquad
	u_h^i(s_h^i,a_h^i) \coloneqq \sum_{n = N^{i}_h(s^i_h, a^i_h)}^{N(s_h^i,a_h^i,h)} \frac{\lambda^{i}_h}{n},  
\end{align*}
which clearly satisfy
\begin{align*}
	\big| u_h^i(s_h^i,a_h^i) \big| 
	\leq e\sum_{n = N^{i}_h(s^i_h, a^i_h)}^{N(s_h^i,a_h^i,h)} \frac{1}{n}\leq e \log T \eqqcolon C_{\mathrm{u}}
	\qquad \text{and} \qquad
	  \|W_{h+1}^i\|_\infty \leq H \eqqcolon C_{\mathrm{w}}. 
\end{align*}
Here, we have made use of the properties $\sum_{n = N^{i}_h(s^i_h, a^i_h)}^{N^{K-1}_h(s^i_h, a^i_h)} \frac{1}{n} \leq \log T$ and $\lambda_h^k \leq e$ (cf.~\eqref{eq:lambda_kh_bound}).
With these in place, applying Lemma~\ref{lemma:martingale-union-all2} reveals that
\begin{align}
&\left|\sum_{h = 1}^H \sum_{k = 1}^K \sum_{n = N^{k}_h(s^k_h, a^k_h)}^{N^{K-1}_h(s^k_h, a^k_h)} \frac{\lambda^{k}_h}{n} \big(P^k_{h} - P_{h, s^k_h, a^k_h} \big)V^{\star}_{h+1} \right| = \left|\sum_{h = 1}^H \sum_{k = 1}^K X_{k,h} \right| \nonumber\\
&\quad \lesssim \sqrt{ C_{\mathrm{u}}^{2} HSA \sum_{h=1}^{H} \sum_{i=1}^{K}\mathbb{E}_{i,h-1}\left[\big|(P_{h}^{i}-P_{h,s_{h}^{i},a_{h}^{i}})W_{h+1}^{i}\big|^{2}\right] \log\frac{T}{\delta} }  + C_{\mathrm{u}} C_{\mathrm{w}} HSA  \log\frac{T}{\delta}  \notag\\
&\quad \overset{\mathrm{(i)}}{\asymp} \sqrt{ \sum_{h = 1}^H \sum_{k = 1}^K \Var_{h, s^k_h, a^k_h}(V^{\star}_{h+1}) \cdot HSA\log^3\frac{T}{\delta} }  + H^2SA \log^2\frac{T}{\delta}  \notag\\
& \quad \overset{\mathrm{(ii)}}{\lesssim} \sqrt{HSA \big(HT + H^4\sqrt{SAT} \big) \log^4\frac{SAT}{\delta}} + H^2SA \log^2\frac{T}{\delta} \notag\\ 
& \quad  \lesssim \sqrt{HSA\big(HT+H^{7}SA\big)\log^{4}\frac{SAT}{\delta}}+H^{2}SA\log^{2}\frac{T}{\delta} \notag\\
& \quad \overset{\mathrm{(iii)}}{\lesssim} \sqrt{H^2SAT\log^4\frac{SAT}{\delta}} +  H^4SA \log^2\frac{SAT}{\delta} \label{equ:R32-term2-result}
\end{align}
with probability at least $1-{\delta}/{3}$. Here, (i) comes from the definition in \eqref{lemma1:equ2}, (ii) holds due to \eqref{equ:var-vsatr-1} and \eqref{equ:var-vsatr-2}, whereas (iii) is valid since
\[
	HT+H^{4}\sqrt{SAT}=HT+\sqrt{H^{7}SA}\cdot\sqrt{HT}\lesssim HT+H^{7}SA
\]
due to the Cauchy-Schwarz inequality.

\item Turning attention the third term of \eqref{equ:reference-drift1}, we need to properly cope with the dependency between $P_h^k$ and $V_{h+1}^{\rref, K}$. Towards this, we shall resort to the standard epsilon-net argument (see, e.g., \citep{Tao2012RMT}), which will be presented in Appendix~\ref{appendix:covering}. The final bound reads like
 \begin{align}
 & \left|\sum_{h=1}^{H}\sum_{k=1}^{K}\sum_{n=N_{h}^{k}(s_{h}^{k},a_{h}^{k})}^{N_{h}^{K-1}(s_{h}^{k},a_{h}^{k})}\frac{\lambda_{h}^{k}}{n}\big(P_{h}^{k}-P_{h,s_{h}^{k},a_{h}^{k}}\big)\big(V_{h+1}^{\rref,K}-V_{h+1}^{\star}\big)\right| \lesssim H^{4}SA\log^{2}\frac{SAT}{\delta}+\sqrt{H^{3}SAK\log^{3}\frac{SAT}{\delta}}.
	 \label{eq:final-term-control-246}
\end{align}

\item
Combining \eqref{equ:R32-term1-result}, \eqref{equ:R32-term2-result} and \eqref{eq:final-term-control-246} with \eqref{equ:reference-drift1}, we can use the union bound to demonstrate that
\begin{align}
	\mathcal{R}_3^2   &\leq  C_{3,2} \bigg\{  H^6 SA\log^3\frac{SAT}{\delta} + \sqrt{H^2SAT\log^4\frac{SAT}{\delta}} \bigg\}
	\label{eq:R32-upper-bound}
\end{align}
with probability at least $1-\delta$, where $C_{3,2}>0$ is some constant.

\end{itemize}

\paragraph{Step 3: final bound of $ \mathcal{R}_3$.} 
Putting the above results \eqref{eq:step1-R31-bound-final} and \eqref{eq:R32-upper-bound} together, we immediately arrive at
\begin{align}
	 \mathcal{R}_3  \leq  \big| \mathcal{R}_3^1 \big|  +  \mathcal{R}_3^2 
	&\leq C_{\mathrm{r},3} \bigg\{ H^6 SA\log^3\frac{SAT}{\delta} + \sqrt{H^2SAT\log^4\frac{SAT}{\delta}} \bigg\}
\end{align}
with probability at least $1-2\delta$, where $C_{\mathrm{r},3}>0$ is some constant.  
This immediately concludes the proof.

\subsubsection{Proof of \eqref{eq:final-term-control-246}}
\label{appendix:covering}
\paragraph{Step 1: concentration bounds for a fixed group of vectors.} 

Consider a fixed group of vectors $ \{ V^{\mathrm{d}}_{h+1} \in \mathbb{R}^{S} \mid 1\leq h \leq H\}$ obeying the following properties: 
\begin{align}
	V_{h+1}^{\star} \leq V^{\mathrm{d}}_{h+1} &\leq H \qquad \text{ for } 1\leq h\leq H .
	\label{equ:assumption-1}
\end{align}
We intend to control the following sum
\[
\sum_{h=1}^{H}\sum_{k=1}^{K}\sum_{n=N_{h}^{k}(s_{h}^{k},a_{h}^{k})}^{N_{h}^{K-1}(s_{h}^{k},a_{h}^{k})}\frac{\lambda_{h}^{k}}{n}\big(P_{h}^{k}-P_{h,s_{h}^{k},a_{h}^{k}}\big)\big( V_{h+1}^{\mathrm{d}} -V_{h+1}^{\star}\big) .
\]

To do so, we shall resort to Lemma~\ref{lemma:martingale-union-all2}. For the moment, let us take $N(s,a,h) \coloneqq N_h^{K-1}(s,a)$
and
\begin{align*}
	W_{h+1}^i \coloneqq V_{h+1}^{\mathrm{d}} - V^{\star}_{h+1},
	\qquad  \qquad
	u_h^i(s_h^i,a_h^i) \coloneqq \sum_{n = N^{i}_h(s^i_h, a^i_h)}^{N(s_h^i,a_h^i,h)} \frac{\lambda^{i}_h}{n}. 
\end{align*}
It is easily seen that
\begin{align*}
	\big|u_h^i(s_h^i,a_h^i) \big| 
	\leq e\sum_{n = N^{i}_h(s^i_h, a^i_h)}^{N(s_h^i,a_h^i,h)} \frac{1}{n}\leq e \log T \eqqcolon C_{\mathrm{u}}
	\qquad \text{and} \qquad
	  \|W_{h+1}^i\|_\infty \leq H \eqqcolon C_{\mathrm{w}}, 
\end{align*}
which hold due to the facts $\sum_{n = N^{i}_h(s^i_h, a^i_h)}^{N^{K}_h(s^i_h, a^i_h)} \frac{1}{n} \leq \log T$ and $\lambda_h^k \leq e$ (cf.~\eqref{eq:lambda_kh_bound}) as well as the property that $V_{h+1}^{\mathrm{d}}(s), V^{\star}_{h+1}(s)\in [0,H]$. 
Thus, invoking Lemma~\ref{lemma:martingale-union-all2} yields
\begin{align}
 & \left|\sum_{h=1}^{H}\sum_{k=1}^{K}\sum_{n=N_{h}^{k}(s_{h}^{k},a_{h}^{k})}^{N_{h}^{K-1}(s_{h}^{k},a_{h}^{k})}\frac{\lambda_{h}^{k}}{n}\big(P_{h}^{k}-P_{h,s_{h}^{k},a_{h}^{k}}\big)\big(V_{h+1}^{\mathrm{d}}-V_{h+1}^{\star}\big)\right|=\left|\sum_{h=1}^{H}\sum_{k=1}^{K}X_{k,h}\right|\nonumber\\
 & \lesssim\sqrt{C_{\mathrm{u}}^{2}C_{\mathrm{w}}\sum_{h=1}^{H}\sum_{i=1}^{K}\mathbb{E}_{i,h-1}\left[P_{h}^{i}W_{h+1}^{i}\right]\log\frac{K^{HSA}}{\delta_{0}}}+C_{\mathrm{u}}C_{\mathrm{w}}\log\frac{K^{HSA}}{\delta_{0}}\nonumber\\
	& \lesssim\sqrt{H\sum_{h=1}^{H}\sum_{i=1}^{K}P_{h,s_{h}^{k},a_{h}^{k}}\left(V_{h+1}^{\mathrm{d}}-V_{h+1}^{\star}\right)\big(\log^{2}T\big)\log\frac{K^{HSA}}{\delta_{0}}}+H \big(\log T\big) \log\frac{K^{HSA}}{\delta_{0}}\label{equ:piecewise-bound} 
\end{align}
with probability at least $1- \delta_0$, where the choice of $\delta_0$ will be revealed momentarily.

\paragraph{Step 2: constructing and controlling an epsilon net.}

Our argument in Step~1 is only applicable to a fixed group of vectors. 
The next step is then to construct an epsilon net that allows one to cover the set of interest. 
Specifically, let us construct an epsilon net $\mathcal{N}_{h+1,\alpha}$ (the value of $\alpha$ will be specified shortly) for each $h\in [H]$ such that:
\begin{itemize}
	\item[a)]
for any $V_{h+1} \in [0,H]^S$, one can find a point $V^{\mathrm{net}}_{h+1} \in \mathcal{N}_{h+1,\alpha}$ obeying
\[
	0\leq V_{h+1} (s) - V^{\mathrm{net}}_{h+1} (s) \leq \alpha \qquad \text{for all }s\in \cS;
\]
\item[b)] its cardinality obeys
\begin{equation}
	\big| \mathcal{N}_{h+1,\alpha} \big| \leq \Big( \frac{H}{\alpha} \Big)^S.  
	\label{eq:cardinality-N-h-alpha}
\end{equation}
\end{itemize}
Clearly, this also means that
\[
	\big| \mathcal{N}_{2,\alpha} \times \mathcal{N}_{3,\alpha} \times \cdots \times \mathcal{N}_{H+1,\alpha} \big| 
	\leq \Big( \frac{H}{\alpha} \Big)^{SH}.
\]

Set $\delta_0 = \frac{1}{6} \delta / \big( \frac{H}{\alpha} \big)^{SH}$. 
Taking \eqref{equ:piecewise-bound} together the union bound implies that: with probability at least $1- \delta_0 \big( \frac{H}{\alpha} \big)^{SH} = 1-\delta/6$,  
one has
\begin{align}
 & \left|\sum_{h=1}^{H}\sum_{k=1}^{K}\sum_{n=N_{h}^{k}(s_{h}^{k},a_{h}^{k})}^{N_{h}^{K-1}(s_{h}^{k},a_{h}^{k})}\frac{\lambda_{h}^{k}}{n}\big(P_{h}^{k}-P_{h,s_{h}^{k},a_{h}^{k}}\big)\big(V_{h+1}^{\mathrm{net}}-V_{h+1}^{\star}\big)\right|\nonumber\\
	& \lesssim\sqrt{H\sum_{h=1}^{H}\sum_{i=1}^{K}P_{h,s_{h}^{k},a_{h}^{k}}\left(V_{h+1}^{\mathrm{net}}-V_{h+1}^{\star}\right)\big(\log^{2}T\big)\log\frac{K^{HSA}}{\delta_{0}}}+H (\log T)\log\frac{K^{HSA}}{\delta_{0}} \notag\\
 & \lesssim\sqrt{H^{2}SA\sum_{h=1}^{H}\sum_{i=1}^{K}P_{h,s_{h}^{k},a_{h}^{k}}\left(V_{h+1}^{\mathrm{net}}-V_{h+1}^{\star}\right)\big(\log^{2}T\big)\log\frac{SAT}{\delta \alpha}}+H^{2}SA\log^2 \frac{SAT}{\delta \alpha}
	\label{equ:piecewise-bound-uniform}
\end{align}
simultaneously for all $\{V_{h+1}^{\mathrm{net}} \mid 1\leq h\leq H\}$ obeying $V_{h+1}^{\mathrm{d}} \in \mathcal{N}_{h+1,\alpha} $ ($h\in [H]$).

\paragraph{Step 3: obtaining uniform bounds.}
We are now positioned to establish a uniform bound over the entire set of interest. 
Consider an arbitrary group of vectors
$ \{ V^{\mathrm{u}}_{h+1} \in \mathbb{R}^{S} \mid 1\leq h \leq H\}$ obeying \eqref{equ:assumption-1}.  
By construction, one can find a group of points $\big\{ V^{\mathrm{net}}_{h+1} \in \mathcal{N}_{h+1,\alpha} \mid h\in [H] \big\}$
such that
\begin{align}
	0\leq  V^{\mathrm{u}}_{h+1}(s) - V^{\mathrm{net}}_{h+1} (s) \leq \alpha \qquad \text{for all }(h,s)\in \cS\times [H]. 
	\label{eq:Vd-Vnet-dominate-condition}
\end{align}
It is readily seen that
\begin{align}
 & \left|\sum_{k=1}^{K}\sum_{n=N_{h}^{k}(s_{h}^{k},a_{h}^{k})}^{N_{h}^{K-1}(s_{h}^{k},a_{h}^{k})}\frac{\lambda_{h}^{k}}{n}\big(P_{h}^{k}-P_{h,s_{h}^{k},a_{h}^{k}}\big)\big(V_{h+1}^{\mathrm{u}}-V_{h+1}^{\mathrm{net}}\big)\right|\nonumber\\
	& \leq\left|\sum_{k=1}^{K}\sum_{n=N_{h}^{k}(s_{h}^{k},a_{h}^{k})}^{N_{h}^{K-1}(s_{h}^{k},a_{h}^{k})}\frac{\lambda_{h}^{k}}{n}\Big( \big\Vert P_{h}^{k} \big\Vert_{1}  + \big\Vert P_{h,s_{h}^{k},a_{h}^{k}}\big\Vert _{1} \Big) \big\|V_{h+1}^{\mathrm{u}} -V_{h+1}^{\mathrm{net}} \big\|_{\infty}\right|\nonumber\\
 & \leq 2eK \alpha\log T ,
\end{align}
where the last inequality follows from $\sum_{n = N^{i}_h(s^i_h, a^i_h)}^{N^{K-1}_h(s^i_h, a^i_h)} \frac{1}{n} \leq \log T$ and $\lambda_h^k \leq e$ (cf.~\eqref{eq:lambda_kh_bound}).
Consequently, by taking $\alpha = 1/(SAT)$, we can deduce that
\begin{align}
 & \left|\sum_{h=1}^{H}\sum_{k=1}^{K}\sum_{n=N_{h}^{k}(s_{h}^{k},a_{h}^{k})}^{N_{h}^{K-1}(s_{h}^{k},a_{h}^{k})}\frac{\lambda_{h}^{k}}{n}\big(P_{h}^{k}-P_{h,s_{h}^{k},a_{h}^{k}}\big)\big(V_{h+1}^{\mathrm{u}}-V_{h+1}^{\star}\big)\right|\nonumber\\
 & \quad\leq\left|\sum_{h=1}^{H}\sum_{k=1}^{K}\sum_{n=N_{h}^{k}(s_{h}^{k},a_{h}^{k})}^{N_{h}^{K-1}(s_{h}^{k},a_{h}^{k})}\frac{\lambda_{h}^{k}}{n}\big(P_{h}^{k}-P_{h,s_{h}^{k},a_{h}^{k}}\big)\big(V_{h+1}^{\mathrm{net}}-V_{h+1}^{\star}\big)\right|\notag\\
 & \qquad\quad+\sum_{h=1}^{H}\left|\sum_{k=1}^{K}\sum_{n=N_{h}^{k}(s_{h}^{k},a_{h}^{k})}^{N_{h}^{K-1}(s_{h}^{k},a_{h}^{k})}\frac{\lambda_{h}^{k}}{n}\big(P_{h}^{k}-P_{h,s_{h}^{k},a_{h}^{k}}\big)\big(V_{h+1}^{\mathrm{u}}-V_{h+1}^{\mathrm{net}}\big)\right|\nonumber\\
 & \quad\lesssim\left|\sum_{h=1}^{H}\sum_{k=1}^{K}\sum_{n=N_{h}^{k}(s_{h}^{k},a_{h}^{k})}^{N_{h}^{K-1}(s_{h}^{k},a_{h}^{k})}\frac{\lambda_{h}^{k}}{n}\big(P_{h}^{k}-P_{h,s_{h}^{k},a_{h}^{k}}\big)\big(V_{h+1}^{\mathrm{net}}-V_{h+1}^{\star}\big)\right|+HK\alpha\log T\nonumber\\
 & \quad\lesssim\sqrt{H^{2}SA\sum_{h=1}^{H}\sum_{i=1}^{K}P_{h,s_{h}^{k},a_{h}^{k}}\left(V_{h+1}^{\mathrm{net}}-V_{h+1}^{\star}\right)\big(\log^{2}T\big)\log\frac{SAT}{\delta\alpha}}+H^{2}SA\log^2\frac{SAT}{\delta \alpha}+HK\alpha\log T \notag\\
 & \quad \asymp \sqrt{H^{2}SA\sum_{h=1}^{H}\sum_{i=1}^{K}P_{h,s_{h}^{k},a_{h}^{k}}\left(V_{h+1}^{\mathrm{u}}-V_{h+1}^{\star}\right)\big(\log^{2}T\big)\log\frac{SAT}{\delta}}+H^{2}SA\log^2\frac{SAT}{\delta }, 
	\label{eq:uniform-bound-all-Vd}
\end{align}
where the last line holds due to the condition \eqref{eq:Vd-Vnet-dominate-condition} and our choice of $\alpha$.  
To summarize, with probability exceeding $1-\delta/6$, the property \eqref{eq:uniform-bound-all-Vd} holds simultaneously for all 
$ \{ V^{\mathrm{u}}_{h+1} \in \mathbb{R}^{S} \mid 1\leq h \leq H\}$ obeying \eqref{equ:assumption-1}.

\paragraph{Step 4: controlling the original term of interest. } 
 With the above union bound in hand, we are ready to control the original term of interest
 \begin{align}
 \sum_{h = 1}^H \sum_{k = 1}^K \sum_{n = N^{k}_h(s^k_h, a^k_h)}^{N^{K-1}_h(s^k_h, a^k_h)} \frac{\lambda^{k}_h}{n}
	 \big(P^k_{h} - P_{h, s^k_h, a^k_h} \big) \big(V^{\rref, K}_{h+1} -V_{h+1}^\star \big).
 \end{align}
 To begin with, it can be easily verified using \eqref{equ:optimism-ref-Q} that
 \begin{align}
 	V_{h+1}^{\star} \leq V^{\rref, K}_{h+1} &\leq H \qquad \text{ for all } 1\leq h \leq H.
 \end{align}
Moreover, we make the observation that
\begin{align}
  \sum_{h=1}^{H}\sum_{k=1}^{K}P_{h,s_{h}^{k},a_{h}^{k}}\big(V_{h+1}^{\rref,K}-V_{h+1}^{\star}\big)
	&\overset{\mathrm{(i)}}{\leq}\sum_{h=1}^{H}\sum_{k=1}^{K}P_{h,s_{h}^{k},a_{h}^{k}}\big(V_{h+1}^{\rref,k}-V_{h+1}^{\star}\big) \notag\\
	& \overset{\mathrm{(ii)}}{\leq}\sqrt{H^{7}SAK\log\frac{SAT}{\delta}}+ H^3SA + HK
\end{align}
with probability exceeding $1-\delta/6$, 
 where (i) holds because $V_{h+1}^{\rref}$ is monotonically non-increasing (in view of the monotonicity of $V_h(s)$ in \eqref{eq:monotonicity_Vh} and the update rule in line~\ref{eq:line-number-16} of Algorithm~\ref{algo:vr-k}), and (ii) follows from \eqref{equ:freed-234}.
 Substitution into \eqref{eq:uniform-bound-all-Vd} yields
 \begin{align}
 & \left|\sum_{h=1}^{H}\sum_{k=1}^{K}\sum_{n=N_{h}^{k}(s_{h}^{k},a_{h}^{k})}^{N_{h}^{K-1}(s_{h}^{k},a_{h}^{k})}\frac{\lambda_{h}^{k}}{n}\big(P_{h}^{k}-P_{h,s_{h}^{k},a_{h}^{k}}\big)\big(V_{h+1}^{\rref,K}-V_{h+1}^{\star}\big)\right|\nonumber\\
 & \quad\lesssim\sqrt{H^{2}SA\sum_{h=1}^{H}\sum_{i=1}^{K}P_{h,s_{h}^{k},a_{h}^{k}}\left(V_{h+1}^{\rref,K}-V_{h+1}^{\star}\right)\big(\log^{2}T\big)\log\frac{SAT}{\delta}}+H^{2}SA\log^2\frac{SAT}{\delta} \notag\\
 & \quad\lesssim\sqrt{H^{2}SA\left\{ \sqrt{H^{7}SAK\log\frac{SAT}{\delta}}+ H^3SA +HK\right\} \big(\log^{2}T\big)\log\frac{SAT}{\delta}}+H^{2}SA\log^2\frac{SAT}{\delta} \notag\\
 & \quad\lesssim\sqrt{H^{2}SA\left\{ H^{6}SA\log\frac{SAT}{\delta}+ H^3SA + HK\right\} \log^{3}\frac{SAT}{\delta}}+H^{2}SA\log^2\frac{SAT}{\delta} \notag\\
 & \quad\lesssim H^{4}SA\log^{2}\frac{SAT}{\delta}+\sqrt{H^{3}SAK\log^{3}\frac{SAT}{\delta}}, 
	 \label{eq:final-term-control-246-repeat}
\end{align}
where the penultimate line holds since 
\[
\sqrt{H^{7}SAK\log\frac{SAT}{\delta}}=\sqrt{H^{6}SA\log\frac{SAT}{\delta}}\sqrt{HK}\lesssim H^{6}SA\log\frac{SAT}{\delta}+HK.
\]

\bibliography{bibfileRL}

\begin{thebibliography}{}

\bibitem[Agarwal et~al., 2020]{agarwal2019optimality}
Agarwal, A., Kakade, S., and Yang, L.~F. (2020).
\newblock Model-based reinforcement learning with a generative model is minimax
  optimal.
\newblock {\em Conference on Learning Theory}, pages 67--83.

\bibitem[Agrawal and Jia, 2017]{agrawal2017posterior}
Agrawal, S. and Jia, R. (2017).
\newblock Posterior sampling for reinforcement learning: worst-case regret
  bounds.
\newblock {\em arXiv preprint arXiv:1705.07041}.

\bibitem[Auer et~al., 2002]{auer2002nonstochastic}
Auer, P., Cesa-Bianchi, N., Freund, Y., and Schapire, R.~E. (2002).
\newblock The nonstochastic multiarmed bandit problem.
\newblock {\em SIAM journal on computing}, 32(1):48--77.

\bibitem[Auer and Ortner, 2010]{auer2010ucb}
Auer, P. and Ortner, R. (2010).
\newblock {UCB} revisited: Improved regret bounds for the stochastic
  multi-armed bandit problem.
\newblock {\em Periodica Mathematica Hungarica}, 61(1-2):55--65.

\bibitem[Azar et~al., 2011]{ghavamzadeh2011speedy}
Azar, M.~G., Kappen, H.~J., Ghavamzadeh, M., and Munos, R. (2011).
\newblock Speedy {Q}-learning.
\newblock In {\em Advances in neural information processing systems}, pages
  2411--2419.

\bibitem[Azar et~al., 2013]{azar2013minimax}
Azar, M.~G., Munos, R., and Kappen, H.~J. (2013).
\newblock Minimax {PAC} bounds on the sample complexity of reinforcement
  learning with a generative model.
\newblock {\em Machine learning}, 91(3):325--349.

\bibitem[Azar et~al., 2017]{azar2017minimax}
Azar, M.~G., Osband, I., and Munos, R. (2017).
\newblock Minimax regret bounds for reinforcement learning.
\newblock In {\em Proceedings of the 34th International Conference on Machine
  Learning-Volume 70}, pages 263--272. JMLR. org.

\bibitem[Bai et~al., 2019]{bai2019provably}
Bai, Y., Xie, T., Jiang, N., and Wang, Y.-X. (2019).
\newblock Provably efficient $q$-learning with low switching cost.
\newblock In {\em Advances in Neural Information Processing Systems}, pages
  8002--8011.

\bibitem[Bartlett and Tewari, 2009]{bartlett2009regal}
Bartlett, P. and Tewari, A. (2009).
\newblock Regal: a regularization based algorithm for reinforcement learning in
  weakly communicating {MDP}s.
\newblock In {\em Uncertainty in Artificial Intelligence: Proceedings of the
  25th Conference}, pages 35--42. AUAI Press.

\bibitem[Beck and Srikant, 2012]{beck2012error}
Beck, C.~L. and Srikant, R. (2012).
\newblock Error bounds for constant step-size {Q}-learning.
\newblock {\em Systems \& control letters}, 61(12):1203--1208.

\bibitem[Bertsekas, 2017]{bertsekas2017dynamic}
Bertsekas, D.~P. (2017).
\newblock {\em Dynamic programming and optimal control (4th edition)}.
\newblock Athena Scientific.

\bibitem[Chen et~al., 2020]{chen2020finite}
Chen, Z., Maguluri, S.~T., Shakkottai, S., and Shanmugam, K. (2020).
\newblock Finite-sample analysis of stochastic approximation using smooth
  convex envelopes.
\newblock {\em arXiv preprint arXiv:2002.00874}.

\bibitem[Chen et~al., 2021]{chen2021lyapunov}
Chen, Z., Maguluri, S.~T., Shakkottai, S., and Shanmugam, K. (2021).
\newblock A {L}yapunov theory for finite-sample guarantees of asynchronous
  {Q}-learning and {TD}-learning variants.
\newblock {\em arXiv preprint arXiv:2102.01567}.

\bibitem[Dann et~al., 2017]{dann2017unifying}
Dann, C., Lattimore, T., and Brunskill, E. (2017).
\newblock Unifying {PAC} and regret: Uniform {PAC} bounds for episodic
  reinforcement learning.
\newblock {\em arXiv preprint arXiv:1703.07710}.

\bibitem[Domingues et~al., 2021]{domingues2021episodic}
Domingues, O.~D., M{\'e}nard, P., Kaufmann, E., and Valko, M. (2021).
\newblock Episodic reinforcement learning in finite {MDP}s: Minimax lower
  bounds revisited.
\newblock In {\em Algorithmic Learning Theory}, pages 578--598. PMLR.

\bibitem[Dong et~al., 2019]{dong2019q}
Dong, K., Wang, Y., Chen, X., and Wang, L. (2019).
\newblock {Q}-learning with {UCB} exploration is sample efficient for
  infinite-horizon {MDP}.
\newblock {\em arXiv preprint arXiv:1901.09311}.

\bibitem[Du et~al., 2017]{du2017stochastic}
Du, S.~S., Chen, J., Li, L., Xiao, L., and Zhou, D. (2017).
\newblock Stochastic variance reduction methods for policy evaluation.
\newblock In {\em Proceedings of the 34th International Conference on Machine
  Learning-Volume 70}, pages 1049--1058. JMLR. org.

\bibitem[Du et~al., 2020]{du2019good}
Du, S.~S., Kakade, S.~M., Wang, R., and Yang, L.~F. (2020).
\newblock Is a good representation sufficient for sample efficient
  reinforcement learning?
\newblock In {\em International Conference on Learning Representations}.

\bibitem[Du et~al., 2019]{du2019provably}
Du, S.~S., Luo, Y., Wang, R., and Zhang, H. (2019).
\newblock Provably efficient {Q}-learning with function approximation via
  distribution shift error checking oracle.
\newblock In {\em Advances in Neural Information Processing Systems}, pages
  8058--8068.

\bibitem[Efroni et~al., 2019]{efroni2019tight}
Efroni, Y., Merlis, N., Ghavamzadeh, M., and Mannor, S. (2019).
\newblock Tight regret bounds for model-based reinforcement learning with
  greedy policies.
\newblock {\em arXiv preprint arXiv:1905.11527}.

\bibitem[Even-Dar and Mansour, 2003]{even2003learning}
Even-Dar, E. and Mansour, Y. (2003).
\newblock Learning rates for {Q}-learning.
\newblock {\em Journal of machine learning Research}, 5(Dec):1--25.

\bibitem[Fan et~al., 2019]{fan2019theoretical}
Fan, J., Wang, Z., Xie, Y., and Yang, Z. (2019).
\newblock A theoretical analysis of deep {Q}-learning.
\newblock {\em arXiv e-prints}, pages arXiv--1901.

\bibitem[Freedman, 1975]{freedman1975tail}
Freedman, D.~A. (1975).
\newblock On tail probabilities for martingales.
\newblock {\em the Annals of Probability}, pages 100--118.

\bibitem[Gower et~al., 2020]{gower2020variance}
Gower, R.~M., Schmidt, M., Bach, F., and Richt{\'a}rik, P. (2020).
\newblock Variance-reduced methods for machine learning.
\newblock {\em Proceedings of the IEEE}, 108(11):1968--1983.

\bibitem[He et~al., 2020]{he2020nearly}
He, J., Zhou, D., and Gu, Q. (2020).
\newblock Nearly minimax optimal reinforcement learning for discounted {MDP}s.
\newblock {\em arXiv preprint arXiv:2010.00587}.

\bibitem[Jaakkola et~al., 1994]{jaakkola1994convergence}
Jaakkola, T., Jordan, M.~I., and Singh, S.~P. (1994).
\newblock Convergence of stochastic iterative dynamic programming algorithms.
\newblock In {\em Advances in neural information processing systems}, pages
  703--710.

\bibitem[Jafarnia-Jahromi et~al., 2020]{jafarnia2020model}
Jafarnia-Jahromi, M., Wei, C.-Y., Jain, R., and Luo, H. (2020).
\newblock A model-free learning algorithm for infinite-horizon average-reward
  {MDP}s with near-optimal regret.
\newblock {\em arXiv preprint arXiv:2006.04354}.

\bibitem[Jaksch et~al., 2010]{jaksch2010near}
Jaksch, T., Ortner, R., and Auer, P. (2010).
\newblock Near-optimal regret bounds for reinforcement learning.
\newblock {\em Journal of Machine Learning Research}, 11(4).

\bibitem[Jin et~al., 2018a]{jin2018q}
Jin, C., Allen-Zhu, Z., Bubeck, S., and Jordan, M.~I. (2018a).
\newblock Is {Q}-learning provably efficient?
\newblock In {\em Advances in Neural Information Processing Systems}, pages
  4863--4873.

\bibitem[Jin et~al., 2018b]{jin2018qarxiv}
Jin, C., Allen-Zhu, Z., Bubeck, S., and Jordan, M.~I. (2018b).
\newblock Is {Q}-learning provably efficient?
\newblock {\em arXiv preprint arXiv:1807.03765}.

\bibitem[Jin et~al., 2020]{jin2020provably}
Jin, C., Yang, Z., Wang, Z., and Jordan, M.~I. (2020).
\newblock Provably efficient reinforcement learning with linear function
  approximation.
\newblock In {\em Conference on Learning Theory}, pages 2137--2143. PMLR.

\bibitem[Johnson and Zhang, 2013]{johnson2013accelerating}
Johnson, R. and Zhang, T. (2013).
\newblock Accelerating stochastic gradient descent using predictive variance
  reduction.
\newblock In {\em Advances in neural information processing systems}, pages
  315--323.

\bibitem[Kakade, 2003]{kakade2003sample}
Kakade, S. (2003).
\newblock {\em On the sample complexity of reinforcement learning}.
\newblock PhD thesis, University of London.

\bibitem[Khamaru et~al., 2020]{khamaru2020temporal}
Khamaru, K., Pananjady, A., Ruan, F., Wainwright, M.~J., and Jordan, M.~I.
  (2020).
\newblock Is temporal difference learning optimal? an instance-dependent
  analysis.
\newblock {\em arXiv preprint arXiv:2003.07337}.

\bibitem[Lai and Robbins, 1985]{lai1985asymptotically}
Lai, T.~L. and Robbins, H. (1985).
\newblock Asymptotically efficient adaptive allocation rules.
\newblock {\em Advances in applied mathematics}, 6(1):4--22.

\bibitem[Lattimore and Szepesv{\'a}ri, 2020]{lattimore2020bandit}
Lattimore, T. and Szepesv{\'a}ri, C. (2020).
\newblock {\em Bandit algorithms}.
\newblock Cambridge University Press.

\bibitem[Li et~al., 2021a]{li2021tightening}
Li, G., Cai, C., Chen, Y., Gu, Y., Wei, Y., and Chi, Y. (2021a).
\newblock Is {Q}-learning minimax optimal? a tight sample complexity analysis.
\newblock {\em arXiv preprint arXiv:2102.06548}.

\bibitem[Li et~al., 2021b]{li2021sample}
Li, G., Chen, Y., Chi, Y., Gu, Y., and Wei, Y. (2021b).
\newblock Sample-efficient reinforcement learning is feasible for linearly
  realizable {MDP}s with limited revisiting.
\newblock {\em Advances in Neural Information Processing Systems},
  34:16671--16685.

\bibitem[Li et~al., 2020a]{li2020breaking}
Li, G., Wei, Y., Chi, Y., Gu, Y., and Chen, Y. (2020a).
\newblock Breaking the sample size barrier in model-based reinforcement
  learning with a generative model.
\newblock In {\em Advances in Neural Information Processing Systems},
  volume~33.

\bibitem[Li et~al., 2020b]{li2020sample}
Li, G., Wei, Y., Chi, Y., Gu, Y., and Chen, Y. (2020b).
\newblock Sample complexity of asynchronous {Q}-learning: Sharper analysis and
  variance reduction.
\newblock In {\em Advances in Neural Information Processing Systems (NeurIPS)}.

\bibitem[Liu and Su, 2020]{liu2020gamma}
Liu, S. and Su, H. (2020).
\newblock $\gamma$-regret for non-episodic reinforcement learning.
\newblock {\em arXiv:2002.05138}.

\bibitem[Menard et~al., 2021]{menard2021ucb}
Menard, P., Domingues, O.~D., Shang, X., and Valko, M. (2021).
\newblock {UCB} momentum {Q}-learning: Correcting the bias without forgetting.
\newblock {\em arXiv preprint arXiv:2103.01312}.

\bibitem[Murphy, 2005]{murphy2005generalization}
Murphy, S. (2005).
\newblock A generalization error for {Q}-learning.
\newblock {\em Journal of Machine Learning Research}, 6:1073--1097.

\bibitem[Nguyen et~al., 2017]{nguyen2017sarah}
Nguyen, L.~M., Liu, J., Scheinberg, K., and Tak{\'a}{\v{c}}, M. (2017).
\newblock {SARAH}: A novel method for machine learning problems using
  stochastic recursive gradient.
\newblock In {\em International Conference on Machine Learning}, pages
  2613--2621. PMLR.

\bibitem[Osband and Van~Roy, 2016]{osband2016lower}
Osband, I. and Van~Roy, B. (2016).
\newblock On lower bounds for regret in reinforcement learning.
\newblock {\em arXiv preprint arXiv:1608.02732}.

\bibitem[Pacchiano et~al., 2020]{pacchiano2020optimism}
Pacchiano, A., Ball, P., Parker-Holder, J., Choromanski, K., and Roberts, S.
  (2020).
\newblock On optimism in model-based reinforcement learning.
\newblock {\em arXiv preprint arXiv:2006.11911}.

\bibitem[Puterman, 2014]{puterman2014markov}
Puterman, M.~L. (2014).
\newblock {\em Markov decision processes: discrete stochastic dynamic
  programming}.
\newblock John Wiley \& Sons.

\bibitem[Qu and Wierman, 2020]{qu2020finite}
Qu, G. and Wierman, A. (2020).
\newblock Finite-time analysis of asynchronous stochastic approximation and
  {Q}-learning.
\newblock {\em Conference on Learning Theory}, pages 3185--3205.

\bibitem[Robbins and Monro, 1951]{robbins1951stochastic}
Robbins, H. and Monro, S. (1951).
\newblock A stochastic approximation method.
\newblock {\em The annals of mathematical statistics}, pages 400--407.

\bibitem[Shi et~al., 2022]{shi2022pessimistic}
Shi, L., Li, G., Wei, Y., Chen, Y., and Chi, Y. (2022).
\newblock Pessimistic {Q}-learning for offline reinforcement learning: Towards
  optimal sample complexity.
\newblock {\em arXiv preprint arXiv:2202.13890}.

\bibitem[Sidford et~al., 2018a]{sidford2018near}
Sidford, A., Wang, M., Wu, X., Yang, L., and Ye, Y. (2018a).
\newblock Near-optimal time and sample complexities for solving markov decision
  processes with a generative model.
\newblock In {\em Advances in Neural Information Processing Systems}, pages
  5186--5196.

\bibitem[Sidford et~al., 2018b]{sidford2018variance}
Sidford, A., Wang, M., Wu, X., and Ye, Y. (2018b).
\newblock Variance reduced value iteration and faster algorithms for solving
  {M}arkov decision processes.
\newblock In {\em Proceedings of the Twenty-Ninth Annual ACM-SIAM Symposium on
  Discrete Algorithms}, pages 770--787. SIAM.

\bibitem[Strehl et~al., 2006]{strehl2006pac}
Strehl, A.~L., Li, L., Wiewiora, E., Langford, J., and Littman, M.~L. (2006).
\newblock {PAC} model-free reinforcement learning.
\newblock In {\em Proceedings of the 23rd international conference on Machine
  learning}, pages 881--888.

\bibitem[Szepesv{\'a}ri, 1997]{szepesvari1997asymptotic}
Szepesv{\'a}ri, C. (1997).
\newblock The asymptotic convergence-rate of {Q}-learning.
\newblock In {\em NIPS}, volume~10, pages 1064--1070. Citeseer.

\bibitem[Tao, 2012]{Tao2012RMT}
Tao, T. (2012).
\newblock {\em Topics in Random Matrix Theory}.
\newblock Graduate Studies in Mathematics. American Mathematical Society,
  Providence, Rhode Island.

\bibitem[Tropp, 2011]{tropp2011freedman}
Tropp, J. (2011).
\newblock Freedman's inequality for matrix martingales.
\newblock {\em Electronic Communications in Probability}, 16:262--270.

\bibitem[Tsitsiklis, 1994]{tsitsiklis1994asynchronous}
Tsitsiklis, J.~N. (1994).
\newblock Asynchronous stochastic approximation and {Q}-learning.
\newblock {\em Machine learning}, 16(3):185--202.

\bibitem[Wai et~al., 2019]{wai2019variance}
Wai, H.-T., Hong, M., Yang, Z., Wang, Z., and Tang, K. (2019).
\newblock Variance reduced policy evaluation with smooth function
  approximation.
\newblock {\em Advances in Neural Information Processing Systems},
  32:5784--5795.

\bibitem[Wainwright, 2019a]{wainwright2019stochastic}
Wainwright, M.~J. (2019a).
\newblock Stochastic approximation with cone-contractive operators: Sharp
  $\ell_{\infty}$-bounds for {Q}-learning.
\newblock {\em arXiv preprint arXiv:1905.06265}.

\bibitem[Wainwright, 2019b]{wainwright2019variance}
Wainwright, M.~J. (2019b).
\newblock Variance-reduced {Q}-learning is minimax optimal.
\newblock {\em arXiv preprint arXiv:1906.04697}.

\bibitem[Wang et~al., 2021]{wang2021sample}
Wang, B., Yan, Y., and Fan, J. (2021).
\newblock Sample-efficient reinforcement learning for linearly-parameterized
  mdps with a generative model.
\newblock {\em Advances in Neural Information Processing Systems},
  34:23009--23022.

\bibitem[Watkins and Dayan, 1992]{watkins1992q}
Watkins, C.~J. and Dayan, P. (1992).
\newblock {Q}-learning.
\newblock {\em Machine learning}, 8(3-4):279--292.

\bibitem[Watkins, 1989]{watkins1989learning}
Watkins, C. J. C.~H. (1989).
\newblock Learning from delayed rewards.
\newblock {\em PhD thesis, King's College, University of Cambridge}.

\bibitem[Weng et~al., 2020]{weng2020momentum}
Weng, B., Xiong, H., Zhao, L., Liang, Y., and Zhang, W. (2020).
\newblock Momentum {Q}-learning with finite-sample convergence guarantee.
\newblock {\em arXiv preprint arXiv:2007.15418}.

\bibitem[Xiong et~al., 2020]{xiong2020finite}
Xiong, H., Zhao, L., Liang, Y., and Zhang, W. (2020).
\newblock Finite-time analysis for double {Q}-learning.
\newblock {\em Advances in Neural Information Processing Systems}, 33.

\bibitem[Xu et~al., 2019]{xu2019reanalysis}
Xu, T., Wang, Z., Zhou, Y., and Liang, Y. (2019).
\newblock Reanalysis of variance reduced temporal difference learning.
\newblock In {\em International Conference on Learning Representations}.

\bibitem[Yang et~al., 2021]{yang2021q}
Yang, K., Yang, L., and Du, S. (2021).
\newblock Q-learning with logarithmic regret.
\newblock In {\em International Conference on Artificial Intelligence and
  Statistics}, pages 1576--1584. PMLR.

\bibitem[Yin et~al., 2021]{yin2021near}
Yin, M., Bai, Y., and Wang, Y.-X. (2021).
\newblock Near-optimal offline reinforcement learning via double variance
  reduction.
\newblock {\em arXiv preprint arXiv:2102.01748}.

\bibitem[Zanette and Brunskill, 2019]{zanette2019tighter}
Zanette, A. and Brunskill, E. (2019).
\newblock Tighter problem-dependent regret bounds in reinforcement learning
  without domain knowledge using value function bounds.
\newblock In {\em International Conference on Machine Learning}, pages
  7304--7312. PMLR.

\bibitem[Zhang et~al., 2020a]{zhang2020model_game}
Zhang, K., Kakade, S., Basar, T., and Yang, L. (2020a).
\newblock Model-based multi-agent {RL} in zero-sum {M}arkov games with
  near-optimal sample complexity.
\newblock {\em Advances in Neural Information Processing Systems}, 33.

\bibitem[Zhang et~al., 2020b]{zhang2020reinforcement}
Zhang, Z., Ji, X., and Du, S.~S. (2020b).
\newblock Is reinforcement learning more difficult than bandits? a near-optimal
  algorithm escaping the curse of horizon.
\newblock {\em arXiv preprint arXiv:2009.13503}.

\bibitem[Zhang et~al., 2020c]{zhang2020almost}
Zhang, Z., Zhou, Y., and Ji, X. (2020c).
\newblock Almost optimal model-free reinforcement learning via
  reference-advantage decomposition.
\newblock {\em Advances in Neural Information Processing Systems}, 33.

\bibitem[Zhang et~al., 2020d]{zhang2020model}
Zhang, Z., Zhou, Y., and Ji, X. (2020d).
\newblock Model-free reinforcement learning: from clipped pseudo-regret to
  sample complexity.
\newblock {\em arXiv preprint arXiv:2006.03864}.

\end{thebibliography}
\bibliographystyle{apalike} 

\end{document}